\documentclass{article}

\usepackage[final]{neurips_2022}

\usepackage{caption}
\usepackage{subcaption}

\usepackage[utf8]{inputenc} % allow utf-8 input
\usepackage[T1]{fontenc}    % use 8-bit T1 fonts
\usepackage{url}            % simple URL typesetting
\usepackage{booktabs}       % professional-quality tables
\usepackage{amsfonts}       % blackboard math symbols
\usepackage{nicefrac}       % compact symbols for 1/2, etc.
\usepackage{microtype}      % microtypography
\usepackage{xcolor}
\usepackage[normalem]{ulem}
\usepackage{graphicx}
\usepackage{amsmath,amssymb,amsfonts,amsthm}
\usepackage{booktabs} % for professional tables
\usepackage{xspace}
\usepackage{enumerate}

\usepackage[bookmarks=false,colorlinks=true,linkcolor=blue,urlcolor=magenta,citecolor=red,pdfstartview=FitH,pagebackref]{hyperref}

\usepackage{algorithm}
\usepackage{algpseudocode}

\algblockdefx{MRepeat}{EndRepeat}{\textbf{repeat}}{}
\algnotext{EndRepeat}

\newboolean{showcomments}
\setboolean{showcomments}{true}

\newcommand{\AMF}[1]{\ifthenelse{\boolean{showcomments}}{\textcolor{blue}{AMF: #1}}{}}
\newcommand{\amin}[1]{\ifthenelse{\boolean{showcomments}}{\textcolor{red}{AR: #1}}{}}
\newcommand{\andrew}[1]{\ifthenelse{\boolean{showcomments}}{\textcolor{cyan}{AW: #1}}{}}

\newcommand{\revised}[2]{#2}
\usepackage{subcaption}

\renewcommand{\AA}{{\mathcal{A}}}

\newcommand{\XX}{{\mathcal{X}}}

\newcommand{\XA}{\XX\times\AA}

\newcommand{\beq}{\begin{equation}}
\newcommand{\eeq}{\end{equation}}
\newcommand{\beqa}{\begin{eqnarray}}
\newcommand{\eeqa}{\end{eqnarray}}
\newcommand{\beqan}{\begin{eqnarray*}}
\newcommand{\eeqan}{\end{eqnarray*}}
\newcommand{\ben}{\begin{eqnarray*}}
\newcommand{\een}{\end{eqnarray*}}

\newcommand{\norm}[1]{\left\Vert#1\right\Vert}
\newcommand{\smallnorm}[1]{\Vert#1\Vert}
\newcommand{\abs}[1]{\left\vert#1\right\vert}

\newcommand{\Real}{\mathbb R}
\newcommand{\real}{\mathbb R}

\newcommand{\eps}{\varepsilon}

\newcommand{\ra}{\rightarrow}

\newcommand{\argmax}{\mathop{\textrm{argmax}}}

\newcommand{\Vmax}{V_{\textrm{max}}}

\newcommand{\eqdef}{\triangleq}

\newcommand{\MM}{\mathcal{M}}
\newcommand{\Qpi}{Q^\pi}
\newcommand{\Tpi}{T^\pi}

\newcommand{\Topt}{{T^*}}

\newcommand{\Id}{\mathbf{I}}

\newcommand{\EEX}[2]{{\mathbb E}_{#1}\left[#2\right]}

\newcommand{\Qopt}{Q^*}
\newcommand{\Vpi}{V^\pi}
\newcommand{\Vopt}{V^*}
\newcommand{\Vhat}{\hat{V}}

\newcommand{\PKernel}{\mathcal{P}}
\newcommand{\RKernel}{\mathcal{R}}

\newcommand{\pihat}{{\hat{\pi}}}
\newcommand{\piopt}{{\pi^*}}

\newcommand{\KL}{\textsf{KL} }
\newcommand{\PKernelhat}{\hat{\PKernel} }

\newcommand{\PKernelpihat}{\hat{\PKernel}^{\pi} }

\newcommand{\dx}{\mathrm{d}x}
\newcommand{\dy}{\mathrm{d}y}
\newcommand{\dz}{\mathrm{d}z}

\newcommand{\drho}{\mathrm{d}\rho}

\newcommand{\PKernelpi}{{\PKernel}^{\pi}}

\newcommand{\Gpi}{G^\pi}

\newcommand{\rpi}{r^\pi}
\newcommand{\BB}{\mathcal{B}}

\newcommand{\pigreedy}{{\pi}_{g}}

\def\balign#1\ealign{\begin{align}#1\end{align}}
\def\baligns#1\ealigns{\begin{align*}#1\end{align*}}
\def\balignat#1\ealign{\begin{alignat}#1\end{alignat}}
\def\balignats#1\ealigns{\begin{alignat*}#1\end{alignat*}}
\def\bitemize#1\eitemize{\begin{itemize}#1\end{itemize}}
\def\benumerate#1\eenumerate{\begin{enumerate}#1\end{enumerate}}

\newenvironment{talign*}
 {\csname align*\endcsname}
 {\endalign}
\newenvironment{talign}
 {\csname align\endcsname}
 {\endalign}

\def\balignst#1\ealignst{\begin{talign*}#1\end{talign*}}
\def\balignt#1\ealignt{\begin{talign}#1\end{talign}}
\let\originalleft\left
\let\originalright\right
\renewcommand{\left}{\mathopen{}\mathclose\bgroup\originalleft}
\renewcommand{\right}{\aftergroup\egroup\originalright}

\def\tinycitep*#1{{\tiny\citep*{#1}}}
\def\tinycitealt*#1{{\tiny\citealt*{#1}}}
\def\tinycite*#1{{\tiny\cite*{#1}}}
\def\smallcitep*#1{{\scriptsize\citep*{#1}}}
\def\smallcitealt*#1{{\scriptsize\citealt*{#1}}}
\def\smallcite*#1{{\scriptsize\cite*{#1}}}

\def\reals{\mathbb{R}} % Real number symbol
\def\<{\left\langle} % Angle brackets
\def\>{\right\rangle}

\def\defeq{\triangleq} % defined equal to
\def\norm#1{\left\|{#1}\right\|} % A norm with 1 argument
\newcommand{\vecle}{\preccurlyeq}
\newcommand{\vecge}{\succcurlyeq}

\newcommand{\grad}{\nabla}
\providecommand{\argmax}{\mathop\mathrm{arg max}} % Defining math symbols

\ifdefined\nonewproofenvironments\else
\ifdefined\ispres\else
\newtheorem{theorem}{Theorem}
\newtheorem{lemma}[theorem]{Lemma}

\renewenvironment{proof}{\noindent\textbf{Proof.}\hspace*{.3em}}{\qed\\}
\newenvironment{proof-sketch}{\noindent\textbf{Proof Sketch}
  \hspace*{1em}}{\qed\bigskip\\}
\newenvironment{proof-idea}{\noindent\textbf{Proof Idea}
  \hspace*{1em}}{\qed\bigskip\\}
\newenvironment{proof-of-lemma}[1][{}]{\noindent\textbf{Proof of Lemma {#1}}
  \hspace*{1em}}{\qed\\}
\newenvironment{proof-of-theorem}[1][{}]{\noindent\textbf{Proof of Theorem {#1}}
  \hspace*{1em}}{\qed\\}
\newenvironment{proof-attempt}{\noindent\textbf{Proof Attempt}
  \hspace*{1em}}{\qed\bigskip\\}

\newenvironment{remark}{\noindent\textbf{Remark.}
  \hspace*{0em}}{\smallskip}%\bigskip}

\fi

\newtheorem{proposition}[theorem]{Proposition}

\fi
\numberwithin{equation}{section}
\makeatletter
\@addtoreset{equation}{section}
\makeatother

\hypersetup{
  colorlinks,
  linkcolor={red!50!black},
  citecolor={blue!50!black},
  urlcolor={blue!80!black}
}

\renewcommand{\Pr}[1]{\mathbb{P}\left( #1 \right)}

\renewcommand{\abs}[1]{\left|#1\right|}

\newcommand{\handout}[5]{
  \noindent
  \begin{center}
    \framebox{
      \vbox{
        \hbox to 5.78in { {\bf \title } \hfill #2 }
        \vspace{4mm}
        \hbox to 5.78in { {\Large \hfill #5  \hfill} }
        \vspace{2mm}
        \hbox to 5.78in { {\em #3 \hfill #4} }
      }
    }
  \end{center}
  \vspace*{4mm}
}

\newcommand{\etahatpi}{ {\hat \eta^\pi} }
\newcommand{\errvm}{ \epsilon^\text{value} }
\newcommand{\errpim}{ \epsilon^\text{policy} }
\newcommand{\errv}[1]{ {\epsilon_{#1}^\text{value} } }
\newcommand{\errpi}[1]{ {\epsilon_{#1}^\text{policy} } }
\newcommand{\normfrho}[1]{\norm{#1}_{4, \rho}}

\newif\ifconsiderlater
\considerlaterfalse
\ifconsiderlater
	\newcommand{\todo}[1]{{\color{cyan} \textbf{XXX [#1] XXX}}}
\else
\newcommand{\todo}[1]{}
\fi

\newif\ifSupp
\Supptrue

\title{Operator Splitting Value Iteration}
\author{
Amin Rakhsha$^{1,2}$
\And
Andrew Wang$^{1,2}$
\And Mohammad Ghavamzadeh$^3$ \And 
Amir-massoud Farahmand$^{2,1}$
\\ \\
$^1$Department of Computer Science, University of Toronto
\\ % \quad 
$^2$Vector Institute \quad
$^3$Google Research
}

\begin{document}
\maketitle

%%%%%%%%%%%%%%%%%%%%%%%%%%%%%%%%%%%%%%%%%%%%%%%
%%%%%%%%%%%%%%%%%%%%%%%%%%%%%%%%%%%%%%%%%%%%%%%
%%%%%%%%%%%%%%%%%%%%%%%%%%%%%%%%%%%%%%%%%%%%%%%
\begin{abstract}
We introduce new planning and reinforcement learning algorithms for discounted MDPs that utilize an approximate model of the environment to accelerate the convergence of the value function.
Inspired by the splitting approach in numerical linear algebra, we introduce \emph{Operator Splitting Value Iteration} (OS-VI) for both Policy Evaluation and Control problems. OS-VI achieves a much faster convergence rate when the model is accurate enough. We also introduce a sample-based version of the algorithm called OS-Dyna. Unlike the traditional Dyna architecture, OS-Dyna still converges to the correct value function in presence of model approximation error. % To the best of our knowledge, this is the only model-based algorithm with this property.
\end{abstract}
%%%%%%%%%%%%%%%%%%%%%%%%%%%%%%%%%%%%%%%%%%%%%%%
%%%%%%%%%%%%%%%%%%%%%%%%%%%%%%%%%%%%%%%%%%%%%%%
%%%%%%%%%%%%%%%%%%%%%%%%%%%%%%%%%%%%%%%%%%%%%%%

%%%%%%%%%%%%%%%%%%%%%%%%%%%%%%%%%%%%%%%%%%%%%%%
\section{Introduction}
\label{sec:Introduction}
% !TEX root =  OSVI.tex

%Value Iteration (VI) is a fundamental algorithm in Dynamic Programming and is a backbone of many reinforcement learning (RL) algorithms such as Temporal Difference Learning, FQI, etc.
%This algorithm, however, can be slow for problems with discount factor close to $1$, as its convergence rate is $O(\gamma^k)$.
%This paper considers the problem of acceleration of this algorithm, and proposes an algorithm that can have a faster convergence rate.

Consider a planning problem for a discounted MDP with dynamics $\PKernel$. Suppose that we have access to an approximate model $\PKernelhat \approx \PKernel$ as well.
For example, $\PKernel$ might be a high-fidelity, but slow, simulator, and $\PKernelhat$ is a lower-fidelity, but fast, simulator.
Or in the context of model-based reinforcement learning~(MBRL), $\PKernel$ is the unknown dynamics of a real-world system, from which we can only acquire expensive samples, and $\PKernelhat$ is a learned model, from which samples can be cheaply acquired.
% The approximate model might be available because of the availability of computationally cheap simulator (e.g., $\PKernel$ is a high-fidelity but expensi robotic simulator, $\PKernelhat$ is a low-fidelity simulator)
%
Can we use this approximate model $\PKernelhat$ to \emph{accelerate} the computation of the value function of a policy $\pi$ (Policy Evaluation (PE) problem) or the optimal value function (Control problem)?

The Value Iteration (VI) algorithm and its approximate variant are fundamental algorithms in Dynamic Programming that can find the (approximate) value of a policy or the optimal value function. They are also the backbone of many reinforcement learning (RL) algorithms such as Temporal Difference Learning~\citep{Sutton1988}, Fitted Value Iteration~\citep{Gordon1995,Ernst05,Munos08JMLR}, and Deep Q Network~\citep{MnihKavukcuogluSilveretal2015}.
Value Iteration, however, can be slow when the discount factor is close to $1$, as its convergence rate is $O(\gamma^k)$. Moreover, even though we could use VI using $\PKernelhat$ instead of $\PKernel$ to avoid expensive queries to $\PKernel$, the obtained value function would converge to a solution different from the value function of the original MDP.
\todo{Mention PI and that it has the same convergence rate?}

This paper proposes an algorithm called \emph{Operator Splitting Value Iteration} (OS-VI) that benefits from an approximate model $\PKernelhat$ to potentially accelerate the convergence of the value function sequence to the value function with respect to (w.r.t.) the true model $\PKernel$ (Section~\ref{sec:OSVI}).
This algorithm is for both PE (Section~\ref{sec:OSVI-OSVI-PE}) and Control (Section~\ref{sec:OSVI-OSVI-Control}) problems.
The acceleration is not uniform though, and depends on how close $\PKernelhat$ is to $\PKernel$ (Section~\ref{sec:Theory}).

%
%The acceleration is potential because the rate depends on how close 
%The acceleration 
%
%%, for both PE and Control problems,
%
%This paper considers the problem of acceleration of this algorithm, and proposes an algorithm that can have a faster convergence rate, and yet converge to the solution of the original problem.
%
%
%This paper considers the problem of acceleration of this algorithm, and proposes an algorithm that can have a faster convergence rate.

A key inspiration behind OS-VI is the (matrix) splitting approach in the numerical linear algebra, which is used to iteratively solve large linear systems of equations~\citep{Varga2000,Saad2003,GolubVanLoan2013}. With a proper choice of splitting, one may change the convergence rate of linear systems solvers.
We show that the conventional VI for PE can be seen as a particular choice of splitting. This observation suggests that one may choose other forms of splitting as well in order to change the convergence rate. It turns out that we can choose a splitting that benefits from having access to $\PKernelhat$ (Section~\ref{sec:VIandSplitting}).
The new splitting leads to OS-VI for PE.
For the Control problem, the connection between solving linear system of equations and VI is not as straightforward anymore, as the former is linear, while the latter is not, but we can still get inspired from the splitting approach to design OS-VI for Control. The key step of such an algorithm is a new \emph{policy improvement} step. % In conventional VI (and Policy Iteration and Modified Policy Iteration), we choose the greedy policy as the improved policy. That choice is reasonable for the conventional Bellman operator, and leads to an improved policy. In our new procedure, the  natural policy improvement step involves solving an auxiliary MDP with the dynamics of $\PKernelhat$.\todo{Refer to specific subsection.} \todo{Removed to save space! -AMF } XXX

The form of the OS-VI algorithm opens up a connection to MBRL where the approximate model $\PKernelhat$ is learned using data. This leads to the OS-Dyna algorithm, inspired by a generic Dyna architecture~\citep{Sutton1990}. OS-Dyna is a hybrid of model-free and model-based algorithms. It uses the learned model in its inner planning loop, alike Dyna, but uses samples from the true model $\PKernel$ in order to correct the effect of errors in the model.
Existing MBRL algorithms would converge to an incorrect solution if the approximate model $\PKernelhat$ does not converge to the true model $\PKernel$. This would be the case whenever model approximation error exists. On the other hand, OS-Dyna can still converge to the correct value function even when $\PKernelhat$ does not converge to $\PKernel$.
As far as we know, this is the first model-based RL algorithm with such property.\todo{Check Predictor-Corrector Policy Optimization (ICML 2019) though. It might be an exception. -AMF}

%%%%%%%%%%%%%%%%%%%%%%%%%%%%%%%%%%%%%%%%%%%%%%%
\section{From value iteration to splitting-based linear system of equations solvers and back}
 % I haven't completely decided on MSVI or OSVI, hence the back and forth! (AMF)
\label{sec:VIandSplitting}
% !TEX root =  OSVI.tex

We briefly describe the VI algorithm and the splitting methods for solving linear system of equations, and explain their connections. We consider a discounted Markov Decision Process (MDP) $(\XX, \AA, \RKernel, \PKernel, \gamma)$~\citep{Bertsekas96,SzepesvariBook10,SuttonBarto2018}.
We defer formal definitions to the supplementary material. We only mention that for a policy $\pi$, we denote by $\PKernelpi$ its transition kernel, by $r^\pi: \XX \ra \Real$ the expected value of its reward distribution, and by $\Vpi = \Vpi(\RKernel, \PKernel)$ its state-value function. We also represent the optimal state-value function by $\Vopt = \Vopt(\RKernel, \PKernel)$ and the optimal policy by $\pi^* = \pi^*(\RKernel, \PKernel)$. The Bellman operator $\Tpi: \BB(\XX) \ra \BB(\XX)$ for policy $\pi$ and the Bellman optimality operator $\Topt: \BB(\XX) \ra \BB(\XX) $ are\footnote{For countable state and action spaces, the integrals are replaced by summations. \revised{}{We present OS-VI and its theoretical analysis for general state/action spaces, but limit our experiments to finite state/action problems.}
%
%\revised{}{The OS-VI algorithm in Section~\ref{sec:OSVI} and its convergence analysis in Section~\ref{sec:Theory} are valid for the general state/action space problems, including continuous ones. The OS-Dyna algorithm in Section~\ref{sec:OSVI-Dyna} and the experiments in Section~\ref{sec:Experiments} are only presented for problems with finite state/action spaces.
}
\begin{align*}
	(\Tpi V)(x) \eqdef r^\pi(x) + \gamma \int \PKernelpi(\dy|x) V(y);
	\quad
	(\Topt V)(x) \eqdef \max_{a \in \AA} \left\{ r(x,a) + \gamma \int \PKernel(\dy|x,a) V(y) \right\}.
\end{align*}
%
%%%%%%%%%%%%%%%%%%%%%%%%%%%%%%%%%%%%%%%%%%%%%%%
These operators can be written more compactly as 
$\Tpi: V \mapsto \rpi + \gamma \PKernelpi V$
and
$\Topt: V \mapsto \max_{\pi} \{ \rpi + \gamma \PKernelpi V \}$.
The \emph{greedy} policy at state $x \in \XX$ is
\begin{align}
\label{eq:greedy-policy}
	\pigreedy(x;V) \leftarrow \argmax_{a \in \AA} \left\{ r(x,a) + \gamma \int \PKernel(\dy|x,a) V(y) \right\},
\end{align}
or more compactly, $\pigreedy(V) \leftarrow \argmax_{\pi} \Tpi V$.
We have $\Topt V = T^{\pigreedy(V)} V$, that is,
the effect of the Bellman optimality operator $\Topt$ applied to a value function $V$ is the same as applying the Bellman operator of the \emph{greedy} policy w.r.t. $V$ to $V$.
%
% Finally, we denote the $m$-step transition kernel of policy $\pi$ by ${\PKernelpi}^{(m)}$ and define the discounted future-state distribution as
% %
% \begin{align}
% \label{eq:Discounted-Future-State}
%     	\eta^\pi (\cdot | x) = (1 - \gamma) \sum_{m = 0}^{\infty} \gamma^m {\PKernelpi}^{(m)}(\cdot | x).
% \end{align}
% \vspace{-2em}
%%%%%%%%%%%%%%%%%%%%%%%%%%%%%%%%%%%%%%%%%%%%%%%
\subsection{Value Iteration}
\label{sec:VIandSplitting-ValueIteration}
The value function $\Vpi$ and the optimal value function $\Vopt$ are the fixed points of the operators $\Tpi$ and $\Topt$, respectively, and satisfy the Bellman equation. For the PE problem, this means that
\begin{align}
\label{eq:BellmanEquation-PE}
	\Vpi = \rpi + \gamma \PKernelpi \Vpi \Rightarrow (\Id - \gamma \PKernelpi) \Vpi = \rpi.
\end{align}

There are several ways to compute the value function of a policy $\pi$ or the optimal value function $\Vopt$, including the iterative methods such as VI and Policy Iteration (PI) algorithms, or solving a linear system of equations (for PE) or linear programming (for Control). We focus on the VI algorithm in this work. %~\citep{Bertsekas96}.
%
%In order to compute the value function of a policy $\pi$ or the optimal value function $\Vopt$, we can use the Value Iteration algorithm. 
VI repeatedly applies the Bellman operator to the most recent approximation of the value function: Given an initial value function $V_0$, it generates a sequence $(V_k)_{k \geq 0}$ as follows:
\begin{align}
\label{eq:VI}
	V_{k} \leftarrow 
		\begin{cases}
			\Tpi V_{k-1}, & \text{(Policy Evaluation)} \\
			\Topt V_{k-1}. & \text{(Control)}
		\end{cases}
\end{align}
VI for Control can be written in an equivalent form: At iteration $k$, we first compute the greedy policy $\pi_{k} \leftarrow \pigreedy(V_{k-1})$ (policy improvement step), and then $V_{k} \leftarrow T^{\pi_{k}} V_{k-1}$. 
Therefore, the policy improvement step is obtained through finding a policy that is greedy w.r.t. the last value function $V_{k-1}$, that is, the best policy if we only look one step ahead.
% Note that in VI for Control, we do not evaluate the policy completely; we only move towards it direction by one application of the Bellman operator $T^{\pi_{k+1}}$ on $V_k$. %The Policy Iteration algorithm, however, evaluates the policy.
This form will be conductive for our later developments. 
As the Bellman operator is a $\gamma$-contraction w.r.t. the supremum norm, the convergence rate of $V_k$ to $\Vpi$ (or $\Vopt$) would be $O(\gamma^k)$. This rate can be slow when $\gamma$ is  close to $1$. %, e.g., $0.99$ or $0.999$.

%%%%%%%%%%%%%%%%%%%%%%%%%%%%%%%%%%%%%%%%%%%%%%%
\subsection{Matrix splitting for solving linear system of equations}
\label{sec:VIandSplitting-MatrixSplitting}
The VI for PE can be seen as a (matrix) splitting-based iterative method to solve the linear system of equations~\eqref{eq:BellmanEquation-PE}.
Consider the linear system
$
	A z = b
$,
with $A \in \Real^{d \times d}$ and $z, b \in \Real^d$. Suppose that $A$ is decomposed to $A = M - N$ for some choices of $M, N \in \Real^{d \times d}$ (more generally, $A$, $M$, and $N$ can be linear operators). Therefore, $z$ satisfies
%\[
%A z = b \Rightarrow 
$M z = N z + b$.
%\]
The splitting-based iterative approach defines the new approximation $z_{k}$ given the current  $z_{k-1}$ by solving
\begin{align*}
%\label{eq:MatrixSplittingIteration} 
	M z_{k} = N z_{k-1} + b,
\end{align*}
or equivalently,
\begin{align}
\label{eq:MatrixSplittingIteration-InverseForm}
	z_{k} = M^{-1} (N z_{k-1} + b).
\end{align}
%Unless the matrix $M$ is the identity $\Id$, this requires solving for $z_{k+1}$, i.e.,
%\begin{align*}
%	z_{k+1} = M^{-1} (N z_k + b).
%\end{align*}
%
To analyze the convergence of this iterative method, consider the error $e_k \eqdef z_k - z$. As $z = M^{-1} (N z + b)$, we have $e_k = M^{-1} (N z_{k-1} + b) - M^{-1} (N z + b) = M^{-1} N (z_{k-1} - z)$, so the dynamics of the error is
%To analyze the convergence of this iterative method, consider the error $e_k \eqdef z_k - z$. As $M z = N z + b$, the dynamics of the error is
%Whether this procedure converges depends on the choice of $M$ and $N$.
%This iteration is convergent under certain conditions on $M$ and $N$.
%
\begin{align}
\label{eq:MatrixSplittingIteration-ErrorPropagation}
	e_{k} = M^{-1} N e_{k-1} = (M^{-1} N)^2 e_{k-2} = \cdots = (M^{-1} N)^k e_0.
\end{align}
Let $G \eqdef M^{-1} N$. % This dynamics depends on the behaviour of $G^k$.
The norm of the sequence $(e_k)_{k\geq1}$ can be upper bounded as
\begin{align}
\label{eq:MatrixSplittingIteration-ErrorNormUpperBounding}
	\norm{e_{k} } = \norm{ G^k e_0 } \leq \norm{G^k} \norm{e_0} \leq \norm{G}^k \norm{e_0}.
\end{align}
The errors are (norm-) convergent if $\smallnorm{G} = \smallnorm{M^{-1} N } < 1$, for some choice of norm. More generally, the necessary and sufficient condition for convergence is that the spectral radius $\rho(G)$, the maximum absolute value of eigenvalues of $G$, is smaller than one, see e.g., Theorem 4.1 of~\citet{Saad2003} or Theorem 11.2.1 of~\citet{GolubVanLoan2013}.%
\footnote{For any matrix norm, we have $\rho(G) \leq \norm{G}$, so the condition on the norm is sufficient, but not necessary. Our analysis will be norm-based.}
The convergence is faster if the spectral radius (or norm) is closer to zero.

The success of this iterative procedure depends on how we choose $M$ and $N$ such that the norm (or spectral radius) is as small as possible. Also we want to choose an $M$ such that solving $M z_{k} = N z_{k-1} + b$ is not very expensive. For example, if $M$ is an identity matrix $\Id$, we get that $N = \Id - A$, and the iteration becomes $z_{k} = (\Id - A) z_{k-1} + b$. 
This iteration is convergent if $\rho(\Id - A) < 1$.
%More generally, if $M$ is diagonal, computing its inverse is also cheap.
Other commonly used choices lead to the Jacobi and Gauss-Seidel methods that are described in the supplementary material. %These all show that are multiple ways to split $A$ to $M$ and $N$, each with their own convergence properties.

%Other choices of $M$ and $N$ are possible too.
%For example, if we decompose $A$ by its diagonal part $D$, its strictly lower triangular part $-L$, and its strictly upper triangular part $-U$ (so $A = D - L - U$), then 
%the choice of
%$M = D$ and $N = L + U$ leads to the Jacobi iteration. Clearly, the computation of $M^{-1} = D^{-1}$ is easy.
%If we select 
%$M = D - L$ and $N = U$ (or $M = D - U$ and $N = U$), we get the forward (or backward) Gauss-Seidel iteration.
%In all these cases, solving $M z_{k} = N z_{k-1} + b$ is easy.
%The convergence of these methods can be established too. For instance, if $A$ is strictly diagonally dominated, the Jacobi iteration is convergent (Theorem 11.2.2 of~\citet{GolubVanLoan2013}).
%%
%These examples, as well as other choices available in the numerical linear algebra literature such as the Successive Over Relaxation method, show that there are multiple ways to split $A$ to $M$ and $N$, each with their own convergence properties.

We are now ready to make the connection between splitting-based iterative methods and VI for PE. If we choose $A = \Id - \gamma \PKernelpi$, we see that equation $A \Vpi = \rpi$ is indeed the Bellman equation for policy $\pi$~\eqref{eq:BellmanEquation-PE}.
The VI for PE, which is $V_{k} = \gamma \PKernelpi V_{k-1} + \rpi = (\Id - A) V_{k-1} + \rpi$, corresponds to the choice of $M = \Id$ and $N = \gamma \PKernelpi$.
This brings up the question of whether it is possible to split $A$ differently so that the resulting VI-like procedure has better convergence properties.
%
%Some choices inspired by the Jacobi and Gauss-Seidel procedures have already been explored in the past~\citet{KushnerKleinman1971}.
We next suggest a particular choice. % in the next section.

%%%%%%%%%%%%%%%%%%%%%%%%%%%%%%%%%%%%%%%%%%%%%%%
\section{Operator splitting value iteration algorithm}
\label{sec:OSVI}
% !TEX root =  OSVI.tex
\newcommand{\Mpi}{{M^\pi}}
\newcommand{\Npi}{N^\pi}

\newcommand{\Spi}{S^\pi}
\newcommand{\Sopt}{{S^*}}

\newcommand{\rbar}{\bar{r}}
\newcommand{\Tbar}{\bar{T}}

We introduce the Operator Splitting Value Iteration (OS-VI) algorithm. We start from the PE problem and develop the Control version based on that. We also present a visualization of how OS-VI works.

%%%%%%%%%%%%%%%%%%%%%%%%%%%%%%%%%%%%%%%%%%%%%%%
\subsection{OS-VI for policy evaluation}
\label{sec:OSVI-OSVI-PE}

Given a policy $\pi$, true model $\PKernel$, and approximate model $\PKernelhat$, we split $\Id - \gamma \PKernelpi$ to $\Mpi$ and $\Npi$ as
\begin{align*}
	\Mpi = \Id - \gamma \PKernelhat^\pi \qquad , \qquad \Npi = \gamma (\PKernelpi - \PKernelhat^\pi).
\end{align*}
Following the recipe of~\eqref{eq:MatrixSplittingIteration-InverseForm}, the OS-VI algorithm for PE would be
\begin{align}
\label{eq:OSVI-PE-MandNForm}
	V_{k} \leftarrow
	(\Id - \gamma \PKernelhat^\pi)^{-1} \left[\rpi + \gamma (\PKernelpi - \PKernelhat^\pi) V_{k-1}\right],
%	=
%	(\Id - \gamma \PKernelhat^\pi)^{-1} (\rpi + \gamma (\PKernelpi - \PKernelhat^\pi) V).
\end{align}
starting from an initial $V_0$.%
\footnote{Although splitting is originally studied mostly in the context of linear algebra and matrices, we are applying the idea more generally. We are not assuming that the state space $\XX$ is finite, and we allow it to be more general, such as a subset of $\Real^d$. Consequently, $\Mpi$, $\Npi$, $\PKernelpi$, etc. are operators rather than matrices.}
To gain more intuition and prepare for further developments, we define a few notations. We define the
\textit{Varga operator} $\Spi: \BB(\XX) \ra \BB(\XX)$, named after Richard S Varga (1928 -- 2022) who has made significant contributions to matrix analysis, as the mapping between the space of all bounded functions over $\XX$ to the same space as
\begin{align*}
%\label{eq:S-Operator-policy}
	\Spi: V \mapsto (\Id - \gamma \PKernelhat^\pi)^{-1} \left[\rpi + \gamma (\PKernelpi - \PKernelhat^\pi) V \right].
\end{align*}

Observe that~\eqref{eq:OSVI-PE-MandNForm} can be compactly written as
\begin{align}
\label{eq:OSVI-PE-SpiForm}
	V_{k} \leftarrow \Spi V_{k-1}.
\end{align}
%that is, at each iteration we apply the $\Spi$ operator to the last estimate of the value function.
It is not difficult to see that $\Spi \Vpi = \Vpi$, i.e., the value function of a policy $\pi$ is a fixed-point of the Varga operator $\Spi$. This and other properties of the Varga operator are shown in the supplementary material.

Given any value function $V$, define an auxiliary reward function $\rbar_V: \XA \ra \Real$ as 
\begin{align}
\label{eq:rbar-definition}
	\rbar_V(x,a) \eqdef 
	r(x,a) + \gamma \int \left( \PKernel(\dy | x,a) - \PKernelhat(\dy | x,a) \right) V(y).
\end{align}
Similar to $\rpi$, we define the notation $\rbar^\pi_V: \XX \ra \Real$ as $\rbar^\pi_V(x) = \rbar_V(x,\pi(x))$ for a deterministic policy $\pi$ (and similarly for a stochastic policy).
With this notation, the effect of applying $\Spi$ to $V$ is
\[
\Spi V = (\Id - \gamma \PKernelhat^\pi)^{-1} \rbar^\pi_V.
\]
This is the value function of following $\pi$ in an MDP with dynamics $\PKernelhat$ and reward $\rbar_V$. Therefore, at each iteration of OS-VI (PE), we evaluate the policy $\pi$ according to the approximate dynamics, and a reward function that consists of the original reward $r$ and the correction term $\gamma (\PKernel - \PKernelhat) V_{k-1}$.
The computation of this value function is a standard PE problem with the approximate model. For instance, we may use another VI (PE) with dynamics $\PKernelhat$ to solve it: We initialize $U_0 \leftarrow V$, and then for $i \geq 1$, we set
%
%\begin{align*}
%	& 
$	U_{i} \leftarrow \rbar^\pi_V + \gamma \PKernelhat^\pi U_{i-1}$.
%\end{align*}
This converges to $\Spi V = (\Id - \gamma \PKernelhat^\pi)^{-1} \rbar^\pi_V$ with the usual rate of $O(\gamma^i)$.
Note that aside from the computation of $\rbar^\pi_V$, which requires querying $\PKernel$ in order to compute the $\PKernelpi V$ term, this iterative process only uses the approximate model $\PKernelhat$, which is assumed to be cheap to access.

What is the benefit of this OS-VI procedure? If the approximate model $\PKernelhat$ is close to the true dynamics $\PKernel$, this leads to a faster convergence of $V_k$ to $\Vpi$, as shall be quantified soon.
The faster convergence is in terms of the number of queries to $\PKernel$, which is assumed to be expensive. To see this, consider the hypothetical case that $\PKernelhat$ is exactly the same as $\PKernel$, for example, if the cheap simulator happens to perfectly match the reality. Then, $\Spi V = (\Id - \gamma \PKernelpi)^{-1} (\rpi + 0 V) = \Vpi$, and the value function for the original MDP is obtained in one iteration of OS-VI.
Of course, we often can only hope for $\PKernelhat \approx \PKernel$. In Section~\ref{sec:Theory}, we study the impact of error in $\PKernelhat$ on the convergence rate of OS-VI in more details, and show that the convergence of OS-VI can be much faster than classic algorithms even if $\PKernelhat$ is only a close approximation of $\PKernel$.
\subsection{OS-VI for control}
\label{sec:OSVI-OSVI-Control}

The VI for Control can be seen as an iterative procedure that computes the greedy policy $\pi_{k} \leftarrow \pigreedy(V_{k-1}) = \argmax_\pi \Tpi V_{k-1}$ in its policy improvement step, and then uses one step of the Bellman operator w.r.t. the obtained policy $\pi_{k}$ to compute the new estimate of the value function $V_{k} \leftarrow T^{\pi_{k}} V_{k-1}$, as described after~\eqref{eq:VI}.
The OS-VI algorithm for Control follows a similar structure with the difference that {\bf 1)} the improved policy is the optimizer of the Varga operator, and not the Bellman operator, and {\bf 2)} the new value function is obtained by applying the Varga operator of the newly obtained policy. To be concrete, given a value function $V$, define the \emph{$S$-improved policy}
\begin{align}
\label{eq:S-improved-policy}
	\pi_V(V) \defeq \argmax_{\pi} \Spi V [=  (\Id - \gamma \PKernelhat^\pi)^{-1} \rbar^\pi_V].
\end{align}
 This policy is the \emph{optimal} policy for the auxiliary MDP $(\XX, \AA, \rbar_V, \PKernelhat, \gamma)$.
% that has the dynamics of $\PKernelhat$ and the reward function $\rbar_V$, i.e., $(\XX, \AA, \rbar_V, \PKernelhat, \gamma)$.
%
%
We also define the \textit{Varga optimality operator} $\Sopt: \BB(\XX) \ra \BB(\XX)$ as
\begin{align*}
%\label{eq:S-Operator-Optimality}
	\Sopt: V \mapsto \max_\pi \Spi V. %[= \max_\pi (\Id - \gamma \PKernelhat^\pi)^{-1} \rbar^\pi_V].
\end{align*}
The function $\Sopt V$ is equal to $S^{\pi_V(V)} V$, i.e., the Varga operator of the $S$-improved policy w.r.t. $V$ applied to a value function $V$ (compare it with $\Topt V = T^{\pigreedy(V)} V$).

The OS-VI (Control) is then simply
\begin{align}
\label{eq:OSVI-Control}
	V_{k} \leftarrow \Sopt V_{k-1},
\end{align}
which in its expanded form, consists of the following two steps:
\begin{align}
%\label{eq:OSVI-Control-Expanded}
%\nonumber
\label{eq:OSVI-Control-Expanded-PolicyImprovement}
	& \pi_{k} \leftarrow \pi_V(V_{k-1}),   \qquad \text{(policy improvement)}. \\
\label{eq:OSVI-Control-Expanded-PartialPE}	
	& V_{k} \leftarrow S^{\pi_{k}} V_{k-1} \; [= (\Id - \PKernelhat^{\pi_{k}})^{-1} (r^{\pi_{k}} + \gamma (\PKernel^{\pi_{k}} - \PKernelhat^{\pi_{k}}) V_{k-1})], \text{(partial policy evaluation)}.
\end{align}

Comparing the $S$-improved policy~\eqref{eq:S-improved-policy} used in the policy improvement step~\eqref{eq:OSVI-Control-Expanded-PolicyImprovement} of OS-VI with the conventional greedy policy~\eqref{eq:greedy-policy} is insightful. The greedy policy is $\argmax_\pi \Tpi V$. Expanding $\Tpi V$, we see that the greedy policy is the maximizer of $\rpi + \gamma \PKernelpi V$.
The function $\rpi + \gamma \PKernelpi V$ is a single-step bootstrapped estimate of the value of $\Vpi$, and its maximizer, the greedy policy, is in general different from the optimal policy, which maximizes $\Vpi$.
On the other hand, the $S$-improved policy $\pi_V(V)$ solves a full MDP with an approximate model $\PKernelhat$ and a reward function that has both the original reward $r$ and the correction term $\gamma (\PKernel - \PKernelhat) V$. In the special case that  $\PKernelhat = \PKernel$, the correction term is zero, and $\pi_V(V)$ would be the optimal policy $\piopt$ for the original MDP.
As often $\PKernelhat \approx \PKernel$, the value function of policy $\pi_V(V)$ is not exactly the optimal value. The partial policy evaluation step~\eqref{eq:OSVI-Control-Expanded-PartialPE} updates the value function to a value that is closer to the optimal value function.

%%%%%%%%%%%%%%%%%%%%%%%
\begin{remark}
The use of matrix splitting-based ideas, either explicitly or implicitly, in the context of dynamic programming is not completely novel to this work.
\citet{KushnerKleinman1971} is one of the earliest paper that mentions the Jacobi and Gauss-Seidel procedures for computing the value function.
\citet{Porteus1975} proposes several transformations to the reward and probability transition matrix with the goal of improving the computational cost of solving the transformed MDP. One of the transformations, called \emph{pre-inverse transform}, has some similarities with the operator splitting of this work. The end result, however, is different.
\citet{BaconPrecup2016} offer a matrix splitting perspective on planning with options.
The connection between multi-step models and matrix splitting is developed in Chapter 4 of \citet{Bacon2018}.
Refer to the supplementary material for more discussion.
\end{remark}

%%%%%%%%%%%%%%%%%%%%%%%%%%%%%%%%%%%%%%%%%%%%%%%
\subsection{Visualizing how OS-VI works}
\label{sec:Algorithms-ExtendedDiscussion-Visualization}
To present some intuition on how OS-VI works, we visualize the value function trajectories of several algorithms, including OS-VI, on a 2-state MDP, in Figure~\ref{fig:OSVI-PE-Visualiziation}.
We consider the policy evaluation for the dynamics $\PKernelpi = \left[ \begin{smallmatrix} 0.9 & 0.1 \\ 0.1 & 0.9 \end{smallmatrix} \right]$ with the reward $\rpi = \left( \begin{smallmatrix} 1 \\ -0.5 \end{smallmatrix} \right)$ and $\gamma = 0.9$.
We consider two approximate models: a relatively accurate $\PKernelhat_\text{accurate}^\pi = \left[ \begin{smallmatrix} 0.85 & 0.15 \\ 0.05 & 0.95 \end{smallmatrix} \right]$, and an inaccurate $\PKernelhat_\text{inaccurate}^\pi = \left[ \begin{smallmatrix} 0.6 & 0.4 \\ 0.3 & 0.7 \end{smallmatrix} \right]$.
%The results are shown in Figure~\ref{fig:OSVI-PE-Visualiziation}.

%We visualize the value function trajectories of several algorithms, including OS-VI, on a 2-state MDP in order to provide some intuition on how OS-VI works.
%%
%We consider the policy evaluation for the dynamics $\PKernelpi = \left[ \begin{smallmatrix} 0.9 & 0.1 \\ 0.1 & 0.9 \end{smallmatrix} \right]$ with the reward $\rpi = \left( \begin{smallmatrix} 1 \\ -0.5 \end{smallmatrix} \right)$ and $\gamma = 0.9$.
%We consider two approximate models: a relatively accurate one $\PKernelhat_\text{accurate}^\pi = \left[ \begin{smallmatrix} 0.85 & 0.15 \\ 0.05 & 0.95 \end{smallmatrix} \right]$, and an inaccurate one $\PKernelhat_\text{inaccurate}^\pi = \left[ \begin{smallmatrix} 0.6 & 0.4 \\ 0.3 & 0.7 \end{smallmatrix} \right]$.
%The results are shown in Figure~\ref{fig:OSVI-PE-Visualiziation}.

In addition to OS-VI (PE), the first algorithm is the conventional VI (PE), which repeatedly applies the Bellman operator according to the true model $\PKernelpi$ to the most recent approximation of the value function. We use $\Tpi_\PKernel$ to refer to its Bellman operator and to label the corresponding trajectory in the value space. This algorithm converges to the true value function $\Vpi$. The second algorithm is VI (PE) that uses $\PKernelpihat$ as the model. This procedure is the basis of the Dyna architecture. 
We use $\Tpi_{\PKernelhat}$ to refer to its Bellman operator and to label the corresponding trajectory in the value space.
Due to the error of $\Tpi_{\PKernelhat}$ compared to $\Tpi_\PKernel$, the algorithm converges to a wrong value function $\hat{V}^\pi$, as both figures show. We observe that even when the model is relatively accurate as in Figure~\ref{fig:OSVI-PE-Visualiziation-MoreAccurate}, the converged value function is quite wrong. This illustrates one limitation of the conventional model-based RL algorithms where an inaccurate model may lead to significantly inaccurate estimate of the value function.

%The first algorithm is the conventional VI (PE), which repeatedly applies the Bellman operator according to the true model $\PKernelpi$ to the most recent approximation of the value function. We use $\Tpi_\PKernel$ to refer to its Bellman operator and to label the corresponding trajectory in the value space. This algorithm converges to the true value function $\Vpi$.
%
%
%The second algorithm is VI (PE) that uses $\PKernelpihat$ as the model. This procedure is the basis of the Dyna architecture. 
%We use $\Tpi_{\PKernelhat}$ to refer to its Bellman operator and to label the corresponding trajectory in the value space.
%%
%Due to the error of $\Tpi_{\PKernelhat}$ compared to $\Tpi_\PKernel$, the algorithm converges to a wrong value function $\hat{V}^\pi$, as both figures show. We observe that even when the model is relatively accurate as in Figure~\ref{fig:OSVI-PE-Visualiziation-MoreAccurate}, the converged value function is quite wrong. This illustrates one limitation of the conventional model-based RL algorithms where an inaccurate model may lead to significantly inaccurate estimate of the value function.

The OS-VI algorithm repeatedly applies the Varga operator $\Spi$ to the most recent approximation of the value function. As discussed earlier, each computation of $\Spi V_{k-1}$ corresponds to solving an auxiliary MDPs $(\XX, \AA, \rbar_{V_{k-1}}, \PKernelhat, \gamma)$. %where the dynamics is the same as $\PKernelhat$ used by the conventional model-based approach but the corrected reward is $\rbar^\pi_{V_{k-1}} = \rpi + \gamma (\PKernelpi - \PKernelpihat) V_{k-1}$~\eqref{eq:rbar-definition}.
We denote the Bellman operator of this auxiliary MDP by $\Tbar_k^\pi$.
The figures show the trajectory generated by the iterative application of $\Spi$ on the most recent value function as well as the trajectory for solving each auxiliary MDP, indicated by $\Tbar_k^\pi$.
We observe that the OS-VI algorithm converges to the correct value function despite using an incorrect model.
When the model is more accurate, very few iteration of OS-VI gets a value close to $\Vpi$ (two iterations in Figure~\ref{fig:OSVI-PE-Visualiziation-MoreAccurate}); when the model is less accurate, a few more iterations are needed. Compared to VI, at least in these examples, the total number of iterations of OS-VI is significantly smaller.
%
%
%Note that a key practical difference between OS-VI and VI is that OS-VI does not use the expensive simulator $\PKernelpi$ to propagate rewards in the state space. This propagation takes many steps if the discount factor $\gamma$ is close to $1$. Instead, OS-VI uses the cheap $\PKernelpihat$ for this task, and only uses $\PKernelpi$ to fix the errors.

When the initial value function is $V_0 = 0$, the result of the first iteration of OS-VI is the same value function computed by the VI with the wrong model $\PKernelpihat$. 
This is because $V_1 \leftarrow \Spi V_0 = (\Id - \gamma \PKernelpihat)^{-1} \rbar^\pi_{V_0}$ and $\rbar^\pi_{V_0} = \rpi$, so $V_1 = (\Id - \gamma \PKernelpihat)^{-1} \rpi$, the same solution as the value function obtained using the approximate model $\PKernelhat$.
In these figures, this shows itself by the overlapping of the red arrows followed by $\Tpi_{\PKernelhat}$ and the first segment of the orange arrows, which are generated by the repeated application of $\bar{T}_1^\pi$.
In further iterations of OS-VI, the auxiliary MDPs change and the path followed by 
$\bar{T}_{k}^\pi$ ($k \geq 2$) deviates from the solution of the VI with the wrong model.
%
%We next study some theoretical properties of OS-VI.
%To deepened our understanding of OS-VI, we theoretically analyze its convergence properties in the next section. 

%%%%%%%%%%%%%%%%%%%%%
\begin{figure}[t]
\centering
      \begin{subfigure}[b]{0.49 \linewidth}
        \centering
	\includegraphics[width= \textwidth]{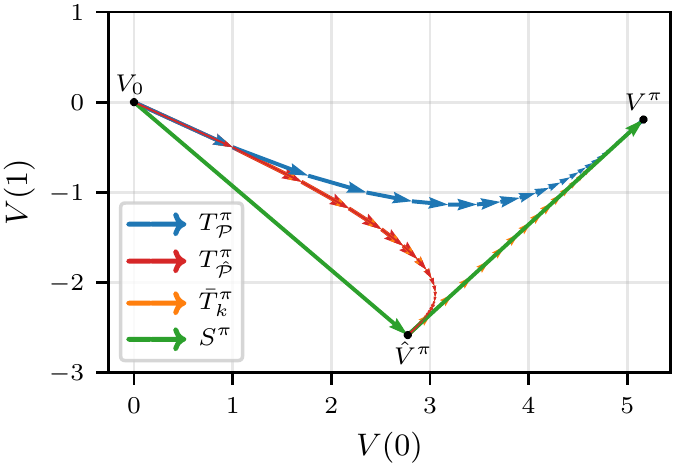}
        \caption{More accurate model $\PKernelhat_\text{accurate}$ } 
        \label{fig:OSVI-PE-Visualiziation-MoreAccurate} 
      \end{subfigure}
      \begin{subfigure}[b]{0.49 \linewidth}
        \centering
	\includegraphics[width= \textwidth]{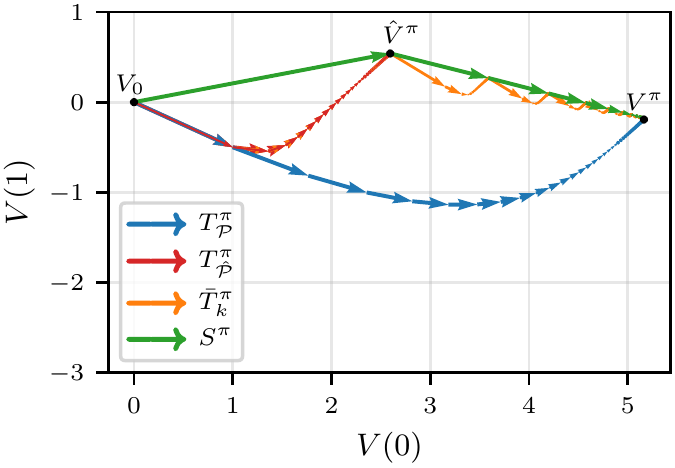}
        \caption{Less accurate model $\PKernelhat_\text{inaccurate}$ } 
        \label{fig:OSVI-PE-Visualiziation-LessAccurate} 
      \end{subfigure}%%      
\caption[short]{The value function trajectories of VI (PE) with the true model ($\Tpi_\PKernel$), VI (PE) with approximate model $\Tpi_{\PKernelhat}$, and OS-VI (PE) ($\Spi$; and $\bar{T}_k$ for its inner loop) for a 2-state problem.}
\label{fig:OSVI-PE-Visualiziation}
\end{figure}
%%%%%%%%%%%%%%%%%%%%%

%%%%%%%%%%%%%%%%%%%%%%%%%%%%%%%%%%%%%%%%%%%%%%%
\section{Convergence analysis of operator splitting value iteration}
\label{sec:Theory}
% !TEX root =  OSVI.tex

% \todo{We can follow this plan. I will probably start writing parts of it on Monday, but feel free to write down the draft. -AMF}
% \begin{itemize}
% 	\item High-level description of what we have in one paragraph
% 	\item The derivation of OS-VI (PE) with the supremum norm, as is copied below.
% 	\item If possible, report the result for OS-VI (Control) with errors at each step. This can be for the supremum norm (so effectively, AOS-VI -- which is a strange name). Before reporting, we need to define the error model.
% 	\item Mention or report the $L_p$-norm results with or without error for PE and/or Control.
	
% 	\item We have to compare with the standard convergence rates for VI, PI, and MPI. OS-VI can be much faster, so we have to highlight it. Maybe we can also compare with PID VI and other accelerated methods too.
% \end{itemize}

In this section, we present the convergence analysis of OS-VI. Our results show that OS-VI has an $O(\gamma'^k)$ convergence rate for an effective discount factor $\gamma'$ that depends on the error between $\PKernelhat$ and $\PKernel$. For small enough error, $\gamma' < \gamma$ and OS-VI has a faster convergence rate compared to the classic VI, Policy Iteration (PI), and Modified Policy Iteration (MPI), which all have $O(\gamma^k)$ behaviour.
We provide results for both the $L_\infty$ and $L_p$ norms.

%We give two sets of results for the convergence of OS-VI in presence of approximation errors. The first set of results is done with the $L_\infty$ norm. While being simple and easy to understand, this analysis explains the convergence behaviors of OS-VI very well. The drawback of this analysis is that the $L_\infty$ norm may be too strict in evaluation of errors. To address this issue, we also give an analysis in $L_p$ norm. 

%%% XXX ORIGINAL XXX
%In this section, we present the convergence analysis of OS-VI. Our results show that if $\PKernelhat$ is a close approximation of the true transition kernel $\PKernel$, OS-VI has a faster convergence rate than classic VI, Policy Iteration (PI), and Modified Policy Iteration (MPI). We give two sets of results for the convergence of OS-VI in presence of approximation errors. The first set of results is done with the $L_\infty$ norm. While being simple and easy to understand, this analysis explains the convergence behaviors of OS-VI very well. The drawback of this analysis is that the $L_\infty$ norm may be too strict in evaluation of errors. To address this issue, we also give an analysis in $L_p$ norm. 

%%%%%%%%%%%%%%%%%%%%%%%%%%%%%%%%%%%%%%%%%%%%%%%
\subsection{Convergence of OS-VI for policy evaluation}
\label{sec:Theory-PE}

We study the convergence behaviour of OS-VI (PE) in presence of error in value updates.
%Recall that each iteration of OS-VI (PE) for PE problem computes $V_{k} \leftarrow \Spi V_{k-1}$, which is equivalent of solving a PE problem in the auxiliary MDP $(\XX, \AA, \rbar_V, \PKernelhat, \gamma)$.
%In practice, this auxiliary MDP might be solved only approximately, and as a result we incur some errors in performing each iteration.
Specifically, we consider that at each iteration $k$, the update \eqref{eq:OSVI-PE-SpiForm} has an error, i.e.,
\begin{align}
\label{eq:AOS-VI-PE}
	V_{k} = S^\pi V_{k-1} + \errv{k}
\end{align}
The error $\errv{k}$ encompasses various sources of errors that might occur in solving the auxiliary MDP $(\XX, \AA, \rbar_{V_{k-1}}, \PKernelhat, \gamma)$.
One source is the function approximation error due to using a function approximator to represent $V_k$, which is often required in large state spaces. Another is the estimation (i.e., statistical) error caused due to having a finite number of samples, instead of direct access to $\PKernel$, in the RL setting.
Refer to~\citet{Munos08JMLR,AntosSzepesvariML08,FarahmandGhavamzadehSzepesvariMannor2016,ChenJiang2019,FanWangXieYang2019} for the discussion of function approximation and estimation errors in the RL context.
%Refer to~\citet{Gyorfi02,SteinwartChritmann2008} for the discussion of function approximation and estimation errors in the supervised learning context and to~\citet{Munos08JMLR,AntosSzepesvariML08,FarahmandGhavamzadehSzepesvariMannor2016,ChenJiang2019,FanWangXieYang2019} for their analysis in the RL context.
In this work, we do not analyze how the number of samples, the function approximator, etc. affect  the errors $\errv{k}$. We offer error propagation results similar to~\citet{Munos07} (approximate VI),~\citet{Munos03} (approximate PI), and ~\citet{ScherrerGhavamzadehGabillonetal2015} (approximate modified PI) for approximate OS-VI.

To study the convergence behaviour of OS-VI (PE), let $\Gpi = (\Id - \gamma \PKernelhat^\pi)^{-1} \gamma (\PKernelpi - \PKernelhat^\pi) $. We use the fact that $S^\pi V^\pi = V^\pi$ and write
\begin{align}
\label{eq:OSVI-PE-ErrorNormUpperBounding-Exact}
\notag
\norm{\Vpi - V_{k}}_\infty &= 
\norm{\Spi\Vpi - \Spi V_{k-1} - \errv{k}}_\infty =
\norm{\Gpi (\Vpi - V_{k-1}) - \errv{k}}_\infty \\
&\le \norm{\Gpi}_\infty \norm{\Vpi - V_{k-1}}_\infty + \norm{\errv{k}}_\infty.
\end{align}
Now, we have that
\begin{align}
\label{eq:OSVI-PE-NormG-UB}
	\norm{\Gpi}_\infty = 
	\norm{ (\Id - \gamma \PKernelhat^\pi)^{-1} \gamma (\PKernelpi - \PKernelhat^\pi) }_\infty
	\leq
	\frac{\gamma}{1 - \gamma} \norm{ \PKernelpi - \PKernelhat^\pi }_\infty,
\end{align}
where we used the fact that for any square matrix $F$ with a matrix norm $\norm{ F }_p < 1$, it holds that
$\smallnorm{ (\Id - F)^{-1} }_p \leq \frac{1}{1 - \smallnorm{F}_p}$ (see Lemma 2.3.3 of~\citealt{GolubVanLoan2013}), and that the supremum norm of a stochastic matrix $\PKernelhat^\pi$ is $1$.
Assuming that $\smallnorm{\errv{k}}_\infty \le \errvm$ for all $k \ge 1$ and defining effective discount factor $\gamma' = \frac{\gamma}{1 - \gamma} \smallnorm{ \PKernelpi - \PKernelhat^\pi }_\infty$, the upper bounds~\eqref{eq:OSVI-PE-ErrorNormUpperBounding-Exact} and~\eqref{eq:OSVI-PE-NormG-UB} lead to 
$\smallnorm{\Vpi - V_{k} }_\infty \leq \gamma'^k \smallnorm{\Vpi - V_0 }_\infty
	+ \frac{1 - \gamma'^k}{1 - \gamma'}  \errvm$.

%%%%%%%%%%%%%%%%%%%%%%%%
%\begin{proposition}
%	\label{prop:PE-infnorm}
%Consider the approximate OS-VI algorithm for PE~\eqref{eq:AOS-VI-PE}. Assume that $\smallnorm{\errv{k}}_\infty \le \errvm$ for all $k \ge 1$. Define the effective discount factor $\gamma' = \frac{\gamma}{1 - \gamma} \smallnorm{ \PKernelpi - \PKernelhat^\pi }_\infty$.
%For any $k \geq 0$, we have
%%
%\begin{align*}
%	\norm{\Vpi - V_{k} }_\infty \leq \gamma'^k \norm{\Vpi - V_0 }_\infty
%	+ \frac{1 - \gamma'^k}{1 - \gamma'} \cdot \errvm.
%\end{align*}
%
%\end{proposition}
%%%%%%%%%%%%%%%%%%%%%%%%

A few remarks are in order.
First, whenever $\gamma' < \gamma$, this is guaranteed to be faster than the convergence rate of the conventional VI, which is $O(\gamma^k)$. This happens if $\smallnorm{ \PKernelpi - \PKernelhat^\pi }_\infty < 1 - \gamma$.
If the model is very accurate, we obtain much faster rate than VI's.
Since each iteration $k$ corresponds to a query to the true model $\PKernel$, a faster rate entails that the algorithm requires fewer number of queries to the expensive model to reach the same level of accuracy.
%
% Therefore, if the approximate model $\PKernelhat$ is accurate enough, we can benefit from this new procedure.

Second, although the model error $\smallnorm{ \PKernelpi - \PKernelhat^\pi }_\infty$ is a reasonable choice to measure the distances between distributions (it is the maximum of the Total Variation distance between $\PKernelpi(\cdot|x)$ and $\PKernelhat^\pi(\cdot|x)$ over  $x$, which itself can be upper bounded by their $\KL$ divergence; see the supplementary material), it is somewhat conservative as it takes the supremum over the state space.
Likewise, requiring $\smallnorm{\errv{k}}_\infty$ to be small is also conservative, as approximating $\Spi V_{k-1}$ using a function approximator given samples (RL setting) often leads to an $L_p$-norm type of guarantee.
%This makes dependence on $\smallnorm{\errv{k}}_\infty$ conservative. 
We now provide a different analysis to address these issues.

To present the $L_p$-norm result, we need to define some notations. First, we define the conditional discounted future-state distribution of policy $\pi$ under $\PKernelhat$ as the following probability distribution:
Given a measurable set $B$, we have
$\hat \eta^\pi (B | x) = (1 - \gamma) \sum_{t = 0}^{\infty} \gamma^t \Pr{X_t \in B | X_0 = x, \pi, \PKernelhat}$,
where the chain $(X_t)_{t \geq 0}$ starts from state $x$ and evolves by following policy $\pi$ under transitions $\PKernelhat$. 
% This is the same as ~\eqref{eq:Discounted-Future-State} with the change of $\PKernel$ to $\PKernelhat$.
%
For an arbitrary distribution $\rho$ over the state space, we define the \textit{discounted future-state distribution concentration coefficient of the approximate model} as
\begin{align}
	\hat C^\pi(\rho)^2 = \frac{1}{\gamma^2}\int \rho(\dx)  \norm{ \frac{\mathrm{d} \hat \eta^\pi (\cdot | x)}{\drho}}_\infty^3.
\end{align}
Here $\frac{\mathrm{d} \hat \eta^\pi (\cdot | x)}{\drho}$ is the Radon–Nikodym derivative of the distribution $\hat \eta^\pi (\cdot | x)$ w.r.t. the distribution $\rho$. It is assumed that for any $x \in \XX$, $\etahatpi(\cdot|x) \ll \rho$, i.e., $\etahatpi(\cdot|x)$ is absolutely continuous w.r.t. $\rho$ (otherwise, the coefficient would be set to infinity).
%
%\begin{align}
%	\hat C^\pi(\rho)^2 = \frac{1}{\gamma^2}\int \rho(\dx) \left(\max_y \frac{	\hat \eta^\pi (\dy | x)}{\rho(\dy)}\right)^3.
%\end{align}
This coefficient measures how concentrated the distribution $\hat \eta^\pi (\cdot | x)$ is compared to $\rho$. This is weighted according to the state distribution $\rho$.
Similar concentrability coefficients, but not exactly this one, have appeared in the error propagation results~\citep{KakadeLangfordCPI2002,Munos03,Munos07,FarahmandMunosSzepesvari10,ScherrerGhavamzadehGabillonetal2015}.
%
%Starting from state $x$ covered in $\rho$, this coefficient is about how well the states reachable under $\PKernelhat^\pi$ (measured by $\hat \eta^\pi (\cdot | x)$) are also covered in $\rho$ (measured by $\rho(\dy)$). If $\rho(\dx) > 0$ and $	\hat \eta^\pi (\dy | x) > 0$, i.e. $y$ is reachable from $x$ under $\PKernelhat^\pi$, implies $\rho(\dy) > 0$ this coefficient is finite.
%\todo{I commented this as I believe it needs a bit more explanation to avoid being confusing. I may try to fix it later. -AMF}
%
Finally, we define the weighted $\chi^2$-divergence of $\PKernelhat^\pi$ and $\PKernel^\pi$ as 
\begin{align*}
\chi^2_{\rho}(\PKernel^\pi \;||\; \PKernelhat^\pi ) \defeq \int \rho(\dx) \chi^2
\left( \PKernelpi(\cdot | x)\;||\; \PKernelhat^\pi (\cdot | x) \right) =
\int \rho(\dx) 
 \int \frac{\left|\PKernelhat^\pi(\dy | x) - \PKernel^\pi (\dy | x) \right|^2}{\PKernelhat^\pi(\dy | x)}.
\end{align*}

This notion of model error is less strict in requiring accurate approximation $\PKernel$ in all states. Usually only a subset of the state space is important or even reachable in a problem. The above model error can focus on only specific areas of the state space through the choice of distribution $\rho$.

We are now ready to present the main theorem for the approximate OS-VI (PE).

%%%%%%%%%%%%%%%%%%%%%%%%
\begin{theorem}
\label{theorem:ErrorPropagation-SupAndLp-PE}
	Consider the approximate OS-VI algorithm for PE~\eqref{eq:AOS-VI-PE}.
Let $\norm{\cdot}_\star$ be either the supremum norm $\norm{\cdot}_\infty$ ($\star = \infty$) or $\norm{\cdot}_{4, \rho}$ for $\rho$ being an arbitrary distribution over the state space ($\star = 4, \rho$).
Assume that $\smallnorm{\errv{k}}_\star \le \errvm$ for all $k \ge 1$.
Furthermore, define the effective discount factor as
\begin{align*}
	\gamma' =
		\frac{\gamma}{1 - \gamma}
		\begin{cases}
			\norm{ \PKernelpi - \PKernelhat^\pi }_\infty  & (\star = \infty),
			\\
			\sqrt{ \hat C^\pi(\rho)  \chi^2_{\rho}(\PKernel^\pi \;||\; \PKernelhat^\pi ) } & (\star = 4, \rho).
		\end{cases}
\end{align*}
	For any $k \geq 0$, we have
$
	\norm{\Vpi - V_{k} }_{\star} \leq \gamma'^k \norm{\Vpi - V_0 }_{\star}
	+ \frac{1 - \gamma'^k}{1 - \gamma'} \cdot \errvm$.	
%	\begin{align*}
%	\norm{\Vpi - V_{k} }_{\star} \leq \gamma'^k \norm{\Vpi - V_0 }_{\star}
%	+ \frac{1 - \gamma'^k}{1 - \gamma'} \cdot \errvm
%	\end{align*}
%% XXX ORIGINAL XXX
%Consider the approximate OS-VI algorithm for PE~\eqref{eq:AOS-VI-PE}. Assume that $\smallnorm{\errv{k}}_\infty \le \errvm$ for all $k \ge 1$. Define the effective discount factor $\gamma' = \frac{\gamma}{1 - \gamma} \smallnorm{ \PKernelpi - \PKernelhat^\pi }_\infty$.
%For any $k \geq 0$, we have
%%
%\begin{align*}
%	\norm{\Vpi - V_{k} }_\infty \leq \gamma'^k \norm{\Vpi - V_0 }_\infty
%	+ \frac{1 - \gamma'^k}{1 - \gamma'} \cdot \errvm.
%\end{align*}
%
\end{theorem}
%%%%%%%%%%%%%%%%%%%%%%%%

%%%%%%%%%%%%%%%%%%%%%%%%%%%%%%%%%%%%%%%%%%%%%%%
\subsection{Convergence of OS-VI for control}
\label{sec:Theory-Control}

We turn to analyzing OS-VI for Control.
We consider two types of errors: The first is the error between the computed value function and the true optimal value function of the auxiliary MDP, i.e., $V_{k} - \Sopt V_{k-1}$. The second is the suboptimality of obtained policy compared to the optimal policy of the auxiliary MDP, i.e., $S^{\pi_k}V_{k-1} - \Sopt V_{k-1}$. Concretely, we have
% SECOND VERSION
%We turn to analyzing OS-VI for Control.
%Recall that each iteration of OS-VI for Control computes $V_{k} \leftarrow \Sopt V_{k-1}$, which requires the computation of the optimal policy and value function of the auxiliary MDP $(\XX, \AA, \rbar_V, \PKernelhat, \gamma)$.
%Similar to the analysis of OS-VI for PE, incurring some errors in each step of OS-VI~\eqref{eq:OSVI-Control} is a possibility.
%We consider two types of errors. The first is the error between the computed value function and the true optimal value function of the auxiliary MDP, i.e., $V_{k} - \Sopt V_{k-1}$. The second is the suboptimality of obtained policy compared to the optimal policy of the auxiliary MDP, i.e., $S^{\pi_k}V_{k-1} - \Sopt V_{k-1}$. Concretely, we have
%XXX ORIGINAL XXX
%We turn to analyzing OS-VI for Control. Similar to the analysis of OS-VI for PE, we assume some error in each step of OS-VI. Each iteration of OS-VI for PE was solving the PE problem in the auxiliary MDP $(\XX, \AA, \rbar_V, \PKernelhat, \gamma)$. The error of this solution was the error of the calculated value function $V_{k+1}$ w.r.t. the true value function in the MDP $\Spi V_k$ leading to \eqref{eq:AOS-VI-PE}. In Control, each iteration is solving the control problem in the same auxiliary MDP. The difference is that we have two types of error. First, the error of value function w.r.t. the true solution, i.e., $V_{k} - \Sopt V_{k-1}$. Second is the suboptimality of policy compared to the optimal policy of the auxiliary MDP, i.e., $S^{\pi_k}V_{k-1} - \Sopt V_{k-1}$.  More specifically
\begin{align}
\label{eq:AOS-VI-Control1}
V_{k} &= \Sopt V_{k-1} + \errv{k},\\
\label{eq:AOS-VI-Control2}
S^{\pi_k}V_{k-1} &= \Sopt V_{k-1} + \errpi{k}.
\end{align}
We have the following result for the approximate OS-VI (Control).
%\begin{proposition}
%		\label{prop:Control-infnorm}
%	Consider the approximate OS-VI algorithm for control ~\eqref{eq:AOS-VI-Control1},\eqref{eq:AOS-VI-Control2}. For any $k \geq 1$, let $\Pi_k = \{\pi^*, \pi_k\} \cup \{\pi_V(V_{i-1}) : 0 \le i < k\}$. Assume $\norm{\errv{i}}_\infty \le \errvm$ for all $i \ge 1$. Define the effective discount factor $\gamma' = \max_{\pi \in \Pi_k}\frac{\gamma}{1 - \gamma} \smallnorm{ \PKernelpi - \PKernelhat^\pi }_\infty$.
%     we have
%	%
%	\begin{align*}
%	\norm{V^{\pi_k} - V^* }_\infty \leq 
%	\frac{2\gamma'^k}{1 - \gamma'} \norm{V_0 - V^*}_\infty +
%\frac{2\gamma' (1 - \gamma'^{k-1})}{(1 - \gamma')^2} \cdot \errvm +
%\frac{1}{1 - \gamma'} \cdot \norm{\errpi{k}}_\infty
%	\end{align*}
%%
%\end{proposition}
%
%
%
%
%\begin{theorem}
%	\label{theorem:Lp-Control}
%	Consider the approximate OS-VI algorithm for control ~\eqref{eq:AOS-VI-Control1},\eqref{eq:AOS-VI-Control2}. Let $\rho$ be an arbitrary distribution over state space. For any $k \geq 1$, let $\Pi_k = \{\pi^*, \pi_k\} \cup \{\pi_V(V_{i-1}) : 0 \le i < k\}$. Assume $\norm{\errv{i}}_{4, \rho} \le \errvm$ for all $i \ge 0$. Define the effective discount factor $\gamma' = \max_{\pi \in \Pi_k} \frac{1}{1 - \gamma} \sqrt{ \sqrt{2} \cdot \hat C^\pi(\rho) \cdot \chi^2_{\rho}(\PKernel^\pi \;||\; \PKernelhat^\pi ) }$.
%	We have
%	%
%	\begin{align*}
%	\norm{V^{\pi_K} - V^* }_{4, \rho} \leq \frac{2\gamma'^k}{1 - \gamma'} \normfrho{V_0 - V^*} +
%\frac{2\gamma' (1 - \gamma'^{k-1})}{(1 - \gamma')^2} \cdot \errvm +
%\frac{1}{1 - \gamma'} \cdot \normfrho{\errpi{k}}
%	\end{align*}
%	
%\end{theorem}
%
%
%%%%%%%%%%%%%%%%%%%%%%%%
\begin{theorem}
\label{theorem:ErrorPropagation-SupAndLp-Control}
Consider the approximate OS-VI algorithm for control ~\eqref{eq:AOS-VI-Control1}-\eqref{eq:AOS-VI-Control2}.
Let $\norm{\cdot}_\star$ be either the supremum norm $\norm{\cdot}_\infty$ ($\star = \infty$) or $\norm{\cdot}_{4, \rho}$ for $\rho$ being an arbitrary distribution over the state space ($\star = 4, \rho$).	
For any $k \geq 1$, let $\Pi_k = \{\pi^*, \pi_k\} \cup \{\pi_V(V_{i-1}) : 1 \le i < k\}$.
Assume that $\smallnorm{\errv{k}}_\star \le \errvm$ for all $k \ge 1$.
Furthermore, define the effective discount factor as
\begin{align*}
	\gamma' =
		\frac{\gamma}{1 - \gamma} 
		\begin{cases}
			 \max_{\pi \in \Pi_k} \norm{ \PKernelpi - \PKernelhat^\pi }_\infty  & (\star = \infty),
			\\
			\max_{\pi \in \Pi_k} \sqrt{ \sqrt{2} \, \hat C^\pi(\rho) \chi^2_{\rho}(\PKernel^\pi \;||\; \PKernelhat^\pi ) } & (\star = 4, \rho).
		\end{cases}
\end{align*}
We then have	
\begin{align*}
	\norm{V^{\pi_k} - \Vopt }_{\star} \leq \frac{2\gamma'^k}{1 - \gamma'} \norm{V_0 - V^*}_\star +
\frac{2\gamma' (1 - \gamma'^{k-1})}{(1 - \gamma')^2} \errvm +
\frac{1}{1 - \gamma'} \norm{\errpi{k}}_\star.
	\end{align*}
\end{theorem}
%%%%%%%%%%%%%%%%%%%%%%%%

We can compare this result with the convergence result of VI. For VI with the supremum norm, following the proof of Equation (2.2) by~\citet{Munos07}, we can show that $\smallnorm{\Vopt - V^{\pi_k} }_\infty \leq \frac{2 \gamma^k}{1 - \gamma} \smallnorm{\Vopt - V_0}_\infty + \frac{2 \gamma (1 - \gamma^{k-1}) \errvm}{(1 - \gamma)^2}$, with $\norm{V_i - \Topt V_{i-1}}_\infty \leq \errvm$ for all $i < k$ (similar result for the $L_p$-norm also holds, see Theorem~5.2 in~\citealt{Munos07}).
For the approximate VI, the initial error $\smallnorm{\Vopt - V_0}_\infty$ decays with the rate of $O(\gamma^k)$. This should be compared with $O(\gamma'^k)$ rate of OS-VI. The effect of error at each step $\errvm$ is also similar: approximate VI has $(1 - \gamma)^{-2}$ dependence while approximate OS-VI has $(1 - \gamma')^{-2}$.
What is remarkable is that as opposed to $\gamma$, which is a fixed parameter of the problem and can be close to $1$, $\gamma'$ can be made arbitrary close to zero when the approximate model $\PKernelhat$ becomes more accurate. The additional information given by $\PKernelhat$ allows us to get much faster rate than VI. Of course, this requires the model to be accurate. An inaccurate model might be detrimental to the convergence rate, and may even lead to divergence.
Similar conclusions can be made in comparing OS-VI with Policy Iteration and Modified Policy Iteration, as discussed in the supplementary material.
\todo{Remark about the randomness of $\Pi$.}

\section{Operator splitting Dyna}
\label{sec:OSVI-Dyna}
% !TEX root =  OSVI.tex

% \todo{To be written. Feel free to start writing it. -AMF}
% \todo{The high-level plan:}
%\begin{itemize}
%	\item So far the approximate model $\PKernelhat$ was assumed to be given, maybe through a simulator.
%	\item We can learn it as well in a MBRL fashion.
%	\item The method would be a hybrid of MF and MB.
%	
%	\item Explanation would be as follows:
%
%	\item Computation of $\Spi$ or $\Sopt$ requires the computation of $NV = \gamma (\PKernel - \PKernelhat) V$. But we don't have $\PKernel$.
%	
%	\item Unbiased samples can be obtained for the $\PKernel V$ term.
%	
%	\item Given $\{(X_i, X'_i)\}_{i=1}^n$, for $V$, we can define
%				\[
%					(\hat{N}^\pi V)(X_i) = \gamma \left( V(X'_i) - \PKernelhat(\dy | X_i) V(y) \right).
%				\]
%
%	\item This provides an unbiased estimate (show it by taking expectation?).
%
%	\item This suggests a Stochastic Approximation or Fitted Value Iteration-like algorithm.
%
%	\item We suggest Dyna-style algorithm.
%
%	\item Have an OS-Dyna box for both PE and Control (maybe in one figure?).
%
%\end{itemize}

In the RL setting, we only have access to samples from $\PKernel$. MBRL algorithms, such as  variants of the Dyna architecture, learn $\PKernelhat$ from those samples, and use it to find the value function or policy.
The learned model $\PKernelhat$ is generally different from $\PKernel$ due to the finiteness of the samples as well as the possibility of model approximation error: the true dynamics $\PKernel$ may not be representable with the function approximator used to represent $\PKernelhat$. This is another way to say that the world may be too big to be represented by our models.
A MBRL algorithm that uses $\PKernelhat$ in lieu of $\PKernel$ does not find the true value of the true MDP.
Based on OS-VI, we propose OS-Dyna, as a hybrid model-based and model-free RL algorithm, that takes advantage of both the true environment and the model in its updates and can converge to the true value function despite using inaccurate $\PKernelhat$.

Learning $\PKernelhat$ in OS-Dyna is similar to other MBRL algorithms~\citep{MoerlandBroekensetal2022}: one can use various model learning approaches, either based on maximum likelihood estimate or a decision-aware model learning approach, to learn the model.
Given a learned $\PKernelhat$, we can compute $V_{k}$ from the auxiliary reward function $\rbar_{k} \defeq \rbar_{V_{k-1}}$ by solving the PE or the Control problem in the auxiliary MDP $(\XX, \AA, \rbar_{k}, \PKernelhat)$, as discussed in Section~\ref{sec:OSVI}. % A possible challenge of using a learned model is that in many cases, we can only take samples from it and do not have access to the full matrices in calculations. Fortunately, solving PE or control problems with such a model is a very common practice in RL and is also done in Dyna architecture. Therefore, we do not focus on this procedure.

As $V_k$ is a function of $\rbar_k$, we focus on how $\rbar_k$ should be estimated. The update rule of $\rbar_k$ in OS-VI entails that for every $(x, a)$, we have
\begin{align}
\label{eq:rbar-recurssive-PE}
\rbar_{k}(x, a) &= r(x, a) + \gamma \left(\PKernel(\cdot | x,a) - \PKernelhat( \cdot | x,a)\right) V^\pi(\rbar_{k-1}, \PKernelhat), &&\quad \text{(Policy Evaluation)}\\
\label{eq:rbar-recurssive-Control}
\rbar_{k}(x, a) &= r(x, a) + \gamma \left(\PKernel(\cdot | x,a) - \PKernelhat( \cdot | x,a)\right)  V^*(\rbar_{k-1}, \PKernelhat). &&\quad \text{(Control)}
\end{align}

We update our estimation of $\rbar$ using samples, as shall be discussed soon, and then the value function is updated to $V^\pi(\rbar, \PKernelhat)$ (PE) or $V^*(\rbar, \PKernelhat)$ (Control) with most recent estimate of $\rbar$. The challenge is that the above update rules need access to distribution $\PKernel(\cdot | x,a)$  for every $(x, a)$, while we only have samples from $\PKernel$ at some $(x, a)$ pairs. Fortunately, this challenge has been tackled in developing sample-based algorithms based on the classic VI:
\begin{align}
\label{eq:VI-Q-values}
	\forall (x,a): \qquad Q_k(x, a) = r(x, a) + \gamma \PKernel( \cdot | x,a) V_{k-1},
\end{align}
where $V_{k-1} = Q_{k-1}(x, \pi(x))$ in PE and $V_{k-1} = \max_{a'} Q_{k-1}(x, a)$ in Control. There are multiple approaches to develop sample-based algorithms based on \eqref{eq:VI-Q-values} such as Fitted Value Iteration and Stochastic Approximation (SA)~\citep{Borkar2008}. In this paper we use SA to develop OS-Dyna, but we point out that other algorithms and techniques can also be applied to develop other versions of OS-Dyna. The key step in SA is to use samples to form an unbiased estimate of the intended update value. For a step in the true environment leading to $(X_t, A_t, R_t, X_t')$ tuple, we can have the estimate
$Y_t = R_t + \gamma V(X_t')  - \gamma \EEX{X' \sim \PKernelhat(\cdot | X_t,A_t)} {V(X')}$,
where the expectation can also be estimated by samples from $\PKernelhat(\cdot | X_t,A_t)$. 
This estimate $Y_t$ of the update rule can then be used to update $\rbar$. As an example, for a finite state-action problem, the update rule is
\begin{align}
	\rbar(X_t, A_t) \leftarrow \rbar(X_t, A_t) + \alpha_t \left(
			Y_t - \rbar(X_t, A_t)  \right),
\end{align}
where $\alpha_t$ is the learning rate.
The final procedure of OS-Dyna is presented in Algorithm~\ref{alg:OS-Dyna}.

\todo{Most of the description of OS-Dyna in this section isn't specialized to discrete state/action spaces. The rebuttal portrays it as otherwise.}

\todo{Most of the description of OS-Dyna in this section isn't specialized to discrete state/action spaces. The rebuttal portrays it as otherwise.}

\begin{algorithm}[t]
\begin{small}
	\caption{OS-Dyna}
	\begin{algorithmic}[1]
		\State Initialize $V_0, \rbar = 0$, and the approximate model $\PKernelhat$.
		\For{$t = 1, 2, \ldots$}
		\State Sample $(X_t, A_t, R_t, X_t')$ from environment.
		\State Update the model $\PKernelhat$ with $(X_t, A_t, R_t, X_t')$.
%		\State  $\rbar(X_i, A_i) \leftarrow \rbar(X_i, A_i) + \alpha_i \cdot \left(
%		R_i + \gamma V(X_i')  - \gamma \Earg{V(Y) | Y \sim \PKernelhat(. | X_i,A_i)} - \rbar(X_i, A_i) 
%		\right)$
%
		\State  $\rbar(X_t, A_t) \leftarrow \rbar(X_t, A_t) + \alpha_t  \left(
		R_t + \gamma V_{t-1}(X_t')  - \gamma \EEX{X' \sim \PKernelhat( \cdot | X_t,A_t)}{V_{t-1}(X')} - \rbar(X_t, A_t) 
		\right)$.		
		\State $V_t \leftarrow V^{\pi}(\rbar, \PKernelhat)$ (For PE)\quad or \quad $V_t \leftarrow V^*(\rbar, \PKernelhat)\;,\;\pi_t \leftarrow \pi^*(\rbar, \PKernelhat)$ (For Control).
		\EndFor
	\end{algorithmic}
\label{alg:OS-Dyna}
\end{small}
\end{algorithm}

%%%%%%%%%%%%%%%%%%%%%%%%%%%%%%%%%%%%%%%%%%%%%%%
\section{Experiments}
\label{sec:Experiments}
% !TEX root =  OSVI.tex

We evaluate both OS-VI and OS-Dyna in a finite MDP and compare them with existing methods. Here we present the results for the Control problem on a modified cliffwalk environment in a $6 \times 6$ grid with 4 actions (UP, DOWN, LEFT, RIGHT). We postpone studying the PE problem, the results for other environments, and other relevant details to the supplementary material.
%We evaluate both OS-VI and OS-Dyna in a finite MDP, comparing our algorithms with existing methods. The MDP considered is a modified cliffwalk environment in a $6 \times 6$ grid with 4 actions (up, down, left and right). Full details of the environment is available in the supplementary material.
Our convergence analysis shows that the convergence rates of our algorithms depend on the accuracy of $\PKernelhat$. To test OS-VI and OS-Dyna with models of different accuracies, we introduce the smoothed model $\PKernelhat$ of transitions $\PKernel$ with smoothing parameter $\lambda$ as
\begin{equation}\label{eq:exp_model}
\hat{\mathcal{P}}(\cdot | x, a; \mathcal{P}, \lambda) =  (1 - \lambda) \mathcal{P}( \cdot | x,a) +  \lambda U\big( \{x^\prime | \mathcal{P}( x^\prime | x,a) > 0\}\big),
\end{equation} 
where $U(A)$ for some set $A$ is the uniform distribution over $A$. Here, $\lambda$ allows making adjustments to the amount of error introduced in $\PKernelhat$ w.r.t. $\PKernel$. If $\lambda = 0$, $\PKernelhat = \PKernel$ will be the accurate model, and if $\lambda = 1$, $\PKernelhat$ will be uniform over possible next states in $\PKernel$.

\begin{figure}[t]
\centering
\includegraphics[width=1\textwidth]{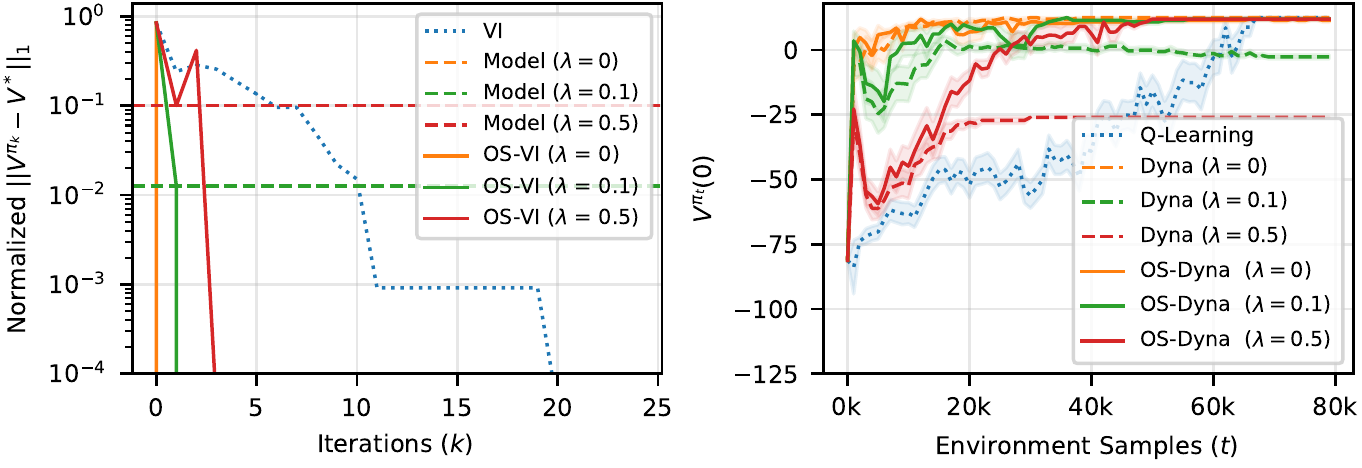}
\caption[short]{\textit{(Left)} 
Normalized error comparisons of OS-VI, VI, and 
the optimal policy of model $\PKernelhat$. \textit{(Right)} Comparison of OS-Dyna with Dyna and Q-Learning in the RL setting.}
\label{fig:main}
\vspace{-1em}
\end{figure}

%We compare OS-VI and OS-Dyna for Control with other existing methods. The results for the PE problem are qualitatively similar and are provided in the supplementary material along with more experiments. 
The left plot in Figure~\ref{fig:main} shows the convergence of OS-VI compared to VI and the solutions the model itself would lead to. The plot shows normalized error of $V^{\pi_k}$ w.r.t $V^*$, i.e., $\norm{V^{\pi_k} - V^*}_1 / \norm{V^*}_1$. It can be seen that OS-VI has faster convergence with more accurate models and introduces acceleration compared to VI across different model errors. Note that the convergence of OS-VI has been achieved despite the error in the model. The dashed lines show how a fully model-based algorithm, which only uses $\PKernelhat$, would obtain a suboptimal solution. %by only relying on the model.

% We compare OS-VI and OS-Dyna for Control with other existing methods. The results for the PE problem are qualitatively similar and are provided in the supplementary material along with more experiments. The left plot in Figure~\ref{fig:main} shows the convergence of OS-VI compared to VI and the solutions the model itself would lead to. The plot shows normalized error of $V^{\pi_k}$ w.r.t $V^*$, i.e. $\norm{V^{\pi_k} - V^*}_1 / \norm{V^*}_1$. It can be seen that OS-VI has faster convergence with more accurate models and introduces acceleration compared to VI across different model errors. Note that the convergence of OS-VI has been achieved despite the error in the model. The dashed lines show how a fully model-based algorithm would obtain a suboptimal solution by only relying on the model.

We also compare OS-Dyna with Dyna and Q-Learning in the RL setting. At each iteration $t$, the algorithms are given a sample $(X_t, A_t, R_t, X_t')$ where $X_t, A_t$ are selected uniformly at random. For OS-Dyna and Dyna we use the smoothed Maximum-likelihood Estimation (MLE) model. If $\mathcal{P}_{\text{MLE}}$ is the current MLE estimation of the environment transitions, OS-Dyna and Dyna use $\PKernelhat(\mathcal{P}_{\text{MLE}}, \lambda)$ defined in \eqref{eq:exp_model} as their models. The learning rates are constant $\alpha$ for iterations $t\le N$ and then decay in the form of $\alpha_t = \alpha / (t-N)$ afterwards. We have fine-tuned the learning rate schedule for \textit{each} algorithm separately for the best results.

The right plot in Figure~\ref{fig:main} shows the results for the RL setting. We evaluate the expected return of the policy at iteration $t$ in the initial state of the environment, i.e., \revised{$V^{\pi_k}(0)$}{$V^{\pi_t}(0)$}. Again, OS-Dyna has converged to the optimal policy much faster than Q-Learning. \revised{for all model errors.}{} Unlike OS-Dyna, \revised{classic}{} Dyna has failed to find the optimal policy in presence of model error. The results show that OS-Dyna can effectively converge faster than Q-Learning without introducing bias to the final solution due to model error.

\section{Conclusion}
\label{sec:Conclusion}
% !TEX root =  OSVI.tex

This paper introduced the Operator Splitting Value Iteration (OS-VI) algorithm, which can benefit from an approximate model $\PKernelhat \approx \PKernel$ to accelerate the convergence of the approximate value to the true value function in terms of the number of queries to the true model $\PKernel$.
With a small model error, its convergence rate is exponentially faster compared to well-known dynamical programming algorithms such as Value Iteration and Policy Iteration.
We also proposed OS-Dyna as a hybrid model-based/model-free algorithm that can bring in the benefits of a model-based RL algorithm without converging to a biased solution, as Dyna or any other purely model-based RL algorithm does.
This paper opens up several future directions.
Empirically studying the algorithms on problems with large state spaces, for which a function approximator such as a DNN is required, is an obvious one. This is postponed to a future work as our aim was to build the mathematical foundation and conducting experiments without worrying about challenges such as the optimization of a DNN.
There are other algorithmic and theoretical directions to be pursued.
One is exploring the space of splittings of $\Id - \gamma \PKernelpi$.
% For example, one may consider $\Mpi = \Id - \sum_{i = 0}^m \gamma^i (\PKernelhat^\pi)^{(i)}$ and $\Npi = \sum_{i=0}^m \gamma^i ( (\PKernelpi)^{(i)} - (\PKernelhat^\pi)^{(i)})$, which can be seen as splitting of a multi-step model.\todo{Not sure if we want to reveal it this much. -AMF}
%
The other is whether we can design Operator Splitting variants of other DP algorithms such as Policy Iteration and Modified Policy Iteration, and study their convergence behaviour.

%%%%%%%%%%%%%%%%%%%%%%%%%%%%%%%%%%%%%
% \clearpage

%%%%%%%%%%%%%%%%%%%%%%%%%%%%%%%%%%%%%
%%%%%%%%%%%%%%%%%%%%%%%%%%%%%%%%%%%%%%%%%%%%%%%
%%%%%%%%%%%%%%%%%%%%%%%%%%%%%%%%%%%%%%%%%%%%%%%
%%%%%%%%%%%%%%%%%%%%%%%%%%%%%%%%%%%%%%%%%%%%%%%
\ifSupp

\clearpage
%\onecolumn
\appendix 
%{\allowdisplaybreaks
%\section{Appendix}
%Hi
%}

%%%%%%%%%%%%%%%%%%%%%%%%%%%%%%%%%%%%%%%%%%%%%%%
\section{List of appendices}
We include a table of content in order to make the navigation through the appendices easier.
\begin{itemize}
    \item Appendix~\ref{sec:Background}: Background, including MDPs (\ref{sec:Background-MDP}), Norms (\ref{sec:Background-Norms}), commonly used matrix splitting in numerical linear algebra (\ref{sec:Background-MatrixSplitting}), and the relation between $\smallnorm{ \PKernelpi - \PKernelhat^\pi }_\infty$, TV, and $\KL$ (\ref{sec:Background-SupModelError-TV-KL}).
    \item Appendix~\ref{sec:AlgsDetails}: Details of the OS-Dyna algorithm.
    \item Appendix~\ref{sec:ProofsAndMoreTheory}: Basic properties of the Varga operator (\ref{sec:ProofsAndMoreTheory-BasicProperties}) and the proofs of the theoretical results for OS-VI (PE) (\ref{sec:ProofsAndMoreTheory-OSVI-PE}) and OS-VI (Control) (\ref{sec:ProofsAndMoreTheory-OSVI-PE}).
    \item Appendix~\ref{sec:AdditionalExperiments}: Description of the environments (\ref{sec:exp-envs}), and additional experiments including the study of convergence rate of OS-VI (\ref{sec:exp-osvi-convergence}), effect of model error (\ref{sec:exp-osvi-modelerror}), and further study of OS-Dyna (\ref{sec:exp-osdyna}).
    \item Appendix~\ref{sec:RelatedWork-Extended}: Extended related work, including detailed comparison of the convergence rate of OS-VI with VI, PI, and MPI (\ref{sec:RelatedWork-Extended-Convergence-Comparison}) and other examples of matrix splitting in dynamic programming and RL (\ref{sec:RelatedWork-Extended-OtherMS-Method}).
\end{itemize}

%\begin{itemize}
%    \item Section~\ref{sec:Background} includes the definitions and the background.
%    \item Section~\ref{sec:AlgsDetails} includes more details of OS-Dyna.
%    \item Section~\ref{sec:ProofsAndMoreTheory} includes some basic properties of the Varga operator and the proofs of the theoretical results.
%    \item Section~\ref{sec:AdditionalExperiments} includes the additional experiments.
%    \item Section~\ref{sec:RelatedWork-Extended} includes the extended related work and other common splitting choices.
%\end{itemize}

%%%%%%%%%%%%%%%%%%%%%%%%%%%%%%%%%%%%%%%%%%%%%%%
\section{Background}
\label{sec:Background}
% !TEX root =  OSVI.tex

%%%%%%%%%%%%%%%%%%%%%%%%%%%%%%%%%%%%%%%%%%%%%%%
\subsection{Markov Decision Processes}
\label{sec:Background-MDP}

We consider a discounted Markov Decision Process (MDP) $(\XX, \AA, \RKernel, \PKernel, \gamma)$~\citep{SzepesvariBook10,BertsekasShreve96,Bertsekas96,SuttonBarto2018}. Our notation is most similar to~\citet{SzepesvariBook10}'s.\footnote{This appendix closely follows Appendix~A of~\citet{FarahmandGhavamzadeh2021}.}
Here $\XX$ is the state space, $\AA$ is the action space, $\RKernel: \XA \ra \MM(\Real)$ is the reward distribution, $\PKernel:  \XA \ra \MM(\XX)$ is the transition probability kernel, and $ 0 \leq \gamma < 1$ is the discount factor.\footnote{For a set $\Omega$, the space of bounded functions is denoted by $\BB(\Omega)$, and the space of probability distributions is denoted by $\MM(\Omega)$. Here we do not go into the measurability issues, so we omit the detail of the necessary $\sigma$-algebra.}

The policy $\pi: \XX \ra \MM(\AA)$ (stochastic policy) or $\pi: \XX \ra \AA$ (deterministic) is a Markov stationary policy. 
Given a policy, we can define $\PKernelpi$ as the transition probability kernel of following $\pi$, and it would be
%$\PKernelpi(\cdot|x) = \sum_{a \in \AA} \PKernel(\cdot|x,a) \pi(a|x)$.
\[
	\PKernelpi(\cdot|x) = \int \PKernel(\cdot|x,a) \pi(\mathrm{d} a|x).
\]
We can define ${\PKernelpi}^{(m)}: \XX \ra \MM(\XX)$ for $m \geq 0$ recursively. For $m = 0$, we use the convention that it is equal to $\Id$, the identity operator (or matrix). 
For $m \geq 1$, we have
\[
	{\PKernelpi}^{(m)}(\cdot|x) = \int \PKernelpi(\dy|x) {\PKernelpi}^{(m-1)}(\cdot|y).
\]
%
%We may also define ${\PKernelpi}^{(m)}: \XA \ra \MM(\XX)$, as the transition probability kernel of choosing action $a$ at state $x$, following $\PKernel(\cdot|x,a)$ for one step, and afterwards, following policy $\pi$ for the remaining $m - 1$ steps.
%Formally,  for $m = 1$, ${\PKernelpi}^{(1)} = \PKernelpi = \PKernel$. For
%$m \geq 2$, we have
%\[
%	{\PKernelpi}^{(m)}(\cdot|x,a) = \int \PKernel(\dy|x,a) {\PKernelpi}^{(m-1)}(\cdot|y).
%\]

The discounted future-state distribution $\eta^\pi$ is defined based on $m$-step transition as
\begin{align*}
    	\eta^\pi (\cdot | x) = (1 - \gamma) \sum_{m = 0}^{\infty} \gamma^m  {\PKernelpi}^{(m)}(\cdot | x).
\end{align*}
We can define $\RKernel^\pi: \XX \ra \MM(\Real)$ in a similar fashion.
The functions $r: \XA \ra \Real$ and $r^\pi: \XX \ra \Real$ are the expected value of the reward distribution.

We use $\Vpi$ and $\Qpi$ to denote the state-value and action-value functions for a policy $\pi$. We use $\Vopt$ and $\Qopt$ to denote the optimal value and action-value functions.
In this work, we mostly use the state-value function, which we simply call the value function.

The Bellman operator $\Tpi: \BB(\XX) \ra \BB(\XX)$ for policy $\pi$ and the Bellman optimality operator $\Topt: \BB(\XX) \ra \BB(\XX) $ are defined as
\begin{align*}
	& (\Tpi V)(x) \eqdef \rpi(x) + \gamma \int \PKernelpi(\dy|x) V(y),
	\\
	&
	(\Topt V)(x) \eqdef \max_{a \in \AA} \left \{ r(x,a) + \gamma \int \PKernel(\dy|x,a) V(y) \right \}.
\end{align*}
%
%\begin{align*}
%%	& (\Tpi V)(x) \eqdef r^\pi(x) + \gamma \sum_{y \in \XX} \PKernelpi(y|x) V(y),
%	& (\Tpi V)(x) \eqdef \rpi(x) + \gamma \int \PKernelpi(\dy|x) V(y),
%	\\
%	&
%%	(\Topt Q)(x,a) \eqdef r(x,a) + \gamma \sum_{y \in \XX} \PKernel(y|x,a) \max_{a' \in \AA} Q(y,a'),
%	(\Topt Q)(x,a) \eqdef r(x,a) + \gamma \int \PKernel(\dy|x,a) \max_{a' \in \AA} Q(y,a').
%\end{align*}
For countable state and action spaces, the integrals are replaced by summations.
The Bellman operators $\Tpi: \BB(\XA) \ra \BB(\XA)$ and $\Topt: \BB(\XA) \ra \BB(\XA)$ (both applied on the action-value function) are defined similarly. We do not use them in the paper, so we do not explicitly define them here.
% $\Tpi Q$ and $\Topt V$ are defined similarly.

We denote $\pigreedy(\cdot;V)$ as the greedy policy w.r.t. $V$, i.e., at each state $x$, we have
\[
	\pigreedy(x;V) \leftarrow \argmax_{a \in \AA} \left\{ r(x,a) + \gamma \int \PKernel(\dy|x,a) V(y) \right\}.
\]
%
%\todo{We may need to define for $V$ instead. Note that it is already defined in the main body, so we may just skip.-AMF}
%We denote $\pi(\cdot;Q)$ as the greedy policy w.r.t. $Q$, i.e., at each state $x$, we have
%\begin{align}
%\label{eq:RLAcc-VI-GreedyPolicy}
%	\pigreedy(x;Q) \leftarrow \argmax_{a \in \Actions} Q(x,a).
%\end{align}

%%%%%%%%%%%%%%%%%%%%%%%%%%%%%%%%%%%%%%%%%%%%%%%
\subsection{Norms and metrics}
\label{sec:Background-Norms}

For function $f: \XX \ra \reals$, the $L_p(\rho)$-norm with respect to distribution $\rho \in \MM(\XX)$ over set $\XX$ is defined as
\begin{align*}
    \norm{f}_{p, \rho}^p \defeq \int f(x)^p \rho(\dx).
\end{align*}
If $\rho$ is the uniform distribution (or a Lebesgue measure), we may drop it in the notation for simplicity and write $\norm{f}_p$. In the special case of $p = \infty$, we have
\begin{align*}
   \norm{f}_{\infty} \defeq \sup_{x} | f(x) |.
\end{align*}

For two distributions $p , q \in \MM(\XX)$, the $\chi^2$-divergence is defined as
\begin{align*}
    \chi^2 (p \;||\; q) \defeq
    \int \frac{(p(\dx) - q(\dx))^2}{q(\dx)}.
\end{align*}

%%%%%%%%%%%%%%%%%%%%%%%%%%%%%%%%%%%%%%%%%%%%%%%
\subsection{Other examples of matrix splitting}
\label{sec:Background-MatrixSplitting}

In addition to the example of $M = \Id$ and $N = \Id - A$ in Section~\ref{sec:VIandSplitting-MatrixSplitting}, there are several other commonly used choices for matrix splitting.

If we decompose $A$ by its diagonal part $D$, its strictly lower triangular part $-L$, and its strictly upper triangular part $-U$ (so $A = D - L - U$), the choice of $M = D$ and $N = L + U$ leads to the Jacobi iteration. Clearly, the computation of $M^{-1} = D^{-1}$ is easy.

If we select 
$M = D - L$ and $N = U$ (or $M = D - U$ and \revised{}{$N = L$}), we get the forward (or backward) Gauss-Seidel iteration.

In all these cases, solving $M z_{k} = N z_{k-1} + b$ is easy.
The convergence of these methods can be established too. For instance, if $A$ is strictly diagonally dominated, the Jacobi iteration is convergent (Theorem 11.2.2 of~\citet{GolubVanLoan2013}).
These examples, as well as other choices available in the numerical linear algebra literature such as the Successive Over Relaxation method, show that there are multiple ways to split $A$ to $M$ and $N$, each with their own convergence properties.

%%%%%%%%%%%%%%%%%%%%%%%%%%%%%%%%%%%%%%%%%%%%%%%
\subsection{Relation between $\smallnorm{ \PKernelpi - \PKernelhat^\pi }_\infty$, the Total Variation error, and the $\KL$ divergence}
\label{sec:Background-SupModelError-TV-KL}

The model error $\smallnorm{ \PKernelpi - \PKernelhat^\pi }_\infty$ appeared in Section~\ref{sec:Theory}, and we argued that it is a reasonable choice to measure the distances between distributions. We expand on it here.

For countable state spaces, the norm $\smallnorm{ \PKernelpi - \PKernelhat^\pi }_\infty$ is the maximum over states of the Total Variation (TV) distance between $\PKernelpi(\cdot|x)$ and $\PKernelhat^\pi(\cdot|x)$. The TV distance itself can be upper bounded by the $\KL$-divergence between the distributions by Pinsker's inequality. So we get
\begin{align*}
	\norm{ \PKernelpi - \PKernelhat^\pi }_\infty =
	\max_{x \in \XX} \norm{ \PKernelpi(\cdot | x) - \PKernelhat^\pi (\cdot | x) }_1
	\leq
	\max_{x \in \XX} \sqrt{ 2 \KL\left( \PKernelpi(\cdot | x) || \PKernelhat^\pi (\cdot | x) \right) }.
\end{align*}
The KL-divergence is the population version of the negative-logarithm loss used in the Maximum Likelihood Estimation (MLE). Admittedly, the maximum over $\XX$ in $\smallnorm{\PKernel^\pi-\PKernelhat^\pi}_\infty$ is perhaps strict. Usually there are states in the state space that are irrelevant to the problem or even unreachable. If the model is learned from samples, it will be very inaccurate in these states.
\todo{Maybe mention that we actually need an IterVAML-like procedure for model learning. Or maybe we can have a combination of VAML/VaGraM and OS-Dyna for the next paper? -AMF}

%%%%%%%%%%%%%%%%%%%%%%%%%%%%%%%%%%%%%%%%%%%%%%%
\section{More details on OS-Dyna}
\label{sec:AlgsDetails}
% !TEX root =  OSVI.tex

Algorithm~\ref{alg:OS-Dyna-Detailed} shows a detailed version of Algorithm~\ref{alg:OS-Dyna} for the special case of finite MDPs and model learning based on MLE.
It is important to note that this is only a particular instantiation, and other variations are possible too.
For example, here we assume that samples are coming from a fixed distribution $\rho$, but it is also possible that they come from a trajectory of the agent's interaction with the environment.
Also many other techniques in RL such as Fitted Value Iteration with a replay buffer can be used instead of the stochastic approximation method we used in line 10 of Algorithm~\ref{alg:OS-Dyna-Detailed}.

Algorithm~\ref{alg:OS-Dyna-Detailed} uses MLE for model learning. In finite MDPs, where $\PKernelhat(\cdot | x,a)$ is a $|\XX|$-dimensional vector of probabilities for every $x,a$, MLE takes the form
\begin{align*}
    \PKernelhat(x' | x,a) = \frac{N(x, a, x')}{\sum_{x''} N(x, a, x'')}.
\end{align*}
Here, $N(x, a, x')$ is the number of times $x'$ is reached from $x$ and action $a$. This gives the model update in Algorithm~\ref{alg:OS-Dyna-Detailed}. In the general case, $\PKernelhat$ may be represented by some parametric distribution $p_\theta$. Nonetheless, the same principle of MLE can be applied to define the loss function for model learning. For example, the model can be updated incrementally by moving towards the gradient of log-likelihood:
\begin{align*}
    \hat \theta \leftarrow \hat \theta + \alpha \grad_{\theta} \sum_{i=1}^t \log p_\theta(X_i'|X_i, A_i) \rvert_{\theta = \hat{\theta}}.
\end{align*}
%~\citep{Sutton1990,Talvite2017,HaSchmidhuber2018NIPS,MoerlandBroekensetal2022,FarahmandVAML2017,Farahmand2018,GrimmBarretoSinghSilver2020,SchrittwieserAntonoglou2019,DOroMetelliTirinzoniPapiniRestelli2020,AbachiGhavamzadehFarahmand2020,LambertAmos2020,AyoubJiaSzepesvariWang2020,NikishinAbachi2022,VoelckerLiaoGargGarahmand2022}.
In addition to using the conventional approach of using MLE for model learning~\citep{Sutton1990,HaSchmidhuber2018NIPS}, more recent research has studied decision-aware model learning, in which the loss function incorporates some aspects of the decision problem itself. Some examples are
\citet{JosephGeramifardRobertsHowRoy2013,FarahmandVAML2017,SilvervanHasseltetal2017,OhSinghLee2017,Farahmand2018,GrimmBarretoSinghSilver2020,SchrittwieserAntonoglou2019,DOroMetelliTirinzoniPapiniRestelli2020,AbachiGhavamzadehFarahmand2020,LambertAmos2020,AyoubJiaSzepesvariWang2020,NikishinAbachi2022,VoelckerLiaoGargGarahmand2022}.

\begin{algorithm}[t]
	\caption{OS-Dyna (Detailed Version)}
	\begin{algorithmic}[1]
	    \State \textbf{Input:} {Sampling distribution $\rho$ over $\XX \times \AA$. Learning rate schedule $\left(\alpha_t \right)_{t=0}^{\infty}$}. Policy $\pi$ (for PE). Inner loop iterations $L$. Environment steps $T$.
		\State Initialize 
		\begin{itemize}
		    \item Value function $V\colon \XX \to \reals$ with $V(x) = 0$ for all $x$,
		    \item Auxiliary reward function $\rbar \colon \XX \times \AA \to \reals$ with $\rbar(x, a) = 0$ for all $x, a$,
		    \item Visitation counts $N \colon \XX \times \AA \times \XX \to \mathbb{Z}$ with $N(x, a, x') = 0$ for all $x, a, x'$.
		    \item Transition Model $\PKernelhat \colon \XX \times \AA \times \XX \to [0,1]$ with $\PKernelhat(x' | x, a) = \frac{1}{|\XX|}$ for all $x, a, x'$.
		\end{itemize}
		\For{$t = 1, 2, \ldots, T$}
		\State Sample $(X_t, A_t)$ from sampling distribution $\rho$.
		\State Take action $A_t$ at $X_t$ in the environment.
		\State Observe $X_t' \sim \PKernel(\cdot | X_t, A_t)$ and $R_t \sim \mathcal{R}(\cdot | X_t, A_t)$.
		\State Set $N(X_t, A_t, X_t') \leftarrow N(X_t, A_t, X_t') + 1$.
		\State Update the model $\forall x' \in \XX: \quad \PKernelhat(x' | X_t, A_t) \leftarrow  \frac{N(X_t, A_t, x')}{\sum_{x''\in \XX} N(X_t, A_t, x'')}$.
		\State Let 
		\begin{align*}
		  %  Y_t &= R_t + \gamma V(X_t')  - \gamma \EEX{X' \sim \PKernelhat(. | X_t,A_t)}{V(X')} \\
		    Y_t = R_t + \gamma V(X_t')  - \gamma \sum_{x' \in \XX} \PKernelhat(x' | X_t, A_t) V(x').
		\end{align*}
		\State  Set $\rbar(X_t, A_t) \leftarrow \rbar(X_t, A_t) + \alpha_t \left( Y_t - \rbar(X_t, A_t)
		\right)$.		
		\State Set $U_0 \leftarrow V$.
		\For{$i = 1, 2, \ldots, L$}
		\State  For every $x \in \XX$, set 
		\begin{align*}
		U_i(x) &\leftarrow 
		\begin{cases}
		\rbar(x, \pi(x)) + \gamma \sum_{x' \in \XX} \PKernelhat(x' | x, \pi(x)) U_{i-1}(x'), & \text{(for PE)}\\
		\max_{a \in \AA}\left[\rbar(x, a) + \gamma \sum_{x' \in \XX} \PKernelhat(x' | x, a) U_{i-1}(x')\right], & \text{(for Control)}\\
		\end{cases}\\
		\hat{\pi}^*(x) &\leftarrow \argmax_{a \in \AA}\left[\rbar(x, a) + \gamma \sum_{x' \in \XX} \PKernelhat(x' | x, a) U_{i-1}(x')\right]. \qquad \text{(for Control)}
		\end{align*}
		\EndFor
		\State Set $V \leftarrow U_L$.
		\EndFor
		\State For PE, output $V$. For control, output $V$ and $\hat{\pi}^*$.
	\end{algorithmic}
\label{alg:OS-Dyna-Detailed}
\end{algorithm}

\section{Proofs and other theoretical results}
\label{sec:ProofsAndMoreTheory}
% !TEX root =  OSVI.tex

%%%%%%%%%%%%%%%%%%%%%%%%%%%%%%%%%%%%%%%%%%%%%%%
\subsection{Basic properties of the Varga operators $\Spi$ and $\Sopt$ and MDPs}
\label{sec:ProofsAndMoreTheory-BasicProperties}

%%%%%%%%%%%%%%%%%%%%%%%
\begin{lemma}
	\label{lemma:operator-stationary}
For any policy $\pi$, we have
\begin{align*}
\Spi V^\pi = V^\pi.
\end{align*}
Also for the optimal value function $V^*$ and the optimal policy $\pi^*$, we have
\begin{align*}
\Sopt V^* = S^{\pi^*}V^* = V^*.
\end{align*}
\end{lemma}
%%%%%%%%%%%%%%%%%%%%%%%
%%%%%%%%%%%%%%%%%%%%%%%
\begin{proof}
For the first part, we write
\begin{align*}
S^\pi V^\pi 
&= (\Id - \gamma \PKernelhat^\pi)^{-1} (r^{\pi} + \gamma (\PKernelpi - \PKernelhat^\pi) V^{\pi})\\
&= (\Id - \gamma \PKernelhat^\pi)^{-1} (r^{\pi} + \gamma \PKernelpi V^{\pi} - \gamma \PKernelhat^\pi V^{\pi})\\
&= (\Id - \gamma \PKernelhat^\pi)^{-1} (V^{\pi} - \gamma \PKernelhat^\pi V^{\pi})\\
&= (\Id - \gamma \PKernelhat^\pi)^{-1} (\Id - \gamma \PKernelhat^\pi )V^{\pi} \\
&= V^{\pi},
\end{align*}
where we used the Bellman equation $V^{\pi} = r^{\pi} + \gamma \PKernelpi V^{\pi}$ in the third equality.

For the second part, note that the second equality is a consequence of the first part and $V^* = V^{\pi^*}$. For the first equality, we prove $S^*V^* = V^*$ by showing that $V^*$ is the optimal value function in MDP $(\XX, \AA, \rbar_{V^*}, \PKernelhat)$. To see this, we show that it satisfies the optimal Bellman equation for this MDP. For any state $x$,
\begin{align*}
\max_a \rbar_{V^*}(x, a) + \gamma \PKernelhat(\cdot | x, a)V^*
&=\max_a r(x, a) + \gamma \PKernel(\cdot | x, a)V^* - \gamma \PKernelhat(\cdot | x, a)V^* + \gamma \PKernelhat(\cdot | x, a)V^*\\
&=\max_a r(x, a) + \gamma \PKernel(\cdot | x, a)V^*\\
&= V^*(x),
\end{align*}
where in the last step we used the optimal Bellman equation in the original MDP.
\end{proof}

\begin{lemma}
	\label{lemma:spi-diff}
	Let $\Gpi = (\Mpi)^{-1}\Npi = (\Id - \gamma \PKernelhat^\pi)^{-1} \gamma (\PKernel - \PKernelhat)$. For any two functions $V_1$ and $V_2$ and policy $\pi$ we have
	\begin{align*}
	\Spi V_1 - \Spi V_2 = \Gpi(V_1 - V_2)
	\end{align*}
\end{lemma}
\begin{proof}
We have 
\begin{align*}
\Spi V_1 - \Spi V_2 &= (\Id - \gamma \PKernelhat^\pi)^{-1} (r^{\pi} + \gamma (\PKernelpi - \PKernelhat^\pi) V_1) - (\Id - \gamma \PKernelhat^\pi)^{-1} (r^{\pi} + \gamma (\PKernelpi - \PKernelhat^\pi) V_2)\\
&=
(\Id - \gamma \PKernelhat^\pi)^{-1} (r^{\pi} + \gamma (\PKernelpi - \PKernelhat^\pi) V_1 - r^{\pi}- \gamma (\PKernelpi - \PKernelhat^\pi) V_2)\\
&=
(\Id - \gamma \PKernelhat^\pi)^{-1} \gamma (\PKernelpi - \PKernelhat^\pi) (V_1 - V_2)\\
&= \Gpi (V_1 - V_2).
\end{align*}
\end{proof}

\begin{lemma}
	\label{lemma:sopt-spi-ineq}
 For any value function $V$ and policy $\pi$, we have $\Sopt V \vecge \Spi V$ where $\vecge$ is componentwise inequality.
\end{lemma}
\begin{proof}
This is direct consequence of definition $\Sopt V= \max_\pi \Spi V$.
\end{proof}

\begin{lemma}
\label{lemma:mpi-equals-etahat}
For any policy $\pi$, initial state $x \in \XX$, and a measurable set $B$, we have that
\[
	(\Mpi)^{-1}(B | x) = \frac{1}{1 - \gamma} \etahatpi(B | x).
\]
%	
%For any policy $\pi$ and two states $x, y$, $\Mpi_{x, y} = (\Id - \gamma \PKernelhat^\pi)^{-1}_{x, y} = \frac{1}{1 - \gamma} \etahatpi(\dy | x)$ where $A_{x, y}$ for matrix $A$ is the element in row $x$ and column $y$.
\end{lemma}
\begin{proof}
Recall that $(\Mpi)^{-1} = (\Id - \gamma \PKernelhat^\pi)^{-1}$. As $\smallnorm{\gamma \PKernelhat^\pi}_\infty = \gamma < 1$, we can use the Neumann expansion  
\[
(\Id - \gamma \PKernelhat^\pi)^{-1} = \sum_{m \geq 0} (\gamma \PKernelhat^\pi)^{(m)}.
\]
Therefore, the probability of starting from state $x$ and reaching a measurable set $B$ is
$
\sum_{m \geq 0} (\gamma \PKernelhat^\pi)^{(m)}(B | x)$, which is $\frac{1}{1 - \gamma} \etahatpi(B | x)$ by the definition of $\etahatpi$.
\end{proof}

\begin{lemma}
	\label{lemma:eta-p-bound}
 For any policy $\pi$, initial state $x \in \XX$, and a measurable set $B$, we have
 \begin{align*}
\int_y \etahatpi(\dy | x)\PKernelhat^\pi(B|y) \le \frac{1}{\gamma} \etahatpi(B | x).
 \end{align*}
\end{lemma}
\begin{proof}
We use the definition of $\etahatpi$ to get that
	\begin{align*}
\int_y \etahatpi(\dy|x) \PKernelhat^\pi(B|y)
&= (1-\gamma) \int_y \left(\sum_{t=0}^{\infty} \gamma^t 	{{} \PKernelpihat}^{(t)}(\dy|x) \right)\PKernelhat^\pi(B|y)\\
&= (1-\gamma) \sum_{t=0}^{\infty} \gamma^t \int_y  {{} \PKernelpihat}^{(t)}(\dy|x) \PKernelhat^\pi(B|y)\\
&= (1-\gamma) \sum_{t=0}^{\infty} \gamma^t 	{{} \PKernelpihat}^{(t+1)}(B|x)\\
&= \frac{(1-\gamma)}{\gamma} \sum_{t=0}^{\infty} \gamma^{t+1}	{{} \PKernelpihat}^{(t+1)}(B|x)\\
&= \frac{1}{\gamma} (\etahatpi(B|x) - (1-\gamma) {{} \PKernelpihat}^{(0)}(B|x))\\
&\le \frac{1}{\gamma} \etahatpi(B|x),
\end{align*}
where the inequality is due to the non-negativity of ${{} \PKernelpihat}^{(0)} = \Id$.
\end{proof}

% !TEX root =  OSVI.tex

\newcommand{\DeltaPpi}{ {\Delta\PKernel^\pi }}

\subsection{Proofs for convergence of OS-VI for policy evaluation}
\label{sec:ProofsAndMoreTheory-OSVI-PE}

%\textbf{Proof of Proposition~\ref{prop:PE-infnorm}}
\textbf{Proof of Theorem~\ref{theorem:ErrorPropagation-SupAndLp-PE} for $\star = \infty$}
%\textbf{Proof of Theorem~\ref{theorem:Lp-PE}}

\begin{proof}
From Lemma~\ref{lemma:operator-stationary}, we have $\Spi V^\pi = V^\pi$. By definition $V_k = \Spi V_{k-1} + \errv{k}$. Using Lemma~\ref{lemma:spi-diff}, we get
\begin{align*}
\norm{\Vpi - V_{k}}_\infty 
&=  \norm{\Spi\Vpi - \Spi V_{k-1} - \errv{k}}_\infty\\
&= \norm{\Gpi (\Vpi - V_{k-1}) - \errv{k}}_\infty \\
&\le \norm{\Gpi}_\infty \norm{\Vpi - V_{k-1}}_\infty + \norm{\errv{k}}_\infty.
\end{align*}
Now, we have that
\begin{align}
\label{eq:OSVI-Proof-G-UpperBound}
\norm{\Gpi}_\infty = 
\norm{ (\Id - \gamma \PKernelhat^\pi)^{-1} \gamma (\PKernelpi - \PKernelhat^\pi) }_\infty
\leq
\frac{\gamma}{1 - \gamma} \norm{ \PKernelpi - \PKernelhat^\pi }_\infty,
\end{align}
where we used the fact that for a linear operator $F$ with $\norm{ F } < 1$, it holds that
$\norm{ (\Id - F)^{-1} } \leq \frac{1}{1 - \norm{F}}$ (when $F$ is a square matrix, this is Lemma 2.3.3 of~\citealt{GolubVanLoan2013}). As $\PKernelhat^\pi$ is a transition kernel, we can choose the supremum norm, which has the property that $\smallnorm{\PKernelhat^\pi}_\infty = 1$.

By defining $\gamma' = \frac{\gamma}{1 - \gamma} \norm{ \PKernelpi - \PKernelhat^\pi }_\infty$ and combining the above two inequalities, we get
\begin{align*}
\norm{\Vpi - V_{k}}_\infty 
\le
\gamma' 
\norm{\Vpi - V_{k-1}}_\infty + \norm{\errv{k}}_\infty.
\end{align*}
Expanding this recursive inequality gives
\begin{align*}
\norm{\Vpi - V_{k}}_\infty 
&\le
\gamma'^k 
\norm{\Vpi - V_{0}}_\infty + \sum_{i = 1}^{k} \gamma'^{k-i} \norm{\errv{i}}_\infty\\
&\le
\gamma'^k 
\norm{\Vpi - V_{0}}_\infty + \errvm \sum_{i = 1}^{k} \gamma'^{k-i} \\
& = 
\gamma'^k 
\norm{\Vpi - V_{0}}_\infty + \frac{1 - \gamma'^k}{1 - \gamma'} \errvm,\\
\end{align*}
which completes the proof.
\end{proof}

Before proving the $L_p$ norm result, we present a key lemma.
%%%%%%%%%%%%%%%%%%%%%
\begin{lemma}
	\label{lemma:Gpiv-bound}
	Let $\rho$ be an arbitrary distribution over state space. Assume that for any $x \in \XX$, $\etahatpi(\cdot|x) \ll \rho$, i.e., $\etahatpi(\cdot|x)$ is absolutely continuous w.r.t. $\rho$.
	For any policy $\pi$ and a function $v\colon \XX \ra \real$, we have
	\begin{align*}
		\norm{\Gpi v}_{4, \rho} \le \frac{\gamma}{1 - \gamma} \sqrt{ \hat C^\pi(\rho) \chi^2_{\rho}(\PKernel^\pi \;||\; \PKernelhat^\pi ) } \, \norm{v}_{4, \rho}.
	\end{align*}
\end{lemma}
%%%%%%%%%%%%%%%%%%%%%
\begin{proof}
Let $\DeltaPpi = \abs{\PKernelpi - \PKernelhat^\pi}$. Using Lemma~\ref{lemma:mpi-equals-etahat}, we expand $\norm{\Gpi v}_{4, \rho}^4$ 
\begin{align}
\notag
\norm{\Gpi v}_{4, \rho}^4
&= \int_x \rho(\dx) \left[\iint_{y,z} \frac{1}{1 - \gamma} \etahatpi(\dy | x)  \cdot \gamma (\PKernelpi(\dz|y) - \PKernelhat^\pi(\dz|y)) \cdot v(z) \right]^4\\ \notag
&\le \int_x \rho(\dx) \left[\iint_{y,z} \frac{1}{1 - \gamma} \etahatpi(\dy | x)  \cdot \gamma \DeltaPpi(\dz|y) \cdot \abs{v(z)} \right]^4\\ \notag
&= \frac{\gamma^4}{(1 - \gamma)^4} \int_x \rho(\dx) \left[\iint_{y,z}
\left(
\sqrt{\rho(\dy)} \cdot \frac{\DeltaPpi(\dz|y)}{\sqrt{\PKernelhat^\pi(\dz | y)}}
\right)
\left(
\frac{\etahatpi(\dy | x) \cdot \abs{v(z)}  \cdot \sqrt{\PKernelhat^\pi(\dz | y)}}{\sqrt{\rho(\dy)} }
\right)
  \right]^4\\ \notag
&\le
\frac{\gamma^4}{(1 - \gamma)^4} \int_x \rho(\dx) 
\left(\iint_{y,z}
\rho(\dy) \cdot \frac{\DeltaPpi(\dz|y)^2}{\PKernelhat^\pi(\dz | y)}
\right)^2
\cdot
\left(
\iint_{y,z}
\frac{\etahatpi(\dy | x)^2 \cdot v(z)^2  \cdot \PKernelhat^\pi(\dz | y)}{\rho(\dy)} 
\right)^2\\ \notag
&=
\frac{\gamma^4}{(1 - \gamma)^4} \cdot 
\chi^2_{\rho}(\PKernel^\pi \;||\; \PKernelhat^\pi )^2 \cdot
\int_x \rho(\dx) 
\left(
\iint_{y,z}
\frac{\etahatpi(\dy | x)^2 \cdot v(z)^2  \cdot \PKernelhat^\pi(\dz | y)}{\rho(\dy)} 
\right)^2\\ \notag
&=
\frac{\gamma^4}{(1 - \gamma)^4} \cdot 
\chi^2_{\rho}(\PKernel^\pi \;||\; \PKernelhat^\pi )^2 \cdot
\int_x \rho(\dx) 
\left[
\int_z
\left(
\sqrt{\rho(\dz)} \cdot v(z)^2
\right)
\cdot
\left(
\int_y
\frac{\etahatpi(\dy | x)^2 \cdot \PKernelhat^\pi(\dz | y)}{\sqrt{\rho(\dz)} \cdot \rho(\dy)}
\right)
\right]^2\\ \notag
&\le
\frac{\gamma^4}{(1 - \gamma)^4} \cdot 
\chi^2_{\rho}(\PKernel^\pi \;||\; \PKernelhat^\pi )^2 \cdot
\int_x \rho(\dx) 
\left[\int_z
\rho(\dz)  v(z)^4
\right]
\left[
\int_z
\left(
\int_y
\frac{\etahatpi(\dy | x)^2 \cdot \PKernelhat^\pi(\dz | y)}{\sqrt{\rho(\dz)} \cdot \rho(\dy)}
\right)^2
\right]\\ 
\label{eq:G-l4-bound1}
&=
\frac{\gamma^4}{(1 - \gamma)^4} \cdot 
\chi^2_{\rho}(\PKernel^\pi \;||\; \PKernelhat^\pi )^2 \cdot
\norm{v}_{4, \rho}^4 \cdot 
\iint_{x, z} \rho(\dx) 
\left(
\int_y
\frac{\etahatpi(\dy | x)^2 \cdot \PKernelhat^\pi(\dz | y)}{\sqrt{\rho(\dz)} \cdot \rho(\dy)}
\right)^2,
\end{align}
where \revised{}{the second and the third inequalities are from the Cauchy-Schwarz inequality.} We now write
\begin{align*}
&\iint_{x, z} \rho(\dx) 
\left(
\int_y
\frac{\etahatpi(\dy | x)^2 \cdot \PKernelhat^\pi(\dz | y)}{\sqrt{\rho(\dz)} \cdot \rho(\dy)}
\right)^2\\
&\qquad= \int_{x} \rho(\dx) \int_z \rho(\dz)
\left(
\int_y
\frac{\etahatpi(\dy | x)^2 \cdot \PKernelhat^\pi(\dz | y)}{\rho(\dz) \cdot \rho(\dy)}
\right)^2\\
&\qquad= \int_{x} \rho(\dx) \int_z \rho(\dz)
\left(
\int_y
\frac{\etahatpi(\dy | x)}{\rho(\dy)}
\cdot
\frac{\etahatpi(\dy | x) \cdot \PKernelhat^\pi(\dz | y)}{\rho(\dz)}
\right)^2\\
&\qquad\le \int_{x} \rho(\dx) \left(\max_y
\frac{\etahatpi(\dy | x)}{\rho(\dy)}\right)^2 \int_z \rho(\dz)
\cdot
\left(\int_y
\frac{\etahatpi(\dy | x) \cdot \PKernelhat^\pi(\dz | y)}{\rho(\dz)}
\right)^2.\\
\end{align*}
Using Lemma~\ref{lemma:eta-p-bound}, we can continue as
\begin{align*}
&\iint_{x, z} \rho(\dx) 
\left(
\int_y
\frac{\etahatpi(\dy | x)^2 \cdot \PKernelhat^\pi(\dz | y)}{\sqrt{\rho(\dz)} \cdot \rho(\dy)}
\right)^2\\
&\qquad\le \int_{x} \rho(\dx) \left(\max_y
\frac{\etahatpi(\dy | x)}{\rho(\dy)}\right)^2 \int_z \rho(\dz)
\cdot
\left(
\frac{\etahatpi(\dz | x)}{\gamma\rho(\dz)}
\right)^2\\
&\qquad= \frac{1}{\gamma^2} \int_{x} \rho(\dx) \left(\max_y
\frac{\etahatpi(\dy | x)}{\rho(\dy)}\right)^2 \int_z \etahatpi(\dz | x)
\cdot
\frac{\etahatpi(\dz | x)}{\rho(\dz)}\\
&\qquad\le \frac{1}{\gamma^2}\int_{x} \rho(\dx) \left(\max_y
\frac{\etahatpi(\dy | x)}{\rho(\dy)}\right)^2 \left(\max_z \frac{\etahatpi(\dz | x)}{\rho(\dz)} \right) \int_z \etahatpi(\dz | x)\\
&\qquad= \frac{1}{\gamma^2}\int_{x} \rho(\dx) \left(\max_y
\frac{\etahatpi(\dy | x)}{\rho(\dy)}\right)^3\\
&\qquad= \hat C^\pi(\rho)^2,
\end{align*}
where we used the fact that $\etahatpi(\cdot|z)$ is a probability distribution, so $\int_z \etahatpi(\dz | x) = 1$.
Substituting in \eqref{eq:G-l4-bound1} gives
\begin{align*}
\norm{\Gpi v}_{4, \rho}^4 \le \frac{\gamma^4}{(1 - \gamma)^4} \cdot 
\chi^2_{\rho}(\PKernel^\pi \;||\; \PKernelhat^\pi )^2 \cdot \hat C^\pi(\rho)^2 \cdot
\norm{v}_{4, \rho}^4,
\end{align*}
which concludes the proof.
\end{proof}

\textbf{Proof of Theorem~\ref{theorem:ErrorPropagation-SupAndLp-PE} for the $L_4(\rho)$ norm}

\begin{proof}
From Lemma~\ref{lemma:operator-stationary}, we have $\Spi V^\pi = V^\pi$. By definition $V_k = \Spi V_{k-1} + \errv{k}$. Using Lemma~\ref{lemma:spi-diff}, we get 
\begin{align*}
\norm{\Vpi - V_{k}}_{4, \rho}
&=  \norm{\Spi\Vpi - \Spi V_{k-1} - \errv{k}}_{4, \rho}\\
&= \norm{\Gpi (\Vpi - V_{k-1}) - \errv{k}}_{4, \rho} \\
&\le \norm{\Gpi (\Vpi - V_{k-1})}_{4, \rho}+ \norm{\errv{k}}_{4, \rho}\\
&\le \gamma' \norm{\Vpi - V_{k-1}}_{4, \rho}+ \norm{\errv{k}}_{4, \rho},
\end{align*}
where we used Lemma~\ref{lemma:Gpiv-bound} in the last step. Expanding this recursive inequality gives
\begin{align*}
\norm{\Vpi - V_{k}}_{4, \rho}
&\le
\gamma'^k 
\norm{\Vpi - V_{0}}_{4, \rho} + \sum_{i = 1}^{k} \gamma'^{k-i} \norm{\errv{k}}_{4, \rho}\\
&\le
\gamma'^k 
\norm{\Vpi - V_{0}}_{4, \rho}+ \errvm \sum_{i = 1}^{k} \gamma'^{k-i} \\
&\le
\gamma'^k 
\norm{\Vpi - V_{0}}_{4, \rho} + \frac{1 - \gamma'^k}{1 - \gamma'} \errvm, \\
\end{align*}
which completes the proof.
\end{proof}

% !TEX root =  OSVI.tex
\subsection{Proofs for convergence of OS-VI for Control}
\label{sec:ProofsAndMoreTheory-OSVI-Control}

We prove Theorem~\ref{theorem:ErrorPropagation-SupAndLp-Control} for the $L_\infty$ and $L_4(\rho)$ cases separately.

For the $L_\infty$ part, we break its proof into two lemmas.
\begin{lemma}
\label{lemma:vpik-to-vik-inf}
Assume that $k \ge 1$. Let $\gamma'$ be the effective discount factor defined in Theorem~\ref{theorem:ErrorPropagation-SupAndLp-Control} for the $\star = \infty$ case. Then,
\begin{align*}
\norm{V^{\pi_k} - V^*}_\infty \le \frac{2\gamma'}{1 - \gamma'}\norm{V_{k-1} - V^*}_\infty + \frac{1}{1 - \gamma'}  \norm{\errpi{k}}_\infty.
\end{align*}
\end{lemma}
\begin{proof}
Let $\vecle$ be the componentwise inequality. Using Lemma~\ref{lemma:spi-diff} and Lemma~\ref{lemma:sopt-spi-ineq} and the definition of $\errpi{k}$, we have
\begin{align*}
	V^* - V^{\pi_k}
	&= S^{\pi^*}V^* - S^{\pi^*}V_{k-1} + S^{\pi^*}V_{k-1} - S^*V_{k-1} + S^*V_{k-1} - S^{\pi_k}V_{k-1} + S^{\pi_k}V_{k-1} - S^{\pi_k}V^{\pi_k}\\
	&\vecle G^{\pi^*} (V^* - V_{k-1}) - \errpi{k} + G^{\pi_k}(V_{k-1} - V^{\pi_k})\\
	&= G^{\pi^*} (V^* - V_{k-1}) - \errpi{k} + G^{\pi_k}(V_{k-1} - V^*) + G^{\pi_k}(V^* - V^{\pi_k})\\
	&= (G^{\pi^*} - G^{\pi_k}) (V^* - V_{k-1}) - \errpi{k} + G^{\pi_k}(V^* - V^{\pi_k}).
\end{align*}
	
Note that by definition $V^{*} - V^{\pi_k} \vecge 0$. Thus, we get
\begin{align*}
\norm{V^* - V^{\pi_k}}_\infty 
&\le \norm{(G^{\pi^*} - G^{\pi_k}) (V^* - V_{k-1}) - \errpi{k} + G^{\pi_k}(V^* - V^{\pi_k})}_\infty\\
&\le \left(\norm{G^{\pi^*}}_\infty + \norm{G^{\pi_k}}_\infty \right) \norm{V^* - V_{k-1}}_\infty + \norm{\errpi{k}}_\infty + \norm{G^{\pi_k}}_\infty \norm{V^* - V^{\pi_k}}_\infty\\
&\le 2\gamma' \norm{V^* - V_{k-1}}_\infty  +  \norm{\errpi{k}}_\infty  + \gamma' \norm{V^* - V^{\pi_k}}_\infty.
\end{align*}
After rearranging, we conclude
\begin{align*}
\norm{V^* - V^{\pi_k}}_\infty  \le
\frac{2\gamma'}{1 - \gamma'} \norm{V^* - V_{k-1}}_\infty + \frac{1}{1 - \gamma'}  \norm{\errpi{k}}_\infty.
\end{align*}
\end{proof}

\begin{lemma}
	\label{lemma:vik-to-vik-inf}
	Assume that $k \ge 1$. Let $\gamma'$ be the effective discount factor defined in  Theorem~\ref{theorem:ErrorPropagation-SupAndLp-Control} for the $\star = \infty$ case. Then for any $1 \le i \le k-1$ we have
	\begin{align*}
	\norm{V_i - V^*}_\infty \le \gamma' \norm{V_{i-1} - V^*}_\infty +   \norm{\errv{i}}_\infty.
	\end{align*}
\end{lemma}
\begin{proof}
Let $\pi'_i = \pi_V(V_{i-1})$. We have by Lemma~\ref{lemma:spi-diff} and Lemma~\ref{lemma:sopt-spi-ineq}
\begin{align*}
&V^* - S^*V_{i-1} = S^*V^* - S^{\pi'_i}V^* + S^{\pi'_i}V^* - S^{\pi'_i}V_{i-1} \vecge G^{\pi'_i}(V^* - V_{i-1}),\\
&V^* - S^*V_{i-1} = S^{\pi^*}V^* - S^{\pi^*}V_{i-1} + S^{\pi^*}V_{i-1} - S^*V_{i-1} \vecle G^{\pi^*}(V^* - V_{i-1}),
\end{align*}
where we used $S^*V^* \vecge S^{\pi'_i}V^*$ and $ S^{\pi^*}V_{i-1} \vecle S^*V_{i-1}$.

Let $|\cdot|$ and $\max$ be componentwise functions. We get
\begin{align*}
\abs{V^{*} - S^*V_{i-1}} &\vecle \max\left(\abs{G^{\pi'_i}(V^* - V_{i-1})}, 
\abs{G^{\pi^*}(V^* - V_{i-1})} \right)\\
\Rightarrow \norm{V^{*} - S^*V_{i-1}}_\infty &\le
\max\left(
\norm{G^{\pi'_i}(V^* - V_{i-1})}_\infty, 
\norm{G^{\pi^*}(V^* - V_{i-1})}_\infty 
\right)\\
&\le 
\max\left(
\gamma' \norm{V^* - V_{i-1}}_\infty, 
\gamma' \norm{V^* - V_{i-1}}_\infty
\right)\\
&= \gamma' \norm{V^* - V_{i-1}}_\infty,
\end{align*}
where we used the upper bound on the norm of $G^{\pi'_i}$ and $G^{\pi^*}$~\eqref{eq:OSVI-Proof-G-UpperBound} and the definition of effective discount factor $\gamma'$.

Finally, we write
\begin{align*}
\norm{V_i - V^*}_\infty 
&\le \norm{V_i - S^*V_{i-1}}_\infty + \norm{S^*V_{i-1} - V^*}_\infty\\
&\le \norm{\errv{i}}_\infty + \gamma' \norm{V^* - V_{i-1}}_\infty, 
\end{align*}
as desired.
\end{proof}

\textbf{Proof of Theorem~\ref{theorem:ErrorPropagation-SupAndLp-Control} -- the $L_\infty$ case}

\begin{proof}
Expanding the recursive result of Lemma~\ref{lemma:vik-to-vik-inf}, we get
\begin{align*}
\norm{V_{k-1} - V^*}_\infty 
&\le \gamma'^{k-1}\norm{V_0 - V^*}_\infty + \sum_{i=1}^{k-1} \gamma'^{k-1-i} \norm{\errv{i}}_\infty\\
&\le \gamma'^{k-1}\norm{V_0 - V^*}_\infty + \errvm \sum_{i=1}^{k-1 }\gamma'^{k-1-i} \\
&= \gamma'^{k-1}\norm{V_0 - V^*}_\infty + \errvm \frac{1 - \gamma'^{k-1}}{1 - \gamma'}.
\end{align*}
Substituting this in Lemma~\ref{lemma:vpik-to-vik-inf}, we get
\begin{align*}
\norm{V^{\pi_k} - V^*}_\infty &\le \frac{2\gamma'}{1 - \gamma'} \left[
\gamma'^{k-1}\norm{V_0 - V^*}_\infty + \errvm  \frac{1 - \gamma'^{k-1}}{1 - \gamma'} \right]  + \frac{1}{1 - \gamma'}  \norm{\errpi{k}}_\infty\\
&= 
\frac{2\gamma'^k}{1 - \gamma'} \norm{V_0 - V^*}_\infty +
\frac{2\gamma' (1 - \gamma'^{k-1})}{(1 - \gamma')^2} \errvm +
\frac{1}{1 - \gamma'} \norm{\errpi{k}}_\infty.
\end{align*}

\end{proof}

We introduce similar lemmas for the proof of the $L_4(\rho)$ part of Theorem~\ref{theorem:ErrorPropagation-SupAndLp-Control}.

\begin{lemma}
	\label{lemma:vpik-to-vik-lp}
	Assume $k \ge 1$, and $\rho$ is a distribution over state space. Let $\gamma'$ be the effective discount factor defined in  in Theorem~\ref{theorem:ErrorPropagation-SupAndLp-Control} for the $\star = 4,\rho$ case. Then 
	\begin{align*}
	\normfrho{V^{\pi_k} - V^*} \le \frac{2\gamma'}{1 - \gamma'}\normfrho{V_{k-1} - V^*}+ \frac{1}{1 - \gamma'}  \normfrho{\errpi{k}}.
	\end{align*}
\end{lemma}
\begin{proof}
	Let $\vecle$ be the componentwise inequality. Exactly similar to the proof of Lemma~\ref{lemma:vpik-to-vik-inf}, we have by Lemma~\ref{lemma:spi-diff} and Lemma~\ref{lemma:sopt-spi-ineq} that
	\begin{align*}
	V^* - V^{\pi_k}
	&= S^{\pi^*}V^* - S^{\pi^*}V_{k-1} + S^{\pi^*}V_{k-1} - S^*V_{k-1} + S^*V_{k-1} - S^{\pi_k}V_{k-1} + S^{\pi_k}V_{k-1} - S^{\pi_k}V^{\pi_k}\\
	&\vecle G^{\pi^*} (V^* - V_{k-1}) - \errpi{k} + G^{\pi_k}(V_{k-1} - V^{\pi_k})\\
	&= G^{\pi^*} (V^* - V_{k-1}) - \errpi{k} + G^{\pi_k}(V_{k-1} - V^*) + G^{\pi_k}(V^* - V^{\pi_k})\\
	&= (G^{\pi^*} - G^{\pi_k}) (V^* - V_{k-1}) - \errpi{k} + G^{\pi_k}(V^* - V^{\pi_k}).
	\end{align*}
	
	Note that by definition $V^{*} - V^{\pi_k} \vecge 0$. Thus, using Lemma~\ref{lemma:Gpiv-bound} we can write
	\begin{align*}
	\normfrho{V^* - V^{\pi_k}}
	&\le \normfrho{(G^{\pi^*} - G^{\pi_k}) (V^* - V_{k-1}) - \errpi{k} + G^{\pi_k}(V^* - V^{\pi_k})}\\
	&\le \normfrho{G^{\pi^*}(V^* - V_{k-1})} + \normfrho{G^{\pi_k} (V^* - V_{k-1})} + \normfrho{G^{\pi_k}(V^* - V^{\pi_k})} + \normfrho{\errpi{k}}\\
	&\le 2\gamma' \normfrho{V^* - V_{k-1}}  +  \normfrho{\errpi{k}}  + \gamma' \normfrho{V^* - V^{\pi_k}}.
	\end{align*}
	After rearranging, we conclude that
	\begin{align*}
	\normfrho{V^* - V^{\pi_k}}  \le
	\frac{2\gamma'}{1 - \gamma'} \normfrho{V^* - V_{k-1}} + \frac{1}{1 - \gamma'}  \normfrho{\errpi{k}}.
	\end{align*}
\end{proof}

\begin{lemma}
	\label{lemma:vik-to-vik-lp}
	Assume that $k \ge 1$. Let $\gamma'$ be the effective discount factor defined in  Theorem~\ref{theorem:ErrorPropagation-SupAndLp-Control} for the $\star = 4, \rho$ case. Then for any $1 \le i \le k-1$ we have
	\begin{align*}
	\normfrho{V_i - V^*} \le \gamma' \normfrho{V_{i-1} - V^*} +   \normfrho{\errv{i}}
	\end{align*}
\end{lemma}
\begin{proof}
	Let $\pi'_i = \pi_V(V_{i-1})$. We have by Lemma~\ref{lemma:spi-diff} and Lemma~\ref{lemma:sopt-spi-ineq}
	\begin{align*}
	&V^* - S^*V_{i-1} = S^*V^* - S^{\pi'_i}V^* + S^{\pi'_i}V^* - S^{\pi'_i}V_{i-1} \vecge G^{\pi'_i}(V^* - V_{i-1}),\\
	&V^* - S^*V_{i-1} = S^{\pi^*}V^* - S^{\pi^*}V_{i-1} + S^{\pi^*}V_{i-1} - S^*V_{i-1} \vecle G^{\pi^*}(V^* - V_{i-1}).
	\end{align*}
	Let $|\cdot|$ and $\max$ be componentwise functions. We get from Lemma~\ref{lemma:Gpiv-bound}
	\begin{align*}
	\abs{V^{*} - S^*V_{i-1}} &\vecle \max\left(\abs{G^{\pi'_i}(V^* - V_{i-1})}, 
	\abs{G^{\pi^*}(V^* - V_{i-1})} \right)\\
	\Rightarrow \normfrho{V^{*} - S^*V_{i-1}}^4 &\le
	\normfrho{G^{\pi'_i}(V^* - V_{i-1})}^4 + 
	\normfrho{G^{\pi^*}(V^* - V_{i-1})}^4 \\
	&\le 
	2\max\bigg(
	\frac{\gamma^4}{(1 - \gamma)^4} \hat C^{\pi_i'}(\rho)^2 \cdot \chi^2_{\rho}(\PKernel^{\pi_i'} \;||\; \PKernelhat^{\pi'_i} )^2, \\
	&\qquad\qquad
		\frac{\gamma^4}{(1 - \gamma)^4} \hat C^{\pi^*}(\rho)^2 \cdot \chi^2_{\rho}(\PKernel^{\pi^*} \;||\; \PKernelhat^{\pi^*} )^2
	\bigg) \normfrho{V^* - V_{i-1}}^4\\
	&\le \gamma'^4 \normfrho{V^* - V_{i-1}}^4.
	\end{align*}
	
	Finally we write
	\begin{align*}
	\normfrho{V_i - V^*} 
	&\le \normfrho{V_i - S^*V_{i-1}} + \normfrho{S^*V_{i-1} - V^*}\\
	&\le \normfrho{\errv{i}} + \gamma' \normfrho{V^* - V_{i-1}}.
	\end{align*}
\end{proof}

\textbf{Proof of Theorem~\ref{theorem:ErrorPropagation-SupAndLp-Control} -- $L_4(\rho$) case}

\begin{proof}
Using Lemma~\ref{lemma:vpik-to-vik-lp} and Lemma~\ref{lemma:vik-to-vik-lp}, the proof follows exactly like the proof of the $L_\infty$ case. Expanding the recursive result of Lemma~\ref{lemma:vik-to-vik-lp}, we get
\begin{align*}
\normfrho{V_{k-1} - V^*} 
&\le \gamma'^{k-1}\normfrho{V_0 - V^*} + \sum_{i=1}^{k-1} \gamma'^{k-1-i} \normfrho{\errv{i}}\\
&\le \gamma'^{k-1}\normfrho{V_0 - V^*} + \errvm \sum_{i=1}^{k-1 }\gamma'^{k-1-i} \\
&= \gamma'^{k-1}\normfrho{V_0 - V^*} + \errvm \cdot \frac{1 - \gamma'^{k-1}}{1 - \gamma'}.
\end{align*}
Substituting this in Lemma~\ref{lemma:vpik-to-vik-lp}, we get
\begin{align*}
\normfrho{V^{\pi_k} - V^*} &\le \frac{2\gamma'}{1 - \gamma'}\cdot \left[
\gamma'^{k-1}\normfrho{V_0 - V^*} + \errvm \cdot \frac{1 - \gamma'^{k-1}}{1 - \gamma'} \right]  + \frac{1}{1 - \gamma'}  \normfrho{\errpi{k}}\\
&= 
\frac{2\gamma'^k}{1 - \gamma'} \normfrho{V_0 - V^*} +
\frac{2\gamma' (1 - \gamma'^{k-1})}{(1 - \gamma')^2} \cdot \errvm +
\frac{1}{1 - \gamma'} \cdot \normfrho{\errpi{k}}.
\end{align*}
\end{proof}

%%%%%%%%%%%%%%%%%%%%%%%%%%%%%%%%%%%%%%%%%%%%%%%
\section{Additional experiments}
\label{sec:AdditionalExperiments}
% !TEX root =  OSVI.tex

In this section, we present further experiments on our algorithms. We consider multiple environments and settings. Appendix~\ref{sec:exp-envs} introduces the environments used. In the next subsections, three sets of experiments are presented.
\begin{enumerate}
    \item We further verify OS-VI's acceleration compared to VI in three environments for both PE and Control problems (Appendix~\ref{sec:exp-osvi-convergence}).
    \item We investigate the effect of model error on OS-VI. Two model error formulations are tested in all environments (Appendix~\ref{sec:exp-osvi-modelerror}).
    \item OS-Dyna is compared to Dyna and model-free algorithms using two learning schedules (Appendix~\ref{sec:exp-osdyna}).
\end{enumerate}

%%%%%%%%%%%%%%%%%%%%%%%%%%%%%%%%%%%%%%%%%%%%%%%
\subsection{Environments}
\label{sec:exp-envs}
We do experiments in three different environments. The details of them are as follows.

\paragraph{Modified Cliffwalk.}

We design a modified cliffwalk environment to better show the differences among the algorithms. The environment is a $6 \times 6$ gridworld shown in Figure~\ref{fig:env-cliff}. The starting state is the top-left corner. The agent receives reward of $20$ at the top-right corner, i.e., every action taken from this state gives reward of $r(x, a)=20$. There are three holes in the middle four cells of the first, third and the fifth row that if fallen into, the agent gets stuck and receives reward of $-32$, $-16$, and $-8$ on every step, respectively. There is a penalty of $-1$ in all other states to encourage finding the shortest route to the goal. The agent has four actions: UP, RIGHT, DOWN, and LEFT. Each action has $90\%$ chance to successfully move the agent in the chosen direction. With probability of $10\%$ one of the other three directions is randomly chosen and the agent moves in that direction. If the agent attempts to get out of the boundary of the environment, it will stay in place.

The discount factor of this environment is $\gamma = 0.9$. For policy evaluation experiments, we evaluate the optimal policy. The optimal policy is to take the safest route among (the two, or the only) closest paths to the right side of environment, as shown in Figure~\ref{fig:env-cliff}. Note that the smoothed models introduced in \eqref{eq:exp_model} will over-estimate the danger of cliffs and may take a suboptimal longer path to be safer. This will make the solution of inaccurate models suboptimal.

\begin{figure}[t]
    \centering
    \includegraphics[width=0.5\textwidth]{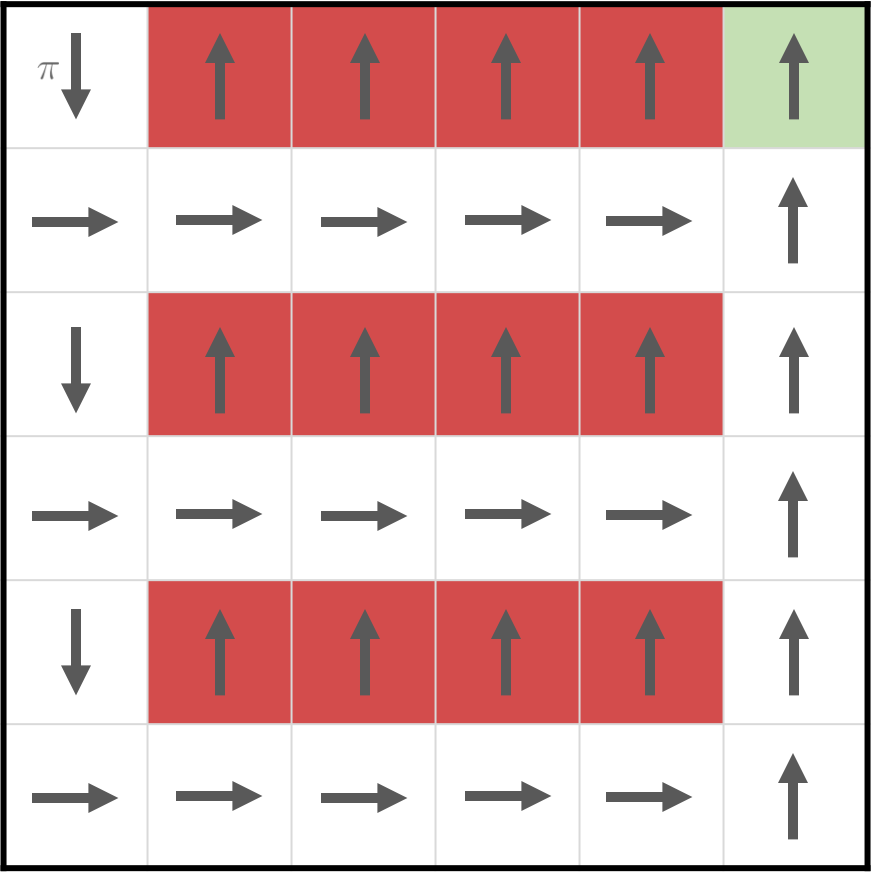}
    \caption{Modified Cliffwalk environment. The red cells in the first, third, and the third rows from the top are holes with reward of $-32$, $-16$, $-8$, respectively. The top-right corner is the goal state with reward of $20$. All other states have reward of $-1$. There is $10\%$ failure probability in actions. Arrows show the policy used for policy evaluation experiments.}
    \label{fig:env-cliff}
\end{figure}

\paragraph{Random MDP (Garnet).} We use the Garnet environment used by \cite{BhatnagarSuttonGhavamzadehLee2009,FarahmandGhavamzadeh2021}. The name is originally chosen as the acronym of \textit{Generic Average Reward Non-stationary Environment Testbed}. We use the same name, even though we implement it for discounted stationary MDPs similar to \citet{FarahmandGhavamzadeh2021}.

Our Garnet problem is parametrized by the tuple $(|\XX|, |\AA|, b_P, b_r)$. Here, $|\XX|$ and $|\AA|$ are the number of states and actions. The value $b_p$ is the branching factor of the environment, which is the number of possible next states for each state-action pair. When generating an instance, for each state-action pair, we randomly select $b_p$ states without replacement as the possible next states. Then, the transition distribution is generated by randomly choosing $b_p - 1$ points on the $(0, 1)$ interval. These points will partition the interval into $b_p$ parts, each corresponding to one of the transition probabilities. The reward function is only state-dependent. We select $b_r$ states without replacement, and for each selected state $x$, we assign $r(x)$ a uniformly sampled value in $(0, 1)$.

We generate \revised{}{100} randomly generated instances of the Garnet problem with $|\XX| = 50$, $|\AA| = 4$, $b_P = 3$, and $b_r = 5$. The discount factor is chosen as $\gamma = 0.99$. For policy evaluation experiments, we use the optimal policy of the instance. \revised{}{The plots for Garnet show the average values on the 100 problem instances along with a shaded area showing one standard error.}

\paragraph{Maze Environment.} This is a simple $3\times 3$ maze shown in Figure~\ref{fig:env-maze}. The top-left corner is the initial state, and the top-right corner is the goal state with reward of $1$. Similar to the modified clifffwalk, the agent has four actions: UP, RIGHT, DOWN, and LEFT. Each action has $90\%$ chance to successfully move the agent in the chosen direction. With probability of $10\%$ one of the other three directions is randomly chosen and the agent moves in that direction. If the agent attempts to get out of the boundary of the environment or hits a wall, it will stay in place. The discount factor is $\gamma=0.9$. We use the optimal policy shown by arrows in Figure~\ref{fig:env-maze} for policy evaluation experiments.

\begin{figure}[t]
    \centering
    \includegraphics[width=0.25\textwidth]{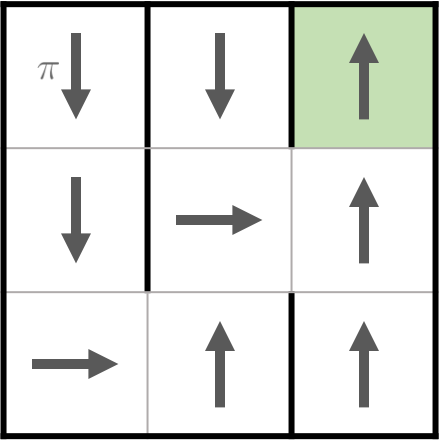}
    \caption{Maze environment. The top-right corner is the goal state with reward of $1$. There is $10\%$ failure probability in actions. Arrows show the policy used for policy evaluation experiments.}
    \label{fig:env-maze}
\end{figure}

%%%%%%%%%%%%%%%%%%%%%%%%%%%%%%%%%%%%%%%%%%%%%%%
\subsection{Convergence rate of OS-VI}
\label{sec:exp-osvi-convergence}
We empirically compare OS-VI with VI in all three environments. In this section and Appendix~\ref{sec:exp-osvi-modelerror}, we evaluate the normalized error of $V_k$, which is
\begin{equation}
\label{eq:normalized-error-PE}
    \frac{\norm{V_k - V^\pi}_1}{\norm{V^\pi}_1}
\end{equation}
for the policy evaluation problem, and
\begin{equation}
\label{eq:normalized-error-control}
    \frac{\norm{V_k - V^*}_1}{\norm{V^*}_1}
\end{equation}
for the Control problem. In the Control problem, while $V^{\pi_k}$ has qualitatively the same behaviour as $V_k$, it usually converges too fast due to the action-gap phenomenon~\citep{FarahmandNIPS2011}. Therefore, we consider the normalized error of $V_k$ instead of $V^{\pi_k}$ to better see the convergence behaviours.

Figure~\ref{fig:osvi-convergence} shows the results for both Policy Evaluation (top) and Control (bottom) problems. The dashed flat lines show the error values obtained by just using the model $\revised{}{\PKernelhat}$. These errors are relatively large and do not decrease. This shows that relying on an approximate model does not lead to accurate value function approximation or a good performance of the resulting policy.
VI does not have this issue as it queries the true $\PKernel$, and eventually converges to the true value function. Its convergence rate in terms of the number of queries to $\PKernel$, however, is significantly slower than OS-VI that can benefit from both $\PKernel$ and $\PKernelhat$. Even an approximate model with relatively large errors, characterized by $\lambda$, enjoys the accelerated convergence rate.

%We see that OS-VI converges much faster than VI when the model is accurate enough for both the policy evaluation and control problems. Also note that only using the model $\revised{}{\PKernelhat}$ would not be enough since the solution of models with error is not correct. 

\begin{figure}[!h]
      \begin{subfigure}[b]{1\linewidth}
        \centering
        \includegraphics[width=1\textwidth]{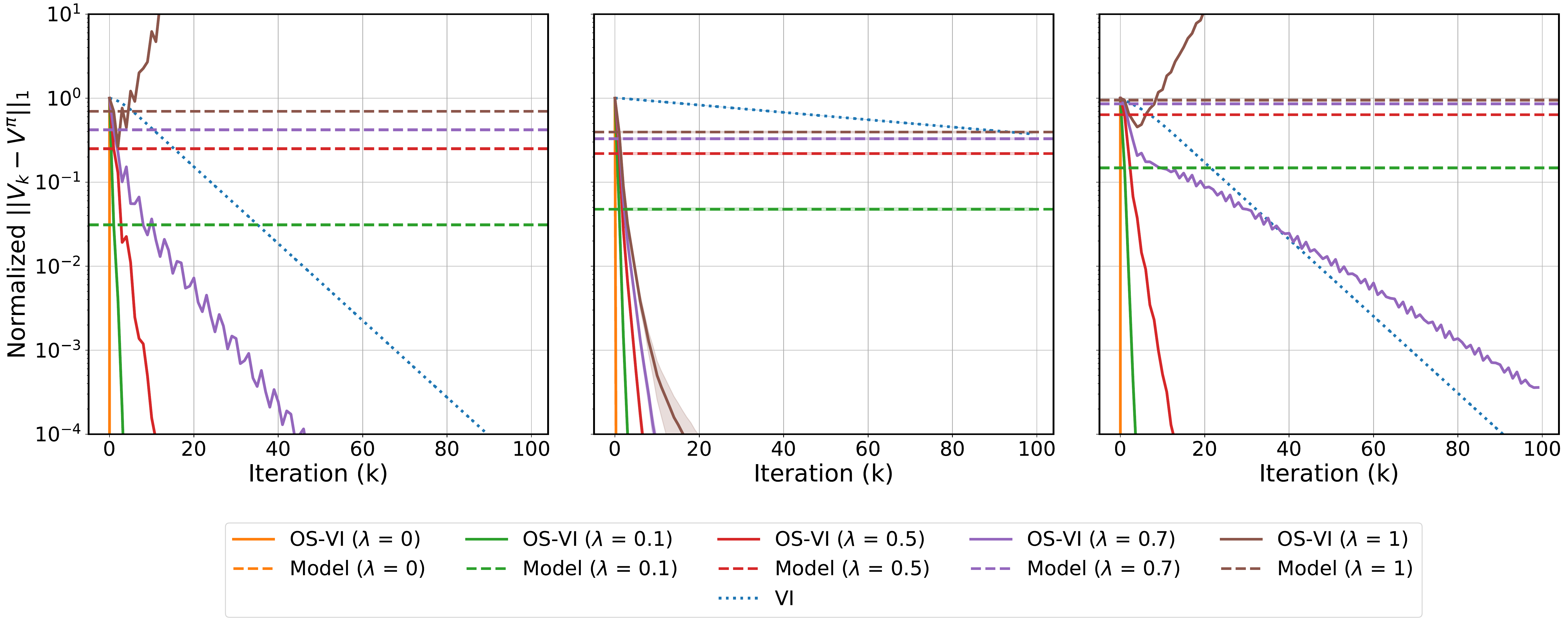}
        \caption{Policy Evaluation} 
        \label{fig:osvi-convergence-pe} 
      \end{subfigure}%%
     
      \begin{subfigure}[b]{1\linewidth}
        \centering
        \includegraphics[width=1\textwidth]{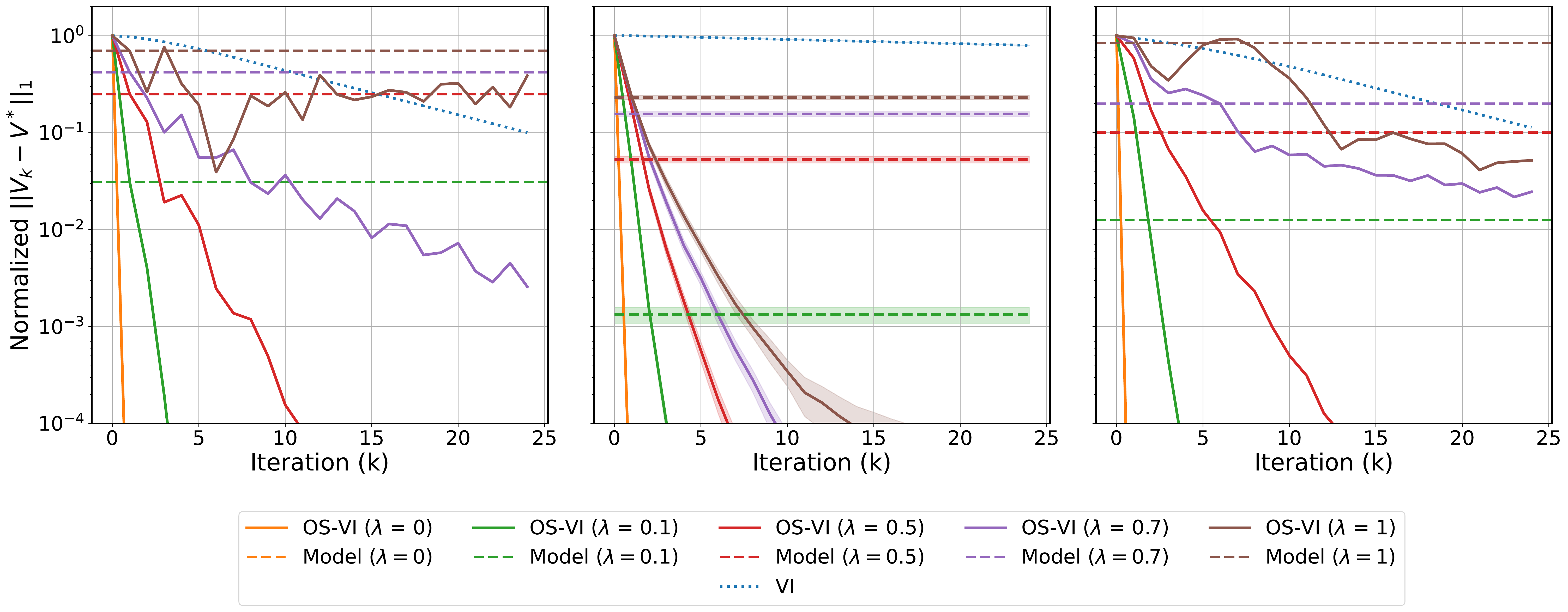}
        \caption{Control} 
        \label{fig1:fig:osvi-convergence-control} 
        \vspace{4ex}
      \end{subfigure} 
      \caption{Comparison of OS-VI with VI and the solution of the model, in the Policy Evaluation \textit{(a)} and the Control \textit{(b)} problems. The comparison is done in maze \textit{(left)}, Garnet \textit{(middle)}, and modified cliffwalk \textit{(right)} environments. \revised{}{Garnet plots are average of 100 instances. The shaded area is one standard error.} }
      \label{fig:osvi-convergence} 
\end{figure}

%%%%%%%%%%%%%%%%%%%%%%%%%%%%%%%%%%%%%%%%%%%%%%%
\subsection{Effect of model error on OS-VI}
\label{sec:exp-osvi-modelerror}
Our theoretical results (Theorems~\ref{theorem:ErrorPropagation-SupAndLp-PE} and~\ref{theorem:ErrorPropagation-SupAndLp-Control}) show that the convergence rate of OS-VI is affected by the accuracy of $\PKernelhat$, as its accuracy determines the effective discount factor $\gamma'$. In this section, we further investigate this effect. To do so, we run OS-VI with smoothed models obtained by a range of smoothing parameters. The smoothed models are defined in \eqref{eq:exp_model}. A higher smoothing parameter $\lambda$ will make the transition distributions more uniform and less accurate.

The normalized errors of OS-VI in several initial iterations ($k \in \{1, 3, 5, 7, 9\}$) of OS-VI are plotted against the smoothing parameter $\lambda$ in Figure~\ref{fig:osvi-modelerror}.
The difference among the iterations shows the convergence rate of OS-VI. For larger model errors, characterized by larger $\lambda$, the gap between errors in different iterations become smaller with a large absolute error. This means that the method has not progressed much towards the correct solution. In very large values of $\lambda$, close to one, there are cases that the order of iterations has changed, and later iterations have larger errors than earlier ones. This is when the value function is diverging.

To better show the divergence scenario of OS-VI, we introduce a new way of introducing model errors. The \emph{self-loop perturbed model} is defined as
\begin{equation}\label{eq:exp_selfloop_error}
\hat{\mathcal{P}}(\cdot | x, a; \mathcal{P}, \lambda) =  (1 - \lambda) \mathcal{P}( \cdot | x,a) +  \lambda I(\cdot | x),
\end{equation} 
where $I(\cdot | x)$ is the distribution of deterministically staying in state $x$. Larger values of $\lambda$ push the model transitions towards an MDP where no transitions occur. Similar to smoothing, larger values of $\lambda$ lead to more inaccurate models. The effect of this new model error is shown in Figure~\ref{fig:osvi-selfloop-modelerror}. We see that a clear divergence happens for large values of $\lambda$, around $0.5$--$0.6$. The error of OS-VI increases with each iteration.

\begin{figure}[!h]
      \begin{subfigure}[b]{1\linewidth}
        \centering
        \includegraphics[width=1\textwidth]{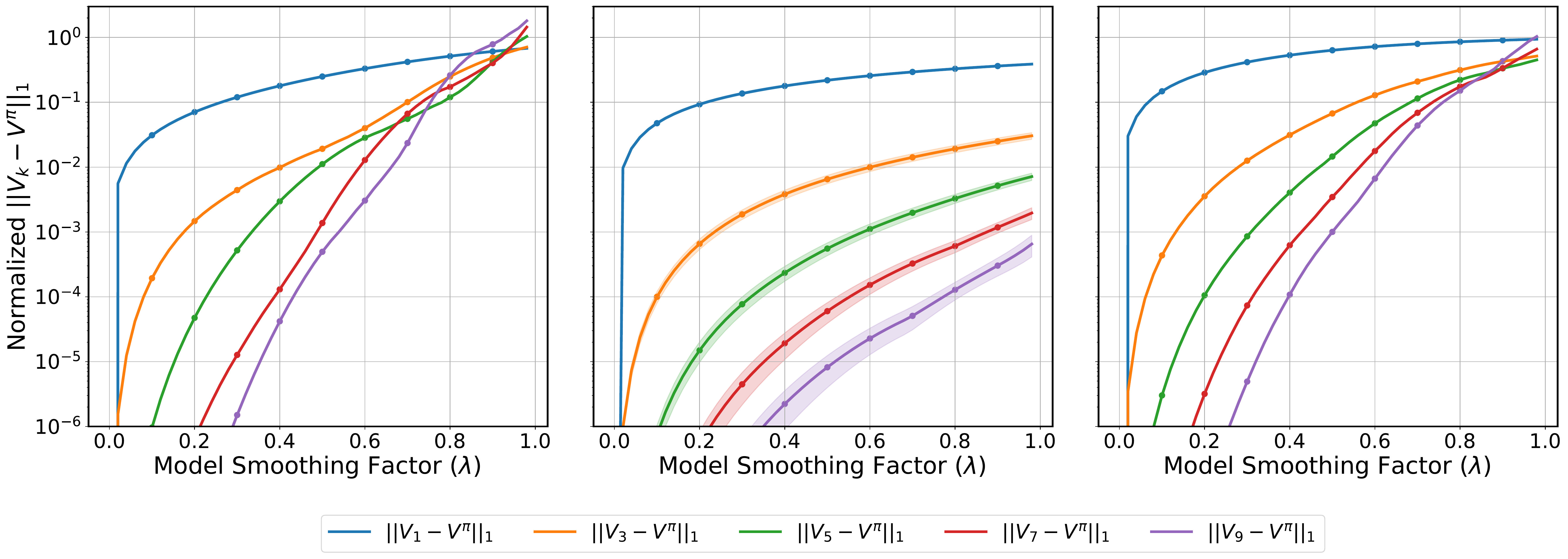}
        \caption{Policy Evaluation} 
        \label{fig:osvi-modelerror-pe} 
      \end{subfigure}%%
     
      \begin{subfigure}[b]{1\linewidth}
        \centering
        \includegraphics[width=1\textwidth]{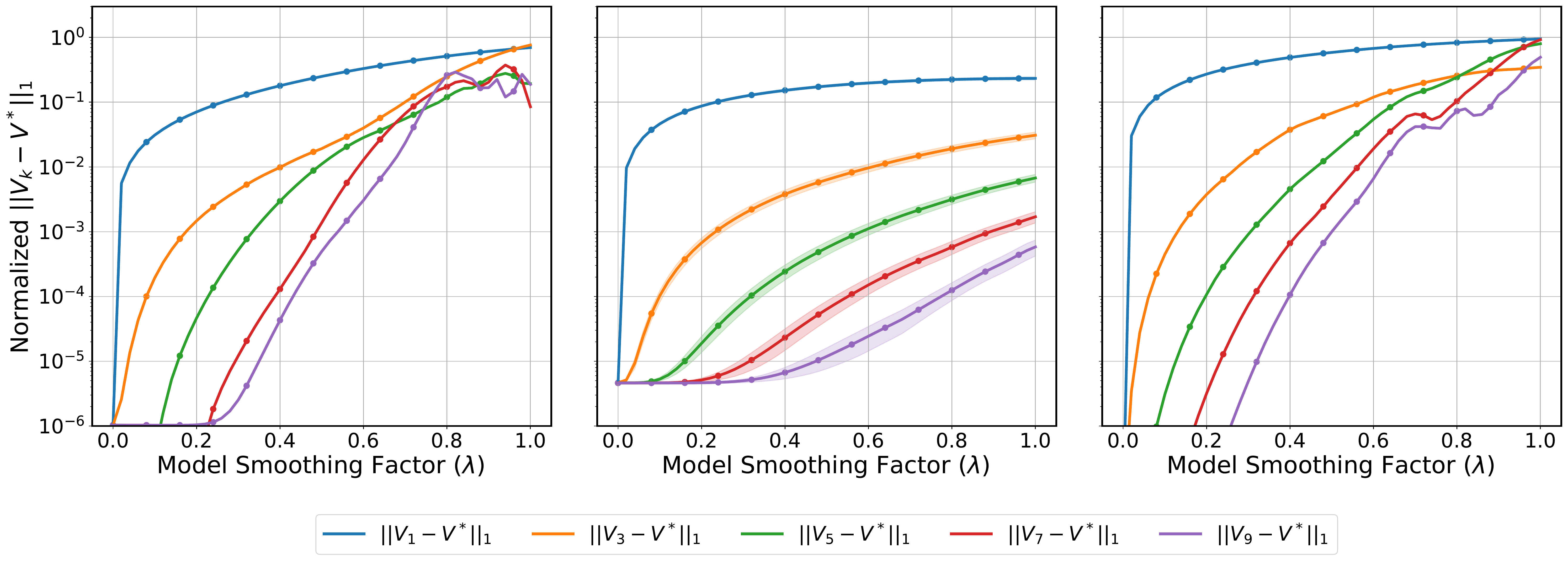}
        \caption{Control} 
        \label{fig1:fig:osvi-modelerror-control} 
        \vspace{4ex}
      \end{subfigure} 
      \caption{Effect of smoothing on convergence of OS-VI in Policy Evaluation \textit{(a)} and Control \textit{(b)} problems. Models are obtained by smoothing as in \eqref{eq:exp_model}. The comparison is done in maze \textit{(left)}, Garnet \textit{(middle)}, and modified cliffwalk \textit{(right)} environments. \revised{}{Garnet plots are average of 100 instances. The shaded area is one standard error.} }
      \label{fig:osvi-modelerror} 
\end{figure}

\begin{figure}[!h]
      \begin{subfigure}[b]{1\linewidth}
        \centering
        \includegraphics[width=1\textwidth]{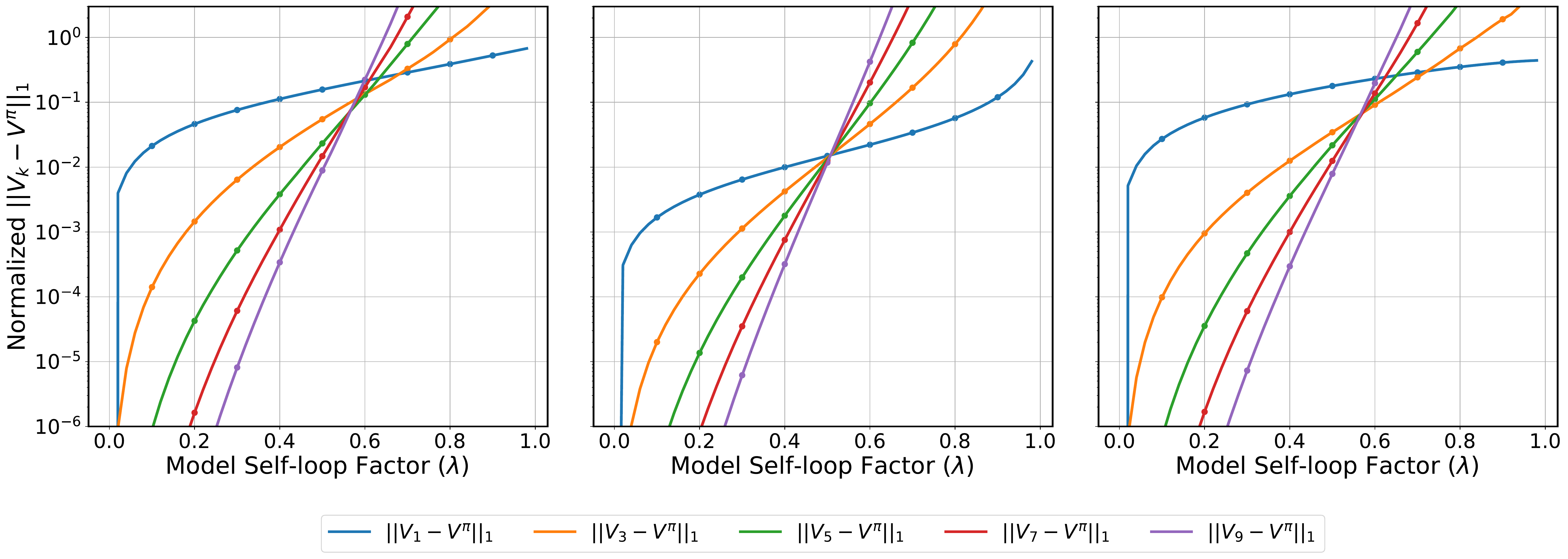}
        \caption{Policy Evaluation} 
        \label{fig:osvi-selfloop-modelerror-pe} 
      \end{subfigure}%%
     
      \begin{subfigure}[b]{1\linewidth}
        \centering
        \includegraphics[width=1\textwidth]{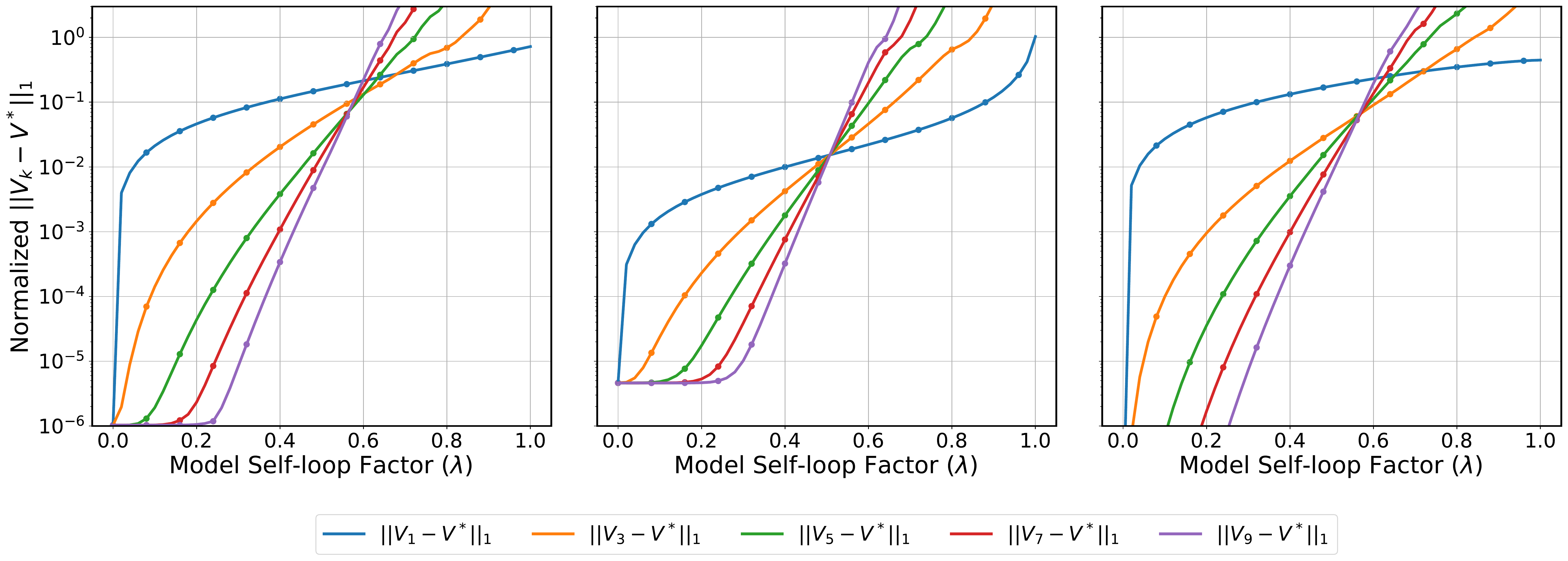}
        \caption{Control} 
        \label{fig1:fig:osvi-selfloop-modelerror-control} 
        \vspace{4ex}
      \end{subfigure} 
      \caption{Effect of self-loop error on convergence of OS-VI in the Policy Evaluation \textit{(a)} and the Control \textit{(b)} problems. Models are obtained by self-loop perturbation as in \eqref{eq:exp_selfloop_error}. The comparison is done in maze \textit{(left)}, Garnet \textit{(middle)}, and modified cliffwalk \textit{(right)} environments. \revised{}{Garnet plots are average of 100 instances. The shaded area is one standard error.}}
      \label{fig:osvi-selfloop-modelerror} 
\end{figure}

%%%%%%%%%%%%%%%%%%%%%%%%%%%%%%%%%%%%%%%%%%%%%%%
\subsection{Additional Experiments on OS-Dyna}
\label{sec:exp-osdyna}
In this section we compare OS-Dyna with Dyna and model-free algorithms. We focus on the modified cliffwalk environment and smoothed MLE models defined in Section~\ref{sec:Experiments}. 

In the implementation of OS-Dyna, $V_k$ is calculated from $\rbar_k$ and $\PKernelhat$ through exact dynamic programming to reduce the noise. Specifically, we find the optimal value function $V^*(\PKernelhat, \rbar_k)$ and the value function $V^\pi(\PKernelhat, \rbar_k)$ by performing VI on MDP $(\XX, \AA, \PKernelhat, \rbar_k)$. The same is true for Dyna. The value function is updated to the exact solution of the model on every iteration. 

For policy evaluation, we compare to TD-Learning, which updates the value function given each with each sample $(X_t, \pi(X_t), R_t, X_t')$ in the following way:
\begin{align}
    V(X_t) \leftarrow V(X_t) + \alpha_t \left(R_t + \gamma \cdot V(X_t') - V(X_t)
    \right).
\end{align}
Here, $\alpha_t$ is the learning rate at step $t$. We use constant and rescaled linear \citep{Wainwright2019StochasticAW} learning rate schedules in PE experiments. The rescaled linear schedule sets $\alpha_t = \frac{\alpha}{1 + (1 - u) \cdot t}$. We fine tune the learning rate schedule for each algorithm independently such that $0.1$ normalized error \eqref{eq:normalized-error-PE} is achieved as fast as possible. In constant learning rate, the value of $\alpha$ is $0.2$ for TD-Learning and $0.05$ for all OS-Dyna instances. In rescaled linear schedule, $\alpha, u = 1, 0.999$ for TD-Learning, and $\alpha, u = 0.8, 0.995$ for OS-Dyna.

The results for PE are shown in Figure~\ref{fig:osdyna-pe}. OS-Dyna converges faster than TD-Learning in both learning rate schedules. Also note that unlike OS-Dyna, Dyna does not converge to true values in presence of model error. It is worth mentioning that Dyna without model error is the best one can do in policy evaluation problem without any additional assumptions on the environment. Thus, it is expected to outperform all other algorithms.

\begin{figure}[!h]
\centering
\includegraphics[width=1\textwidth]{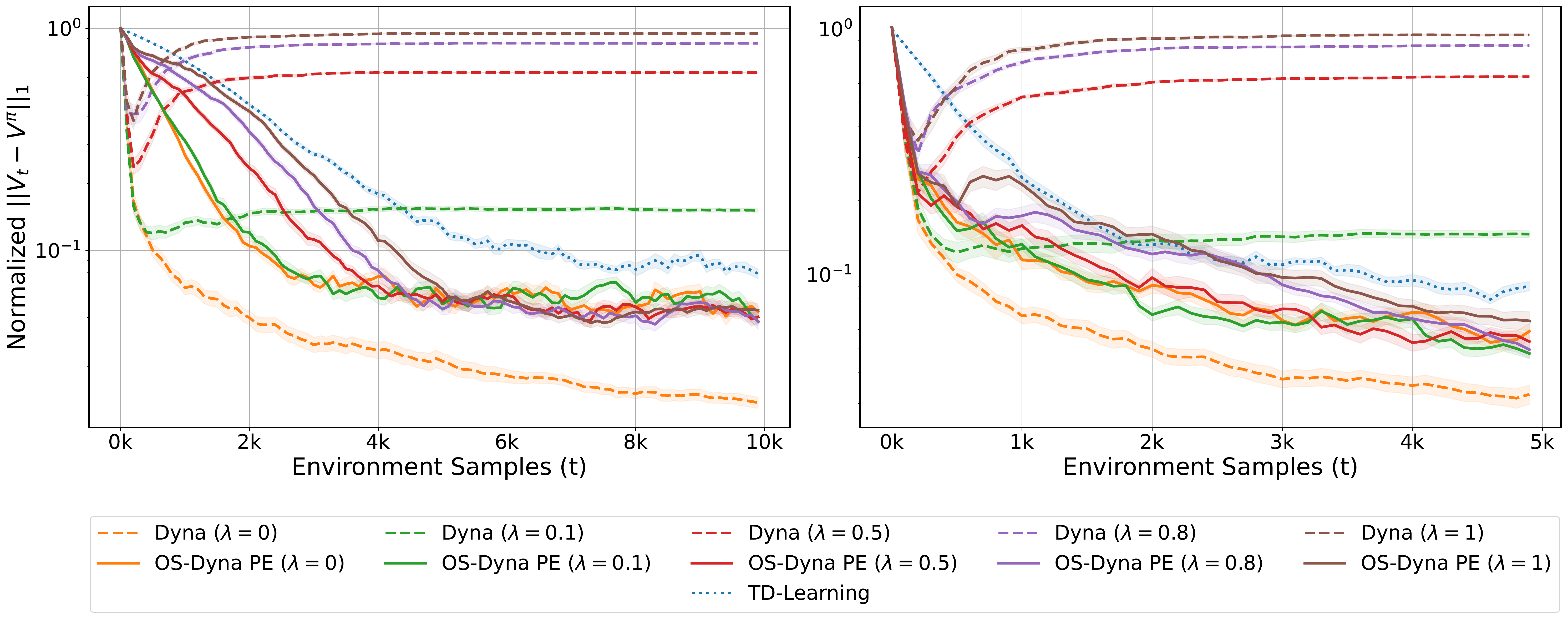}
\caption[short]{Comparison of OS-Dyna with Dyna and TD-learning in the Policy Evaluation problem using constant \textit{(Left) and rescaled linear \textit{(Right)} learning rates. This is average over 20 runs. The shaded area is one standard error.}}
\label{fig:osdyna-pe}
\end{figure}

In Control, OS-Dyna is compared with Dyna and Q-Learning using the delayed decay \citep{SuttonBarto2018} and rescaled linear \citep{Wainwright2019StochasticAW} learning rate schedules. The delayed decay sets $\alpha_t = \alpha$ for $t\le N$ and $\alpha_t = \alpha / (t-N)$ otherwise. The learning rates for each algorithm is fine tuned to achieve the optimal policy as fast as possible and stay stable on it. Figure~\ref{fig:osdyna-control} depicts the expected return $V^{\pi_t}(0)$ of the policy $\pi_t$ obtained by different algorithms (higher is better).

In delayed decay learning rate, we have $\alpha, N = 0.02, 68000$ for Q-Learning. For instances of OS-Dyna we have
\begin{itemize}
    \item OS-Dyna $(\lambda = 0)$: $\alpha, N = 0.02, 30000$
    \item OS-Dyna $(\lambda = 0.1)$: $\alpha, N = 0.02, 35000$
    \item OS-Dyna $(\lambda = 0.5)$: $\alpha, N = 0.02, 50000$
    \revised{}{
    \item OS-Dyna $(\lambda = 0.8)$: $\alpha, N = 0.02, 48000$
    \item OS-Dyna $(\lambda = 1)$: $\alpha, N = 0.02, 80000$
    }
\end{itemize}

In rescaled linear learning rate, we have $\alpha, u = 0.1, 0.9999$ for Q-Learning. For instances of OS-Dyna we have
\begin{itemize}
    \item OS-Dyna $(\lambda = 0)$: $\alpha, u = 1, 0.9$
    \item OS-Dyna $(\lambda = 0.1)$: $\alpha, u = 1, 0.9$
    \item OS-Dyna $(\lambda = 0.5)$: $\alpha, u = 1, 0.9995$
    \revised{}{
    \item OS-Dyna $(\lambda = 0.8)$: $\alpha, u = 1, 0.9995$
    \item OS-Dyna $(\lambda = 1)$: $\alpha, u = 1, 0.9995$
    }
\end{itemize}

\begin{figure}[tb]
\centering
\includegraphics[width=1\textwidth]{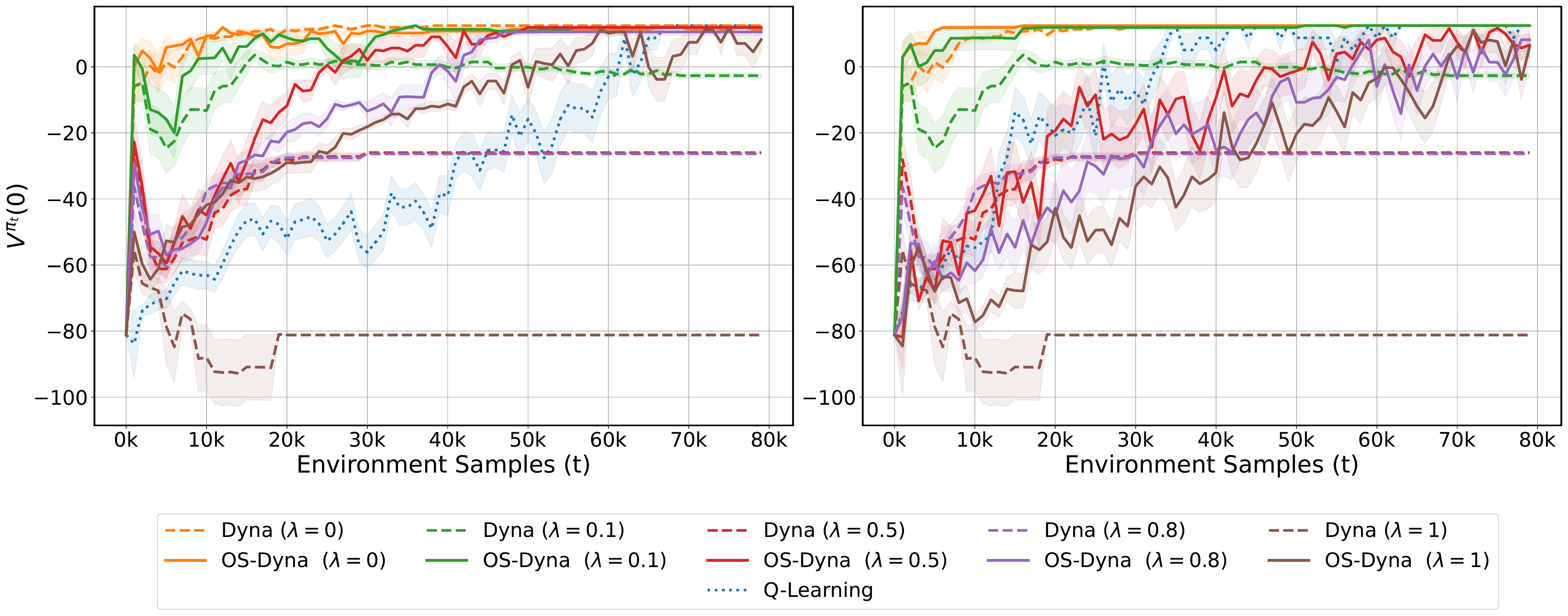}
\caption[short]{Comparison of OS-Dyna with Dyna and Q-learning in the Control problem using delayed decay \textit{(Left)} and rescaled linear \textit{(Right)} learning rates. This is average over 20 runs. The shaded area is one standard error.}
\label{fig:osdyna-control}
\end{figure}

\newpage
%%%%%%%%%%%%%%%%%%%%%%%%%%%%%%%%%%%%%%%%%%%%%%%
\section{Extended related work}
\label{sec:RelatedWork-Extended}
% !TEX root =  OSVI.tex

In this section, we provide a comparative analysis of the convergence behaviour of OS-VI (Appendix~\ref{sec:RelatedWork-Extended-Convergence-Comparison}).
We also point to some work where the  Jacobi and Gauss-Seidel iterations or some form of matrix splitting have been studied in the context of dynamic programming (Appendix~\ref{sec:RelatedWork-Extended-OtherMS-Method}).

%%%%%%%%%%%%%%%%%%%%%%%%%%%%%%%%%%%%%%%%%%%%%%%
\subsection{Comparison of convergence behaviours of OS-VI, value iteration, policy iteration, and modified policy iteration}
\label{sec:RelatedWork-Extended-Convergence-Comparison}

%\noindent \textbf{Value Iteration, Policy Iteration, Modified Policy Iteration}

We briefly compared the convergence behaviour of OS-VI with the convergence of VI in Section~\ref{sec:Theory} after stating Theorem~\ref{theorem:ErrorPropagation-SupAndLp-Control}. Here, we expand that discussion, and include comparison with PI and MPI as well.
We consider three aspects:
\begin{enumerate}
	\item \emph{(Model Error)} How does the model error affect the convergence limit?
	\item \emph{(Transient Error)} How fast the initial error in the approximation of value function diminishes as the iteration number $k$ grows?
	\item \emph{(Approximation Error Amplification)} How are the errors at each step of these algorithms amplified and are affecting the outcome policy?
\end{enumerate}

Let us discuss the \emph{Model Error} first, as it is a crucial difference between OS-VI and other methods such as VI, PI, and MPI.
Suppose that we do not know $\PKernel$, but only have access to $\PKernelhat \neq \PKernel$, and we use the approximate model with VI~\eqref{eq:VI} (that is, we perform
$V_k \leftarrow \max_\pi \{ \rpi + \gamma \PKernelhat V_{k-1} \}$), or PI or MPI, as shall be recalled soon.
These methods then converge to the optimal value function/policy w.r.t. the dynamics $\PKernelhat$. Let us denote the optimal value function w.r.t. $\PKernelhat$ by $\Vhat^*$, and the optimal policy $\pihat^*$.
These are, in general, different from the optimal value function $\Vopt$ and policy $\piopt$ w.r.t. the dynamics $\PKernel$ -- they are biased.

If we execute $\pihat^*$ in the true environment with dynamics $\PKernel$, its performance will be lower than the performance of the optimal policy $\piopt$, measured according to their corresponding value functions.
The performance can be upper bounded as follows~\citep[Theorem 7]{AvilaPiresSzepesvari2016}:
\begin{align*}
	\norm{ V^\piopt - V^{\pihat^*} }_\infty \leq
	\frac{2 \gamma}{1 - \gamma} \norm{ (\PKernel - \PKernelhat) \Vhat^* }_\infty
	\leq
	\frac{2 \gamma \Vmax}{1 - \gamma} \norm{ \PKernel - \PKernelhat }_\infty.
\end{align*}
For the methods that only rely on the approximate model $\PKernelhat$, this performance deterioration is in general inevitable. This is the essence of sim2real problem.\todo{cite!}

OS-VI does not have this issue. By using both $\PKernelhat$ and $\PKernel$, it brings the potential benefit of querying an approximate model (which is supposedly computationally cheaper), while guarding against converging to a biased solution. Of course, this is under the condition that it converges, which is the case if the model is accurate enough. Note that OS-VI uses more information (both $\PKernelhat$ and $\PKernel$) than VI, PI, or MPI, which only use either $\PKernel$ or $\PKernelhat$ -- they cannot benefit from both.

Since OS-VI needs access to $\PKernel$, one may wonder if it is beneficial to use OS-VI after all, as opposed to using VI, PI, or MPI with the true model $\PKernel$. The answer to this question depends on  the convergence rate of these algorithms. A faster algorithm requires fewer queries to the true model $\PKernel$.
The rest of this subsection is dedicated to studying their convergence rates, focusing on the effect of transient error and the approximation error amplification.

To set the stage, let us introduce the approximate models of VI, PI, and MPI. These should be compared with~\eqref{eq:AOS-VI-PE} and~\eqref{eq:AOS-VI-Control1}-\eqref{eq:AOS-VI-Control2} in Section~\ref{sec:Theory}.

Recall that VI iteratively applies the Bellman operator $T$ to the previous value function $V_{k-1}$ in order to obtain the new approximation $V_k$ of the value function, cf.~\eqref{eq:VI}. As discussed for OS-VI in the beginning of Section~\ref{sec:Theory-PE}, there might be an error in each step, which we formalize by considering that an error function $\errv{k}$ is added to the operation of the exact VI:
\begin{align*}
%\label{eq:AVI}
	V_{k} =
		\begin{cases}
			\Tpi V_{k-1} + \errv{k} , & \text{(Policy Evaluation)} \\
			\Topt V_{k-1} + \errv{k}. & \text{(Control)}
		\end{cases}
\end{align*}
When $\eps_k^\text{value} = 0$, we get the exact VI~\eqref{eq:VI}.

At each iteration of PI, we first compute the greedy policy $\pi_k \leftarrow \pigreedy(V_{k-1})$ and then perform PE in order to compute $V^{\pi_k}$. In approximate PI, we might have error at computing the greedy policy or computing its value function. These errors are modelled as
\begin{align}
\label{eq:API-PolicyImprovement}
	& T^{\pi_k} V_{k-1} = \Topt V_{k-1}  + \errpi{k},   \, \qquad \qquad \qquad \quad \text{(policy improvement)} \\
\label{eq:API-PE}	
	& V_{k} = V^{\pi_k} + \errv{k} \; [= (\Id - \PKernel^{\pi_{k}})^{-1} r^{\pi_{k}} + \errv{k} ]. \,\text{(policy evaluation)}
\end{align}

The Modified PI is similar to PI with the difference that instead of aiming to compute $V^{\pi_k}$ at each step exactly (ignoring the $\errv{k}$ term for the moment), it only partially moves towards it by applying $T^{\pi_k}$ for $m \geq 1$ times. That is, $V_{k} \leftarrow (T^{\pi_{k}})^m V_{k-1}$. When $m \ra \infty$, by the contraction property of the Bellman operator, $V_k \ra V^{\pi_k}$. This is exactly the same as PI. When $m = 1$, it is the same as VI.
The approximate MPI is modelled as
\begin{align*}
%\label{eq:AMPI-PolicyImprovement}
	& T^{\pi_k} V_{k-1} = \Topt V_{k-1}  + \errpi{k},   \quad \text{(policy improvement)}. \\
%\label{eq:AMPI-PE}
	& V_{k} = (T^{\pi_{k}})^m V_{k-1} + \errv{k}. \qquad \text{(partial policy evaluation)}.
\end{align*}

To simplify the comparison, we focus on the supremum norm-based analysis for each of these methods.
Some of the existing results are not exactly in the form that we need. For example, they take $k \ra \infty$, which loses the information about the transient error. % Whenever possible, we re-use them in order to obtain error bounds for VI, PI, and MPI.

\noindent \textbf{Convergence of value iteration.}
We consider VI (PE) and VI (Control) separately.
For VI (PE), we derive the bound as follows:
\begin{align*}
	\Vpi - V_k & =
	\Tpi \Vpi - (\Tpi V_{k-1} + \errv{k}) =
	\gamma \PKernelpi (\Vpi - V_{k-1} ) + \errv{k}
	= \cdots \\
	& =
	\sum_{i = 0}^{k-1} (\gamma \PKernelpi)^i \errv{k - i} +
	(\gamma \PKernelpi)^k (\Vpi - V_0).
\end{align*}

Assume that $\smallnorm{\errv{i}}_\infty \le \errvm$ for all $i = 1, \dotsc, k$. We then have
\begin{align}
\label{eq:AVI-ErrorBound-PE}
	\norm{\Vpi - V_k}_\infty \leq 
	\frac{1 - \gamma^{k}}{1 - \gamma} \errvm + 
	\gamma^k \norm{\Vpi - V_0}_\infty.
\end{align}

This upper bound shows the effect of transient error and the approximation error at each iteration.
The transient error decays with the rate of $O(\gamma^k)$. This can go to zero quite slowly when the discount factor is close to one.
The approximation errors $\errv{i}$ of the approximate VI procedure, upper bounded by $\errvm$,  are amplified by a factor of $(1 - \gamma)^{-1}$. Asymptotically, we have $\frac{\errvm}{1 - \gamma}$ behaviour.\todo{We have a forgetting behaviour too, which this bound doesn't capture. -AMF}
\todo{Do we have any experiment with $\gamma = 0.999$? -AMF}

This result should be compared with Theorem~\ref{theorem:ErrorPropagation-SupAndLp-PE} with the choice of $\star = \infty$, which shows that OS-VI (PE) behaves as
\begin{align}
\label{eq:OSVI-ErrorBound-PE}
	\norm{\Vpi - V_{k} }_{\infty} \leq 
	\frac{1 - \gamma'^k}{1 - \gamma'} \errvm + 
	\gamma'^k \norm{\Vpi - V_0 }_{\infty},
\end{align}
with $\gamma' = \frac{\gamma}{1 - \gamma} \smallnorm{ \PKernelpi - \PKernelhat^\pi }_\infty$.
When the model is accurate enough ($\smallnorm{ \PKernelpi - \PKernelhat^\pi }_\infty < 1 - \gamma$), the effective discount factor $\gamma'$ is smaller than the discount factor $\gamma$ of the original MDP. Consequently, the transient error of OS-VI can decay significantly faster than VI's. Moreover, the error amplification of $\errvm$ is by a factor of $(1 - \gamma')^{-1}$, which is smaller than that of approximate VI under the same condition on model accuracy.

We also have a similar result for VI (Control).
\revised{}{
We follow the proof of Equation~(2.2) of~\citet{Munos07} to get that for the greedy policy $\pi_k \leftarrow \pigreedy(V_{k-1})$, we have
\begin{align}
\label{eq:AVI-VktoVpik}
	\norm{\Vopt - V^{\pi_k} }_\infty \leq
	\frac{2\gamma}{1 - \gamma} \norm{\Vopt - V_{k-1}}_\infty.
\end{align}
To upper bound $\smallnorm{\Vopt - V_{k-1}}_\infty$, we add and subtract $\Topt V_{k-2}$ to $\Vopt - V_{k-1}$, and benefit from $\Vopt = \Topt \Vopt$ and the triangle inequality to get
\begin{align*}
	\norm{\Vopt - V_{k-1}}_\infty & \leq 
	\norm{\Topt \Vopt - \Topt V_{k-2}}_\infty + \norm{\Topt V_{k-2} - V_{k-1}}
	\\ &
	\leq
	\gamma \norm{\Vopt - V_{k-2}}_\infty + \norm{\Topt V_{k-2} - V_{k-1}}
	\\ &
	= \gamma \norm{\Vopt - V_{k-2}}_\infty + \norm{\errv{k-1}}_\infty.
\end{align*}
Repeating this argument, we obtain
\begin{align*}
	\norm{\Vopt - V_{k-1}}_\infty \leq 
	\sum_{i=0}^{k-1} \gamma^i \norm{\errv{k - i} }_\infty 
	+
	\gamma^k \norm{\Vopt - V_0}_\infty.
\end{align*}

Plugging this inequality in~\eqref{eq:AVI-VktoVpik} and using the same assumption that $\smallnorm{\errv{i}}_\infty \leq \errvm$ for all $i = 1, 2, \dotsc$ lead to
}
\begin{align}
\label{eq:AVI-ErrorBound-Control}
	\norm{\Vopt - V^{\pi_k} }_\infty \leq
	\frac{2 \gamma}{1 - \gamma}
	\left[
		\frac{1 - \gamma^{k-1}}{1 - \gamma} \errvm +
		\gamma^{k-1} \norm{\Vopt - V_0}_\infty
	\right].
\end{align}

The transient behaviour is $O(\gamma^k)$, as in VI (PE). The amplification of the approximation errors is by a factor of $(1 - \gamma)^{-2}$.
The result for VI (Control) should be compared with Theorem~\ref{theorem:ErrorPropagation-SupAndLp-Control} with $\star = \infty$, which is
\begin{align}
\label{eq:OSVI-ErrorBound-Control}
	\norm{\Vopt - V^{\pi_k}}_{\infty} \leq
	\frac{2\gamma'^k}{1 - \gamma'} \norm{V_0 - V^*}_\infty +
	\frac{2\gamma' (1 - \gamma'^{k-1})}{(1 - \gamma')^2} \,  \errvm +
	\frac{1}{1 - \gamma'} \, \norm{\errpi{k}}_\infty.
\end{align}

As in the OS-VI (PE) case, the transient behaviour is $O(\gamma'^k)$, which can be much faster than VI's whenever the approximate model is accurate enough.
The error amplification is $(1 - \gamma')^{-2}$, which is smaller under the same condition.
We have an extra $\smallnorm{\errpi{k}}_\infty$ term, which is the possible error in the computation of the $S$-improved policy.
The parallel for VI would be the error in the computation of the greedy policy. In the VI model considered above, we did not consider such a source of error. %, but we will in the approximate PI model next.

%%%%%%%%%%%%%%%%%%%%%%%%%%%%%%%%%%%%%%%%%%%%%%%
\noindent \textbf{Convergence of policy iteration.}
Considering that $\errpi{i} = 0$ in~\eqref{eq:API-PolicyImprovement}, we use Lemma~4 of~\citet{Munos03}, which states that
\begin{align*}
	\Vopt - V^{\pi_k} \vecle
	\gamma \PKernel^\piopt (\Vopt - V^{\pi_{k-1}} ) +
	\gamma \left[ \PKernel^{\pi_{k}} (\Id - \gamma \PKernel^{\pi_k})^{-1} (\Id - \gamma \PKernel^{\pi_{k-1}}) - \PKernel^\piopt
			\right] (V_{k-1} - V^{\pi_{k-1}}).
\end{align*}
Noticing that $0 \vecle \Vopt - V^{\pi_k}$, by taking the absolute values of both sides, and using Jensen's inequality, we get that
\begin{align*}
	\left| \Vopt - V^{\pi_k} \right| \vecle &
	\gamma \PKernel^\piopt \left| \Vopt - V^{\pi_{k-1}} \right| + {} \\
	&
	\gamma \left[ \PKernel^{\pi_{k}} (\Id - \gamma \PKernel^{\pi_k})^{-1}
	+ \gamma \PKernel^{\pi_{k}} (\Id - \gamma \PKernel^{\pi_k})^{-1} \PKernel^{\pi_{k-1}} + \PKernel^\piopt
			\right] \left| V_{k-1} - V^{\pi_{k-1}} \right|.
\end{align*}
Taking the supremum of both sides over the state space, and benefitting from $\norm{\PKernel^\pi}_\infty = 1$ and that $\norm{(\Id - \gamma \PKernel^{\pi})^{-1}}_\infty \leq \frac{1}{1 - \gamma}$ (for any $\pi$), we obtain
\begin{align*}
	\norm{\Vopt - V^{\pi_k}}_\infty 
	\leq
	\gamma \norm{\Vopt - V^{\pi_{k-1}}}_\infty + \frac{2 \gamma}{1 - \gamma} \norm{\errv{k-1}}_\infty.
\end{align*}
Expanding this, we get
\begin{align*}
	\norm{\Vopt - V^{\pi_k}}_\infty 
	\leq
	\frac{2\gamma}{1 - \gamma}
	\sum_{i=0}^{k - 1} \gamma^i  \norm{\errv{k - 1 -i}}_\infty
	+
	\gamma^{k-1} \norm{\Vopt - V^{\pi_0}}_\infty.
\end{align*}
%XXX OLD and INCORRECT XXX
%\begin{align*}
%	\norm{\Vopt - V^{\pi_k}}_\infty 
%	\leq
%	\frac{2\gamma}{1 - \gamma}
%	\left[
%		\sum_{i=0}^{k - 2} \gamma^i  \norm{\errv{k - 1 -i}}_\infty
%		+
%		\gamma^{k-1} \norm{V^{\pi_0} - V_0}_\infty
%	\right].
%\end{align*}
%
Assuming that $\smallnorm{\errv{i}}_\infty \le \errvm$ for all $i = 1, \dotsc, k-1$, we get
\begin{align}
\label{eq:API-ErrorBound}
	\norm{\Vopt - V^{\pi_k}}_\infty 
	\leq
	\frac{2\gamma (1 - \gamma^{k}) }{(1 - \gamma)^2} \errvm +
	\gamma^{k-1} \norm{ \Vopt - V^{\pi_0} }_\infty.
\end{align}
%
%XXX OLD and INCORRECT XXX
%\begin{align*}
%	\norm{\Vopt - V^{\pi_k}}_\infty 
%	\leq
%	\frac{2\gamma (1 - \gamma^{k-1}) }{(1 - \gamma)^2} \errvm +
%	\frac{2 \gamma^k}{1 - \gamma} \norm{ V^{\pi_0} - V_0 }_\infty.
%\end{align*}

This shows that approximate PI has the transient behaviour of $O(\gamma^k)$, and it amplifies the PE error $\errv{}$ by a factor of $(1 - \gamma)^{-2}$.
This is the same as VI, and the comparison with OS-VI is exactly the same: whenever the model error is small enough, approximate OS-VI benefits from the approximate model $\PKernelhat$ and improves both the transient error rate and the error amplification.

The results of~\citet{Munos03} does not consider the possibility of $\errpi{i}$ being non-zero. For that, we report the asymptotic result of Proposition 6.2~\citet{Bertsekas96}, which states that
\begin{align}
\label{eq:API-ErrorBound-AsymptoticOnly}
	\limsup_{k \ra \infty} \norm{\Vopt - V^{\pi_k}}_\infty 
	\leq
	\frac{2 \gamma \errvm + \errpim}{(1 - \gamma)^2},
\end{align}
in which $\smallnorm{\errpi{i}}_\infty \leq \errpim$ for all $i \geq 1$.

%%%%%%%%%%%%%%%%%%%%%%%%%%%%%%%%%%%%%%%%%%%%%%%
\noindent \textbf{Convergence of modified policy iteration.}
Lemma 4 of~\citet{ScherrerGhavamzadehGabillonetal2015} leads to
\begin{align}
\label{eq:AMPI-ErrorBound}
	\norm{\Vopt - V^{\pi_k}}_\infty \leq
	\frac{2 \gamma (1 - \gamma^{k-1}) \errvm + (1 - \gamma^k) \errpim }{(1 - \gamma)^2} +
	\frac{2 \gamma^k}{1 - \gamma} \norm{\Vopt - V_0}_\infty.
\end{align}

The transient behaviour is $O(\gamma^k)$, and the error amplification is $(1 - \gamma)^{-2}$ for both PE error $\errvm$ and greedification error $\errpim$.
The comparison with OS-VI is as before, and shows that OS-VI can improve the convergence rate of the transient error as well as reducing the error amplification effect, if the model is accurate enough.

\revised{}{
All these error bounds are summarized in Table~\ref{tab:ErrorBoundsComparison} for ease of comparison. For the error amplification terms, we only consider the asymptotic behaviour by letting $k \ra \infty$ to simplify the presentation.
}

%%%%%%%%%%%%%%%%%%%%%%%
\begin{table}[t]
  \caption{\revised{}{The transient and error amplification effects on $\smallnorm{\Vopt - V^{\pi_k}}_\infty$ for various method}}
  \label{tab:ErrorBoundsComparison}
  \centering
  \begin{tabular}{l l l}
    \toprule
    Method     & Transient Error     & Error Amplification \\
    \midrule
VI (PE)~\eqref{eq:AVI-ErrorBound-PE} 
	& $\gamma^k \norm{\Vpi - V_0}_\infty$ 
	& $\frac{ \errvm }{1 - \gamma}   $
	\\
VI (Control)~\eqref{eq:AVI-ErrorBound-Control}
		& $\frac{2 \gamma^{k}}{1 - \gamma} \norm{\Vopt - V_0}_\infty$
		& $ \frac{2 \gamma \errvm }{(1 - \gamma)^2}   $
	\\
PI~\eqref{eq:API-ErrorBound}-\eqref{eq:API-ErrorBound-AsymptoticOnly}
 		& $\gamma^{k-1} \norm{ \Vopt - V^{\pi_0} }_\infty$
		& $\frac{2\gamma \errvm + \errpim}{(1 - \gamma)^2} $
	\\
MPI~\eqref{eq:AMPI-ErrorBound} 
	& $\frac{2 \gamma^k}{1 - \gamma} \norm{\Vopt - V_0}_\infty$
	& $\frac{2 \gamma \errvm + \errpim }{(1 - \gamma)^2}$
	\\
OS-VI (PE)~\eqref{eq:OSVI-ErrorBound-PE} 
	& $\gamma'^k \norm{\Vpi - V_0 }_{\infty}$
	& $\frac{\errvm}{1 - \gamma'} $
	\\
OS-VI (Control)~\eqref{eq:OSVI-ErrorBound-Control}
	& $\frac{2\gamma'^k}{1 - \gamma'} \norm{\Vopt - V_0}_\infty$
	& $\frac{2\gamma' \errvm }{(1 - \gamma')^2} +
	\frac{\smallnorm{\errpi{k}}_\infty}{1 - \gamma'}$
	\\
    \bottomrule
  \end{tabular}
\end{table}
%%%%%%%%%%%%%%%%%%%%%%%

%%%%%%%%%%%%%%%%%%%%%%%%%%%%%%%%%%%%%%%%%%%%%%%
%%%%%%%%%%%%%%%%%%%%%%%%%%%%%%%%%%%%%%%%%%%%%%%
%%%%%%%%%%%%%%%%%%%%%%%%%%%%%%%%%%%%%%%%%%%%%%%

\todo{There are some papers here that can be cited. I am commenting them out for now. -AMF}
\subsection{Matrix splitting, Jacobi, and Gauss-Seidel iterations for dynamic programming}
\label{sec:RelatedWork-Extended-OtherMS-Method}

\citet{KushnerKleinman1971} is one of the earliest paper we could find that mentions the Jacobi and Gauss-Seidel procedures for computing the value function. The focus of that work, however, is to  propose accelerated variants of the Jacobi and Gauss-Seidel procedures through an over-relaxation procedure (cf. Section 3.1 of~\citealt{Varga2000}).

\citet{BaconPrecup2016} provide a matrix splitting perspective on planning with options. Their use of planning does not refer to the problem of Control (finding the optimal policy), but refers to the PE problem given a set of options that are consistent with the policy that is evaluated.
They show that the computation of the value function using a given set of options can be interpreted as a particular choice of matrix splitting. The splitting depends on the dynamics, intra-option policies, the policy over options, and the termination probability of options. They show that decreasing the probability of termination, which corresponds to longer execution of options, leads to faster convergence of the planning.
Although this is one of a few work that makes the connection between a dynamic programming-based approach and matrix splitting in numerical linear algebra explicit, it is fundamentally different from ours.
They use matrix splitting to shed light on what planning with option does, but do not suggest a new algorithm. Their studied algorithm (VI-like procedures using options) does not benefit from the existence of an approximate $\PKernelhat$ to accelerate. The source of acceleration is the multi-step behaviour of an option.
On a more detailed note, the matrix splitting in their work is of the \emph{regular splitting} type, which has nice properties but is not suitable for the analysis of the splitting in this work. % That is why they could use tools from~\citet{Varga2000} to prove convergence, whereas we had to develop all convergence results ourselves. This allowed us to deal with the nonlinearity of $\Sopt$ operator.
% \todo{Maybe a bit defensive. -AMF}

The connection between multi-step models and matrix splitting is further developed in Chapter 4 of \citet{Bacon2018}.
He starts from the $n$-step model, and its corresponding Bellman-like equation for policy evaluation, which would be $\Vpi = \sum_{t = 0}^{n - 1} (\gamma \PKernelpi)^t \rpi + (\gamma \PKernelpi)^n \Vpi$ (when $n = 1$, this is the usual Bellman equation).
The value of $n$ determines the number of unrolling steps.
When $n$ is randomly selected through a process that at each step decides whether to terminate or continue the unrolling with a probability determined by a function $\lambda: \XX \times \XX \ra [0,1]$, where $\lambda(x,x')$ depends on two consecutive states $x$ and $x'$, this leads to the so-called $\lambda$-models. This is closely related to the $\beta$-models of~\citep{Sutton1995}.
A $\lambda$-model leads to a generalized Bellman equation. Bacon interprets the generalized Bellman equation as a particular choice of matrix splitting. The termination function $\lambda$ leads to a matrix splitting $\Mpi(\lambda)$ and $\Npi(\lambda)$. This in turn determines the convergence rate of the iterative VI-like procedure for the computation of the value function, as the convergence rate depends on the spectral radius of $\Mpi(\lambda)^{-1} \Npi(\lambda)$.
Similar remarks as the case of options applies: 
\citet[Chapter 4]{Bacon2018} sheds light to why already existing methods work, but it does not introduce a new algorithm; the analyzed algorithms do not benefit from an existence of an approximate model $\PKernelhat$.

\citet{Porteus1975} propose several transformations to the reward and the probability transition matrix with the goal of improving the computational cost of solving the transformed MDP.
One of the transformations, called \emph{pre-inverse transform}, has some similarities with the operator splitting of this work. The end result, however, is different.
That work considers a matrix $W^\pi$ and define $\tilde{r}^\pi = (\Id - W^\pi)^{-1} \rpi$ and
$\tilde{P}^\pi = (\Id - W^\pi)^{-1} (\PKernelpi - W^\pi)$.
It requires that for any $\pi$, the matrix $W^\pi$ be a lower triangular and be dominated by $\PKernelpi$ as $0 \leq W^\pi \leq \PKernelpi$ ($W^\pi$ does not need to be a stochastic matrix). The paper then suggests performing one step of the Value Iteration as
\[
	V_{k} \leftarrow \argmax_{\pi} (\Id - W^\pi)^{-1} \left[ \rpi + (\PKernelpi - W^\pi) V_{k-1} \right].
\]
If $W^\pi$ was $\PKernelpihat$, this would be the same as~\eqref{eq:OSVI-Control}. But we consider a probabilistic model $\PKernelpihat$, which does not satisfy the setup of that work, including being a lower triangular or dominated by $\PKernelpi$.
That paper in fact considers $W^\pi$ to be the lower triangular part of $\PKernelpi$ (i.e., $[W^\pi]_{x,x'} = [\PKernelpi]_{x,x'}$ for $1 \leq x' \leq x \leq | \XX |$) and zero otherwise), and then benefits from the lower triangularity of $W^\pi$ to re-derive the Gauss-Seidel variant of VI.

Porteus referred to \citet{Varga1962} to motivate another variant of pre-inverse transformation, in which $W^\pi$ is not only dominated by $\PKernelpi$, but also is diagonal. In that case, 
larger $W^\pi$ leads to smaller spectral radius, which determines the convergence rate. If $W^\pi$ is selected to be the diagonal part of $\PKernelpi$ (i.e., $[W^\pi]_{x,x} = [\PKernelpi]_{x,x}$, and zero for other elements), one retrieves the Jacobi variant of VI.

Although it is difficult to be sure why the condition $0 \leq W^\pi \leq \PKernelpi$ was imposed, the paper's reference to \citet{Varga1962} suggests that he was influenced by the concept of regular splitting, which is satisfied under the aforementioned condition.\todo{Maybe can be removed or commented out, as I cannot be sure about someone's intention 47 years ago. -AMF}

\fi

%%%%%%%%%%%%%%%%%%%%%%%%%%%%%%%%%%%%%%%%%%%%%%%
\begin{ack}
We would like to thank the anonymous reviewers for their comments that helped us to improve the clarity of the paper, as well as other members of the Adaptive Agents Lab who provided feedback on a draft of this paper.
AMF acknowledges the funding from the Canada CIFAR AI Chairs program, as well as the
support of the Natural Sciences and Engineering Research Council of Canada (NSERC) through the Discovery Grant program.
AR was partially supported by Borealis AI through the Borealis AI Global Fellowship Award.
Resources used in preparing this research were provided, in part, by the Province of Ontario, the Government of Canada through CIFAR, and companies sponsoring the Vector Institute.
\end{ack}

\newpage
%\input{checklist}

%%%%%%%%%%%%%%%%%%%%%%%%%%%%%%%%%%%%%
\end{document}